\documentclass[twoside,11pt]{article}
\usepackage{jmlr2e}
\hypersetup{
colorlinks=true,
linkcolor=blue,
filecolor=blue,     
citecolor=blue,
breaklinks=true
}
\usepackage[utf8]{inputenc} % allow utf-8 input
\usepackage[T1]{fontenc}    % use 8-bit T1 fonts
\usepackage{url}            % simple URL typesetting
\usepackage{booktabs}       % professional-quality tables
\usepackage{amsfonts}       % blackboard math symbols
\usepackage{nicefrac}       % compact symbols for 1/2, etc.
\usepackage{microtype}      % microtypography
\usepackage{xcolor}         % colors
\usepackage{amssymb}
\usepackage{natbib}
\usepackage{color}
\usepackage{amsfonts}       
\usepackage{url}
\usepackage{amsmath}
\usepackage{wrapfig}
\usepackage{mathtools}
\usepackage{tikz}
\usepackage{algorithmic}
\usepackage{footmisc}
\usepackage{bibentry}
\nobibliography*

\usepackage{physics}
\usepackage{mathrsfs}
\mathtoolsset{showonlyrefs}
\usepackage{etoolbox}
\usepackage{makecell}
\usepackage[ruled]{algorithm2e}
\DontPrintSemicolon
\usepackage{enumerate}
\usepackage{lastpage}
\usepackage{float}
\usepackage{thmtools, thm-restate}
\usepackage{lastpage}

\newcommand{\explain}[1]{\tag*{(#1)}}
\newcommand{\explaind}[2]{\makebox[0.85\textwidth]{$\displaystyle#1$\hfill(#2)}}

\newcommand{\red}[1]{{\red{magenta}{#1}}}
\newcommand{\fS}{\mathcal{S}}
\newcommand{\fA}{\mathcal{A}}
\newcommand{\fY}{\mathcal{Y}}
\newcommand{\fB}{\mathcal{B}}

\newcommand{\fO}{\mathcal{O}}

\newcommand{\fL}{\mathcal{L}}

\newcommand{\R}{\mathbb{R}}
\newcommand{\N}{\mathbb{N}}
\newcommand{\E}{\mathbb{E}}

\newcommand{\ns}{{|\fS|}}

\newcommand{\bop}{\mathcal{T}}

\newcommand{\dsf}{\delimitershortfall=-1pt }

\newtheorem{xtheorem}{Theorem}[section]
\newtheorem{assumption}{Assumption}%[section]
\newtheorem{xassumption}{Assumption}[section]

\newtheorem{manualassumptioninner}{Assumption}
\newenvironment{manualassumption}[1]{%
  \renewcommand\themanualassumptioninner{#1}%
  \manualassumptioninner
}{\endmanualassumptioninner}

\allowdisplaybreaks

\begin{document}

\title{The ODE Method for Stochastic Approximation and Reinforcement Learning with Markovian Noise}

% The \author macro works with any number of authors. There are two commands
% used to separate the names and addresses of multiple authors: \And and \AND.
%
% Using \And between authors leaves it to LaTeX to determine where to break the
% lines. Using \AND forces a line break at that point. So, if LaTeX puts 3 of 4
% authors names on the first line, and the last on the second line, try using
% \AND instead of \And before the third author name.

\author{\name Shuze Daniel Liu \email shuzeliu@virginia.edu \\
\addr 
Department of Computer Science \\
University of Virginia\\
% 85 Engineer's Way,
Charlottesville, VA, 22903, USA
\AND
\name Shuhang Chen \email shuhang@scaledfoundations.ai \\
\addr 
Scaled Foundations 
\AND
\name Shangtong Zhang \email shangtong@virginia.edu \\
\addr 
Department of Computer Science \\
University of Virginia\\
% 85 Engineer's Way,
Charlottesville, VA, 22903, USA
}

% \jmlrheading{25}{2024}{1-\pageref{LastPage}}{1/21}{xx}{xx}{Shuze Daniel Liu, Shuhang Chen, Shangtong Zhang}
% \ShortHeadings{Stochastic Approximation with Markovian Noise}{Liu, Chen, and Zhang}

\jmlrheading{26}{2025}{1-\pageref{LastPage}}{1/24; Revised
10/24}{2/25}{24-0100}{Shuze Daniel Liu, Shuhang Chen, Shangtong Zhang}
\ShortHeadings{Stochastic Approximation with Markovian Noise}{Liu, Chen, and Zhang}

\firstpageno{1}

\editor{Martha White}

\maketitle

\begin{abstract}
Stochastic approximation is a class of algorithms that update a vector iteratively, incrementally, and stochastically, including, e.g., stochastic gradient descent and temporal difference learning. One fundamental challenge in analyzing a stochastic approximation algorithm is to establish its stability, i.e., to show that the stochastic vector iterates are bounded almost surely. In this paper, we extend the celebrated Borkar-Meyn theorem for stability from the Martingale difference noise setting to the Markovian noise setting, which greatly improves its applicability in reinforcement learning, especially in those off-policy reinforcement learning algorithms with linear function approximation and eligibility traces. Central to our analysis is the diminishing asymptotic rate of change of a few functions, which is implied by both a form of the strong law of large numbers and a form of the law of the iterated logarithm.
% and a commonly used V4 Lyapunov drift condition 
% and trivially holds if the Markov chain is finite and irreducible.
    % The only regularization we need about the Markovian noise is a form of strong law of large numbers.
    % Different from existing works,
    % we do not need Poisson's equation or any form of aperiodicity.
\end{abstract}


% \tableofcontents

\section{Introduction}

% \lsz{remove forward reference. Put auxiliary lemmas first.}
Stochastic approximation \citep{robbins1951stochastic,benveniste1990MP,kushner2003stochastic,borkar2009stochastic} is a class of algorithms that update a vector iteratively,
incrementally,
and stochastically.
Successful examples include stochastic gradient descent \citep{kiefer1952Stochastic} and temporal difference learning \citep{sutton1988learning}.
Given an initial $x_0 \in \R^d$,
stochastic approximation algorithms typically generate a sequence of vectors $\qty{x_n}$ recursively as 
\begin{align}
&x_{n+1} = x_n + \alpha(n)H(x_n,Y_{n+1}) \quad n=0,1,\dots \label{eq: x n updates} 
\end{align}
Here $\qty{\alpha(n)}_{n=0}^\infty$ is a sequence of deterministic learning rates,
$\qty{Y_n}_{n=1}^\infty$ is a sequence of random noise in a general space $\fY$ (not necessarily compact),
and $H:\R^d \times \fY \to \R^d$ is a function that maps the current iterate $x_n$ and noise $Y_{n+1}$ to the actual incremental update.

One way to analyze the asymptotic behavior of $\qty{x_n}$ is to regard $\qty{x_n}$ as Euler's discretization of the ODE
\begin{align}
    \label{eq intro ode}
    \dv{x(t)}{t} = h(x(t)),
\end{align}
where $h(x) \doteq \E\qty[H(x, y)]$ is the expected updates (the expectation will be rigorously defined shortly).
Then the asymptotic behavior of the discrete and stochastic iterates $\qty{x_n}$ can be characterized by continuous and deterministic trajectories of the ODE~\eqref{eq intro ode}. 
To establish this connection between the two,
however,
requires to establish the stability of $\qty{x_n}$ first \citep{kushner2003stochastic,borkar2009stochastic}.
In other words, 
one needs to first show that 
\begin{align}
    \sup_{n} \norm{x_n} < \infty \qq{a.s.,}
\end{align}
which is in general challenging.
Once the stability is confirmed,
the convergence of $\qty{x_n}$ follows easily \citep{kushner2003stochastic,borkar2009stochastic}.
The seminal Borkar-Meyn theorem \citep{borkar2000ode} establishes the desired stability assuming the global asymptotic stability of the following ODE
\begin{align}
    \dv{x(t)}{t} = h_\infty(x),
\end{align}
where $h_\infty(x) \doteq \lim_{c\to\infty} \frac{h(cx)}{c}$.
% which is referred to as the ODE$@\infty$.
Despite the celebrated success of the Borkar-Meyn theorem (see, e.g., \citet{abounadi2001learning,maei2011gradient}),
one major limit is that the Borkar-Meyn theorem requires $\qty{Y_n}$ to be i.i.d. noise.
As a result,
$\qty{H(x_n, Y_{n+1}) - h(x_n)}_{n=0}^\infty$ is then a Martingale difference sequence and the Martingale convergence theorem applies under certain conditions.
However,
in many Reinforcement Learning (RL, \citet{sutton2018reinforcement}) problems,
$\qty{Y_{n}}$ is a Markov chain and is not i.i.d.
\emph{Our main contribution is to extend the Borkar-Meyn theorem to the Markovian noise setting with verifiable assumptions.}
The extension to Markovian noise has been previously explored by 
\citet{ramaswamy2018stability,borkar2021ode}.
%  recently extend the Borkar-Meyn theorem such that it applies to Markovian $\qty{Y_n}$. 
However,
their assumptions are way more restrictive than ours so their results are not applicable in many important RL algorithms,
particularly,
off-policy RL algorithms with eligibility traces \citep{yu2012LSTD,yu2015convergence,yu2017convergence}.
See Section~\ref{sec rl application} for more discussion on this class of RL algorithms.

In \citet{ramaswamy2018stability},
it is assumed that the Differential Inclusion (DI)
\begin{align}
    \dv{x(t)}{t} \in \overline{\text{co}}\qty{H_\infty(x(t), y) | y \in \fY}
\end{align}
is stable,
where $\overline{\text{co}}(\cdot)$ denotes the convex hull and $H_\infty(x, y) \doteq \lim_{c\to\infty}\frac{H(cx, y)}{c}$.
To demonstrate the challenge in verifying this assumption,
we consider a special linear case where $H(x, y) = A(y) x + b(y)$ for some matrix-valued function $A(y)$ and vector-valued function $b(y)$. 
Then one sufficient and commonly used condition \citep{molchanov1989criteria} for this DI to be stable is that the $A(y)$ is uniformly negative definite,
i.e., there exists some strictly positive $\eta$ such that $x^\top A(y)x \leq -\eta \norm{x}^2 \, \forall x \in \R^d, y \in \fY$.
%  largest eigenvalue of $A(y)^\top + A(y)$ is uniformly bounded away from 0, i.e., $\sup_{y \in \fY} \lambda_{\max}(A(y)^\top + A(y)) < 0$,
% where $\lambda_{\max}(\cdot)$ denotes the largest eigenvalue.
However, in many RL algorithms (e.g., \citet{sutton1988learning,sutton2009convergent,sutton2009fast,sutton2016emphatic}, as well as the off-policy RL algorithms with eligibility traces in Section~\ref{sec rl application}),
we can at most say that $\E\qty[A(y)]$ is negative definite.
The individual matrix $A(y)$ does not have any special property.
Intuitively, \citet{ramaswamy2018stability} assume that the function $H_{\infty}(x, y)$ behaves well almost surely,
significantly limiting its application in RL.
In fact, 
we are not aware of any application of \citet{ramaswamy2018stability} in standard RL algorithms.
By contrast,
we only need $h_\infty(x)$ to behave well,
i.e., we only need $H_\infty(x, y)$ to behave well in expectation.
\citet{ramaswamy2018stability} also assume $\fY$ to be compact.
Unfortunately,
in many important RL algorithms mentioned above,
neither DI's stability nor the compactness holds.
% This compactness also limits its application in many important RL problems, e.g.,
% \citet{yu2012LSTD,yu2015convergence,yu2017convergence}.

In \citet{borkar2021ode},
it is assumed that a V4 Laypunov drift condition holds for $\qty{Y_n}$ and the eighth moment of some function is bounded.
Unfortunately,
in many important RL algorithms (see, e.g., those in Section~\ref{sec rl application}),
neither assumption holds.
We instead establish the stability via examining the \emph{asymptotic rate of change} of certain functions, inspired by \citet{kushner2003stochastic}.
When V4 does not hold, 
a form of the strong law of large numbers and a form of the law of the iterated logarithm can be used to establish the desired asymptotic rate of change.
When V4 does hold,
we only need the second moment, instead of the eighth moment, to be bounded to establish the desired asymptotic rate of change.

% , e.g.,
% \citet{yu2012LSTD,yu2015convergence,yu2017convergence},
% neither assumption holds.
% By contrast,

% In many RL problems of interest, e.g., \citet{yu2012LSTD,yu2015convergence,yu2017convergence},
% the assumptions in \citet{borkar2021ode},
% however,
% do not necessarily hold.
% A few strong assumptions in \citet{borkar2000ode},
% \citet{borkar2021ode}, however,
% rely on the existence of Poisson's equation (see, e.g., \citet{benveniste1990MP}),
% which significantly limits its application in RL.

% Intuitively,

% See Section~\ref{sec related work} for a detailed comparison.
% Central to our analysis is the diminishing \emph{asymptotic rate of change} \citep{kushner2003stochastic} of a few functions,
% which is implied by both a form of the strong law of large numbers \citep{kushner2003stochastic}
% and a commonly used V4 Lyapunov drift condition \citep{meyn2012markov}
% and trivially holds if the Markov chain is finite and irreducible.
% without using Poisson's equation.
% Instead, 
% we rely on a form of the law of large numbers,
% which can be further weakened using the notion of asymptotic rate of change \citep{kushner2003stochastic}.
% Those assumptions are arguably the weakest assumptions one can expect for the convergence of~\eqref{eq: x n updates} \citep{tadic2001convergence}.
We demonstrate in Section~\ref{sec rl application} the wide applicability of our results in RL,
especially in off-policy RL algorithms with linear function approximation and eligibility traces,
where the Markovian noise $\qty{Y_n}$ can easily grow \emph{unbounded almost surely} and have \emph{unbounded second moment}. 
The key idea of our approach is to apply the Arzela-Ascoli theorem to the scaled iterates.
Then the Moore-Osgood theorem computes a double limit, 
confirming that the scaled iterates converge to the corresponding limiting ODEs along a carefully chosen \emph{subsequence}.
This subsequence view is an important technical innovation of our approach.
By contrast,
previous works concerning the Borkar-Meyn theorem \citep{borkar2000ode,bhatnagar2011borkar,lakshminarayanan2017stability,ramaswamy2017generalization,ramaswamy2018stability,borkar2021ode}
all seek to establish the convergence along the entire sequence
to invoke a proof by contradiction argument to establish the desired stability.
This subsequence view is essential for our approach because the Arzela-Ascoli theorem can only guarantee the existence of a convergent subsequence.
As a result,
we need a variant of the standard proof by contradiction argument to establish the desired stability.

% Then a standard proof by contradiction (see, e.g., \cite{borkar2000ode}) establishes the desired stability.

% helping characterize the limiting behavior of the scaled iterates.
% Throughout the paper,
% we assume that the iterates $\qty{x_n}$ is always well defined.
% i.e., for each $n$, $\norm{x_n} < \infty$ a.s.

% The function $\alpha: \N \to \R$  can be interpreted as the learning rate.  
% $\qty{Y_n}$ is a time-homogeneous Markov chain evolving in a possibly infinite state space $\fY$,
% This Markov chain has a unique stationary distribution ${d_\fY}: \fY \to \R$.

%  depends on the previous $x_n$ and a variable $Y_{n+1}$. 
% Thus, this Markov chain
% By Markov property, we have
% $Y_{n+1} \sim p(Y_n, \cdot )$ where $p: \fS \times \fS \to [0,1]$ is the transition probability function of this Markov chain.
% We aim to prove $\qty{x_n}$ is convergent under mild assumptions.
% This work is concerned with the stability of the stochastic iterates $\qty{x_n}$.
% Once the stability is established, 
% the convergence of $\qty{x_n}$ comes for free.
% Our assumptions are quite weak.
% In particular,
% we use a set of \emph{asymptotic rate of change} assumptions,
% which follow easily from the law of large numbers on $\qty{Y_n}$.
% We do not need any additional structural information of the Markov chain.
% For example,
% we do not need Poisson's equation nor any form of aperiodicity.
% If $\fY$ is finite and $\qty{Y_n}$ is irreducible, 
% our assumptions hold immediately.

% \section{Stochastic Approximation}
\section{Main Results}
\label{sec main result}

% \sz{A general rule to write a readable proof is that you should start with sth related to the lemma statement instead of some ``random'' thing. 
% If you must do some ``random'' thing first, make them a separate lemma.
% Currently, many proofs in the appendix are hard to read. Please improve them.}

\begin{assumption}\label{assumption: stationary distribution}
The Markov chain $\qty{Y_n}$ has a unique invariant probability measure (i.e., stationary distribution), denoted by ${d_\fY}$.
% Moreover, for any measurable function $g: \fY \to \R^d$ such that $\E_{y\sim {d_\fY}}\left[\norm{g(y)}\right] < \infty$,
% we have
% \begin{align}
    % \label{eq stronger lln}
    % \lim_{n\to\infty} \frac{1}{n+1} \sum_{i=0}^n g(Y_n) = \E_{y\sim{d_\fY}}\left[g(y)\right] \qq{a.s.}
% \end{align}
% for any initial condition $Y_0$.
\end{assumption}
% \sz{Replace ${d_\fY}$ with some other thing. We use ${d_\fY}$ to denote behaivor policy in Section~\ref{sec rl application}. Probably ${d_\fY}$}
Technically speaking, 
the uniqueness and even the existence of the invariant probability measure can be relaxed,
as long as the average of certain functions exists.
We are, however, not aware of any applications where such relaxation is a must.
We, therefore,
use Assumption~\ref{assumption: stationary distribution} to ease presentation
and refer the reader to A1.3 in Chapter 6 of \citet{kushner2003stochastic} as an example of such relaxation. 
In light of the update~\eqref{eq: x n updates},
we use the convention that $\qty{Y_n}$ starts from $n = 1$.

% Assumption~\ref{assumption: stationary distribution} is all we require about the Markov chain. 
% It is stronger than the strong law of large number (see, e.g., Theorem 17.1.2 of \citet{meyn2012markov}),
% whichs only states that~\eqref{eq stronger lln} holds for any $Y_0 \in \fY_g$,
% where $\fY_g$ is an unknown, probably $g$-dependent set such that ${d_\fY}(\fY_g) = 1$.
% For now we present Assumption~\ref{assumption: stationary distribution} in its current form for easing presentation.
% We will discuss weaker assumptions in Section~\ref{sec main proof stability}.

\begin{assumption}
    % \label{assumption: alpha property}
\label{assumption: alpha rate}
% $\forall i < 0$, $\alpha(i) = 0$. $\forall i \geq 0$, 
The learning rates $\qty{\alpha(n)}$ are positive, decreasing, and satisfy
\begin{align}
\label{eq 1/n lr}
\sum_{i=0}^{\infty} \alpha(i) =& \infty, \,  \lim_{n \to \infty} \alpha(n) = 0, \text{and }  \frac{\alpha(n)- \alpha(n+1)}{\alpha(n)} = \fO\left(\alpha(n)\right). 
\end{align}
\end{assumption}
\begin{remark}
    For any $\alpha(n) = \frac{B_1}{(n+B_2)^{\beta}}$ with $\beta \in (0.5, 1]$,
    it can be easily computed that 
\begin{align}
    \frac{\alpha(n) - \alpha(n+1)}{\alpha(n)} = \fO\left(\frac{\beta}{n}\right) = \fO\left(\alpha(n)\right).
\end{align}
    % We, however, do note that there are some other assumptions to be presented shortly that are easy to verify only when $\beta \in (0.5, 1]$.
\end{remark}
% This assumption implies 
% \begin{align}
% \lim_{i \to \infty} \alpha(i) = 0.  \label{eq 1/n lr}
% \end{align}
% \begin{assumption}
    % The learning rates $\qty{\alpha(i)}$ satisfy
    % \begin{align}
    % \end{align}
% \end{assumption}
Next, we make a few assumptions about the function $H$.
For any 
$c \in [1, \infty)$, define 
\begin{align}
H_c(x,y) &\doteq \frac{H(cx,y)}{c} \label{def: H c}.
\end{align}
The function $H_c$ is the rescaled version of the function $H$ and will be used to construct rescaled iterates,
which are key techniques in proving the Borkar-Meyn theorem (see, e.g., \citet{borkar2000ode,borkar2009stochastic}).
Similar to \citet{borkar2000ode,borkar2009stochastic},
we need the limit of $H_c$ to exist in a certain sense when $c \to \infty$.

\begin{assumption}
\label{assumption: specific H c H infinity}
There exists a measurable function $H_\infty(x, y)$, 
a function $\kappa: \R \to \R$ (independent of $x, y$), and a measurable function $b(x, y)$ such that for any $x, y,$
\begin{align}
H_c(x, y) - H_\infty(x, y) =& \kappa(c)b(x, y), \label{eq: h c h inf kappa} \\
\lim_{c\to\infty} \kappa(c) =& 0, \label{eq: kappa to 0}
% \sup_{x \in \fB} \norm{ \E_{y\sim {d_\fY}}\left[b(x, y)\right]} <& \infty \label{eq: b finite},
\end{align}
Moreover, there exists a measurable function $L_b(y)$ such that $\forall x, x', y$,
\begin{align}
\norm{b(x, y) - b(x', y)} \leq L_b(y) \norm{x - x'}. \label{eq: L b Lipschitz}
\end{align}
And the expectation $L_b \doteq \E_{y\sim{d_\fY}}\left[L_b(y)\right]$
is well-defined and finite.
% where $\fB$ is any compact subset of $\R^d$.
\end{assumption}
Assumption~\ref{assumption: specific H c H infinity} provides details on how $H_c$ converges to $H_\infty$ when $c \to \infty$.
%  and will soon be weakened in Section~\ref{sec weaker assumptions}.
We note that in many RL applications, see, e.g., Section~\ref{sec rl application},
the function $b(x, y)$ actually does not depend on $x$ so~\eqref{eq: L b Lipschitz} trivially holds.
We consider $b(x, y)$ as a function of both $x$ and $y$ for generality.
%  immediately implies that
% \begin{align}
    % \label{eq raw H convergence}
    % \lim_{c\to\infty} H_c(x, y) =& H_\infty(x, y)
% \end{align}
% but provides more details about this convergence.
Next, we assume Lipschitz continuity of the functions $H_c$,
which guarantees the existence and uniqueness of the solutions to the corresponding ODEs.
\begin{assumption}\label{assumption: H Lipschitz}
There exists a measurable function $L(y)$ such that for any $x,x',y$,
\begin{align}
\norm{H(x,y) - H(x',y)} &\leq L(y) \norm{x - x'}, \label{eq: H L}\\
\norm{H_\infty(x,y) - H_\infty(x',y)} &\leq L(y) \norm{x - x'}. \label{eq: H L inf}
\end{align}
Moreover,
the following expectations are well-defined and finite 
% \lsz{Should we change to - i.e. their norms are finite} 
for any $x$:
\begin{align}
h(x) &\doteq    \E_{y\sim {d_\fY}}[H(x,y)],  \label{def: h} \\
% h_c(x) &\doteq  \E_{y \sim {d_\fY}}[H_c(x,y)] = \frac{h(cx)}{c} , \\ 
h_\infty(x) &\doteq \E_{y \sim {d_\fY}}[H_{\infty}(x,y)], \\
L &\doteq \E_{y \sim {d_\fY}}[L(y)].
\end{align}
% Moreoever, for any $x$,
% \begin{align}
    % \lim_{c\to\infty} h_c(x) = h_\infty(x).
% \end{align}
\end{assumption}
Apparently, the function $x \mapsto H_c(x, y)$ shares the same Lipschitz constant $L(y)$ as the function $x \mapsto H(x, y)$.
Similar to~\eqref{def: H c}, we define
\begin{align}
    \label{def: h c}
    h_c(x) \doteq \frac{h(cx)}{c}.
\end{align}
The following assumption is the central assumption in the original proof of the Borkar-Meyn theorem.

\begin{assumption}\label{assumption: lim h uniformly convergent} (Assumption A5 in Chapter 3 of \citet{borkar2009stochastic})
As $c \to \infty$,
$h_c(x)$ converges to $h_\infty(x)$ uniformly in $x$ on any compact subsets of $\R^d$. 
The ODE 
\begin{align}
    \label{eq ode at limit}
\frac{d x(t)}{d t} = h_{\infty}(x(t)) \tag{ODE$@\infty$}
\end{align}
has 0 as its globally asymptotically stable equilibrium. 
    % The above ODE is referred to as the ``ODE$@\infty$''.
\end{assumption}
We refer the reader to 
\citet{dai1995positive, dai1995stability, borkar2000ode, borkar2009stochastic, fort2008stability, meyn2008control, meyn2022control}
for the root and history of~\eqref{eq ode at limit}.
\begin{assumption}\label{assumption: lln}
    % The learning rates $\qty{\alpha(n)}$ satisfy
% \begin{align}
% \label{eq 1/n lr}
% \end{align}
Let $g$ denote any of the following functions:
\begin{align}
    \label{eq lln g1}
    y \mapsto& H(x, y) \quad (\forall x), \\
    % y \mapsto& b(x, y) \quad (\forall x), \\
    \label{eq lln g2}
    y \mapsto& L_b(y), \\
    \label{eq lln g3}
    y \mapsto& L(y).
\end{align}
Then for any initial condition $Y_1$, it holds that
\begin{align}
    \label{eq stronger lln}
\lim_{n\to\infty} \alpha(n) \sum_{i=1}^n\qty( g(Y_i) - \E_{y\sim{d_\fY}}\left[g(y)\right]) = 0 \qq{a.s.} 
    % \tag{LLN or LIL}
\end{align}
\end{assumption}
\begin{remark}
    Consider $\alpha(n) = \frac{B_1}{(n + B_2)^\beta}$ as an example again. For $\beta = 1$,
    \eqref{eq stronger lln} is implied by the following Law of Large Numbers~\eqref{eq lln new}
    \begin{align}
        \label{eq lln new}
        \lim_{n\to\infty} \frac{1}{n} \sum_{i=1}^n\qty( g(Y_i) - \E_{y\sim{d_\fY}}\left[g(y)\right]) = 0 \qq{a.s.} \tag{LLN}
    \end{align}
    For $\beta \in (0.5, 1]$, 
    \eqref{eq stronger lln} is implied by the following Law of the Iterated Logarithm~\eqref{eq lil}
    \begin{align}
        \label{eq lil}
        \norm{\sum_{i=1}^n\qty( g(Y_n) - \E_{y\sim{d_\fY}}\left[g(y)\right])} \leq \zeta \sqrt{n \log\log n}  \qq{a.s.,} \tag{LIL}
    \end{align}
where $\zeta$ is a sample path dependent finite constant. 
% For $\beta \in (0, 0.5]$,
% how to ensure~\eqref{eq stronger lln} without strong domain knowledge remains an open problem.
\end{remark}
\begin{remark}
    If the Markov chain $\qty{Y_n}$ is positive\footnote{See page 235 of \citet{meyn2012markov} for the definition of positive chains.} Harris\footnote{See page 204 of \citet{meyn2012markov} for the definition of Harris chains.},
    then~\eqref{eq lln new} holds for any function $g$ whenever $\E\qty[\norm{g(y)}] < \infty$ (Theorem 17.0.1 (i) of \citet{meyn2012markov}). 
    If $\qty{Y_n}$ is further V-uniformly ergodic\footnote{See page 387 of \citet{meyn2012markov} for the definition of V-uniform ergodicity.}, 
    then~\eqref{eq lil} holds (Theorem 17.0.1 (iii) and (iv) of \citet{meyn2012markov}). 
    For the special case where $\fY$ is finite,~\eqref{eq lln new} holds when the Markov chain is irreducible and~\eqref{eq lil} holds when it is further aperiodic.
\end{remark}
\begin{remark}
    % For general Markov chains that are not positive Harris, verifying~\eqref{eq lln new} and~\eqref{eq lil} is challenging.
    We note that~\eqref{eq lln new} is stronger than Doob's strong law of large numbers on stationary processes (see, e.g., Theorem 17.1.2 of \citet{meyn2012markov}, referred to as Doob's LLN hereafter).
Doob's LLN concludes (at most) that~\eqref{eq lln new} holds for any $Y_1 \in \fY_g$,
where $\fY_g$ is an unknown, probably $g$-dependent set such that ${d_\fY}(\fY_g) = 1$.
If we use only Doob's LLN,
all the ``almost surely'' statements in the paper must be replaced by ``$\fY_*$-almost surely'',
where
    $\fY_* \doteq \bigcap_{g} \fY_g$.
This means that all the statements hold only when $Y_1 \in \fY_*$.
However, since the $g$ functions in Assumption~\ref{assumption: lln} depend on $x$,
this $\fY_*$ is an intersection of possibly uncountably many sets $\qty{\fY_g}$.
It is possible that in some applications $\fY_*$ turns out to be a set of interest,
where~\eqref{eq lln new} can indeed be relaxed to Doob's LLN.
But in general,
characterizing $\fY_*$ is pretty challenging.
\end{remark}
\begin{remark}
    The Markov chain $\qty{Y_n}$ we consider in our RL applications in Section~\ref{sec rl application} is a general space Markov chain but is not positive Harris.
    Fortunately, \citet{yu2012LSTD,yu2015convergence,yu2017convergence} have established that~\eqref{eq lln new} holds for those chains.
    Whether~\eqref{eq lil} holds for those chains remains open.
\end{remark}
To better contrast our work with \citet{borkar2021ode},
in the following, we provide an alternative to Assumption~\ref{assumption: lln}.
% Assumption~\ref{assumption possion} allows for broader choices of learning rates, e.g., $\alpha(n) = \frac{B_1}{(n+B_2)^\beta}$ with $\beta \in (0.5, 1]$,
% at the price of requiring stronger regularity of the Markovian noise.
\begin{manualassumption}{\ref*{assumption: lln}$'$}
\label{assumption possion}
The learning rates $\qty{\alpha(n)}$ further satisfy
% \begin{align}
    % \label{eq lr broader}
    $\sum_{n=0}^\infty \alpha(n)^2 < \infty$.
    % \, \lim_{n\to\infty} \frac{1}{\alpha(n+1)} - \frac{1}{\alpha(n)} < \infty.
% \end{align}
The Markov chain $\qty{Y_n}$ is $\psi$-irreducible\footnote{
    See page 91 of \citet{meyn2012markov} for the definition of 
    $\psi$-irreducibility.
    %  is a standard concept in general state space Markov chain,
    % detailing which, however, needs quite a few auxiliary definitions and unnecessarily complicates the presentation.
    % We refer the reader to  its definition.
}.
The Lyapunov drift condition \eqref{eq v4} holds for the Markov chain $\qty{Y_n}$.\footnote{See page 371 of \citet{meyn2012markov} for in-depth discussion about~\eqref{eq v4}.}
In other words, 
there exists a Lyapunov function $v: \fY \to [1, \infty]$ such that for any $y \in \fY$,
\begin{align}
    \label{eq v4}
    \E\left[v(Y_{n+1}) - v(Y_n) | Y_n = y\right] \leq -\delta v(y) + \tau \mathbb{I}_{C}(y) \tag{V4}.
\end{align}
Here $\delta > 0, \tau < \infty$ are constants, 
$C$ is a small set\footnote{
% Let $P$ denote the transition kernel of $\qty{Y_n}$.
% A measurable set $C$ is called a \emph{small set} if there exists some $n_0 > 0$ and a nontrivial measure $\nu$ such that $P^{n_0}(y, B) \geq \nu(B)$ holds for all measurable sets $B$ and $y \in C$.
See page 109 of \citet{meyn2012markov} for the definition of small sets.
},
and $\mathbb{I}$ is the indicator function.
Moreover,
let $g$ be any of the functions $H(0, y), L_b(y)$, and $L(y)$.
Then $g \in \fL^2_{v, \infty}$\footnote{
    $g$ belongs to $\fL^p_{v, \infty}$ if and only if $\sup_{y \in \fY}\frac{\norm{g(y)}_p^p}{v(y)} < \infty$,
    where $v$ is the Lyapunov function in~\eqref{eq v4}.
}.
\end{manualassumption}
Assumption~\ref{assumption possion} uses the idea of \citet{borkar2021ode} but is weaker than its counterparts.
See more detailed comparisons in Section~\ref{sec related work}.
% Assumption~\ref{assumption possion} is trivial in the following sense.
% \begin{remark}
%     \label{remark finite2}
%     If $\qty{Y_n}$ is irreducible and $\fY$ is finite, then both~\eqref{eq v4} and $g \in \fL_{v,\infty}^2$ hold.\footnote{To see this, let $P$ be the transition matrix of $\qty{Y_n}$ and $e$ be the all-one vector.
%     Define $v \doteq e$.
%     Then $Pv - v = -v + Pe = -v + e$.
%     In other words, we set $\delta = \tau = 1$ and $C = \fY$ in~\eqref{eq v4}.
%     The fact that $\fY$ is a small set follows easily from irreducibility, i.e.,
%     there exists an $n_0 > 0$ such that $\min_{y, y'} P^{n_0}(y, y') > 0$.
%     The fact that $g \in \fL_{v, \infty}^2$ follows immediately from the finiteness of $\fY$.}
% \end{remark}
%     Remarks~\ref{remark finite} \&~\ref{remark finite2} show that for finite irreducible $\qty{Y_n}$, 
%     Assumption~\ref{assumption possion} is more applicable than Assumption~\ref{assumption: lln} since it allows more choices of learning rates.
%     But for general state space $\qty{Y_n}$,
%     we argue that Assumption~\ref{assumption: lln} is more applicable than Assumption~\ref{assumption possion},
%     at least in RL.
%     See Sections~\ref{sec related work} \&~\ref{sec rl application} for details.
\begin{remark}
    \label{remark applicability}
    Assumption~\ref{assumption possion} is listed here mostly for better comparison with \citet{borkar2021ode}.
    We are not aware of any RL application where Assumption~\ref{assumption possion} holds but Assumption~\ref{assumption: lln} does not hold.
    Instead,
    in the RL applications in Section~\ref{sec rl application},
    Assumption~\ref{assumption: lln} holds but Assumption~\ref{assumption possion} does not.
    That being said,
    the applicability of Assumptions~\ref{assumption: lln} and~\ref{assumption possion} outside RL is beyond the scope of this work.
\end{remark}
Having listed all the assumptions, 
our main theorem confirms the stability of $\qty{x_n}$. 
\begin{theorem}
\label{thm: stability}
Let Assumptions \ref{assumption: stationary distribution} - \ref{assumption: lim h uniformly convergent} hold. 
Let Assumption \ref{assumption: lln} or~\ref{assumption possion} hold.
Then the iterates $\qty{x_n}$ generated by \eqref{eq: x n updates} are stable, i.e., 
\begin{align}
\sup_n \norm{x_n} < \infty \quad a.s.  
\end{align}
\end{theorem}
Its proof is in Section~\ref{sec main proof stability}.
Once the stability is established,
the convergence follows easily. 

\begin{corollary}
\label{cor: stability main}
Let Assumptions \ref{assumption: stationary distribution} - \ref{assumption: lim h uniformly convergent} hold. 
Let Assumption \ref{assumption: lln} or~\ref{assumption possion} hold.
Then the iterates $\qty{x_n}$ generated by~\eqref{eq: x n updates} converge almost surely to a (sample path dependent) bounded invariant set\footnote{
        A set $X$ is an invariant set of the ODE~\eqref{eq original ode} if and only if for every $x \in X$, 
        there exists a solution $x(t)$ to the ODE~\eqref{eq original ode} such that $x(0) = x$ and $x(t) \in X$ for all $t \in (-\infty, \infty)$. 
        If the ODE~\eqref{eq original ode} is globally asymptotically stable,
        the only bounded invariant set is the singleton $\qty{x_*}$,
        where $x_*$ denotes the unique globally asymptotically stable equilibrium.
        We refer the reader to page 105 of \citet{kushner2003stochastic} for more details.
        } of the ODE\footnote{By $\qty{x_n}$ converges to a set $X$, we mean $\lim_{n \to \infty} \inf_{x\in X} \norm{x_n - x} = 0$.
        }
    \begin{align}
        \label{eq original ode}
    \dv{x(t)}{t} = h(x(t)).
    \end{align}
\end{corollary}
Arguments used in proving Corollary~\ref{cor: stability main} are similar but much simpler than the counterparts in the proof of Theorem~\ref{thm: stability}.
% We thus include a sketch 
% and they have large similarity. We recommend first reading the proof of our main contribution - Theorem \ref{thm: stability} then back to the proof of  Corollary~\ref{cor: stability main}.
We include a proof of Corollary~\ref{cor: stability main} in Appendix \ref{appendix: convergence cor}
with the details of those similar but simpler lemmas omitted to avoid verbatim repetition.

It is worth mentioning that it is easy to extend our results to more general updates 
\begin{align}
x_{n+1} = x_n + \alpha(n)\left(H(x_n,Y_{n+1}) + M_{n+1} + \epsilon_n\right),
\end{align}
where $M_{n+1}$ is a Martingale difference sequence and $\epsilon_n$ is another additive noise.
Similarly, it would require the asymptotic rate of change of $\qty{M_{n+1}}$ and $\qty{\epsilon_n}$ to diminish.
We refer the reader to \citet{kushner2003stochastic} for more details.
Since our main contribution is the stability under the Markovian noise $\qty{Y_{n+1}}$,
we use the simpler update rule~\eqref{eq: x n updates} for improving clarity. 

% We only include a sketch of the proof of Corollary~\ref{cor: stability main} in Appendix \ref{appendix: convergence cor}, for both completeness and comparison with the proof of Theorem~\ref{thm: stability}.

\section{Related Work}
\label{sec related work}

\paragraph*{General $H$.} 
In this paper,
the function $H$ can be a general function and we do not make any linearity assumptions.
We first compare our results with existing works applicable to general $H$ and Markovian noise $\qty{Y_n}$.
Since convergence follows easily from stability,
we focus on comparison in terms of establishing stability.
Notably,
the related stability results in \citet{borkar2000ode,borkar2009stochastic} are superceded by \citet{borkar2021ode}.
We,
therefore,
discuss only \citet{borkar2021ode,kushner2003stochastic,benveniste1990MP}.

Compared with \citet{borkar2021ode},
our improvements lie in two aspects.
First, central to \citet{borkar2021ode} are (i) a V4 Laypunov drift condition,
(ii) an aperiodicity assumption of $\qty{Y_n}$,
and (iii) a boundedness assumption $L(y) \in \fL_{v, \infty}^8$.
By contrast,
our Assumption~\ref{assumption possion} only requires $L(y) \in \fL_{v, \infty}^2$ and does not need aperiodicity.
Second,
we further provide an approach that establishes the stability based on
Assumption~\ref{assumption: lln}
% a form of the strong law of larger numbers (Assumption~\ref{assumption: lln})
% , or asymptotic rate of change (Lemma~\ref{lemma: H h 0}),
without using~\eqref{eq v4},
aperiodicity, and the boundedness in $\fL_{v, \infty}^8$.
As noted in Remark~\ref{remark applicability}, Assumption~\ref{assumption: lln} is more applicable than Assumption~\ref{assumption possion} in RL.

Compared with \citet{kushner2003stochastic},
our main improvement is that we prove stability under the asymptotic rate of change conditions.
By contrast,
\citet{kushner2003stochastic} mostly use stability as a priori and are concerned with the convergence of projected algorithms in the form of 
\begin{align}
    x_{n+1} = \Pi\left(x_n + \alpha(n)H(x_n,Y_{n+1})\right),  
\end{align}
where $\Pi$ is a projection to some compact set to ensure stability of $\qty{x_n}$.
As a result,
the corresponding ODE (cf. Corollary~\ref{cor: stability main}) becomes
\begin{align}
    \dv{x(t)}{t} = h(x(t)) + \xi(t),
\end{align}
where $\xi(t)$ is a reflection term resulting from the projection $\Pi$.
We refer the reader to Section~5.2 of \citet{kushner2003stochastic} for more details regarding this reflection term.
Analyzing these reflection terms typically requires strong domain knowledge, see, e.g., \citet{yu2015convergence,zhang2021breaking}, and Section 5.4 of \citet{borkar2009stochastic}.

We argue that this work combines the best of both \citet{borkar2000ode} and \citet{kushner2003stochastic},
i.e.,
the ODE$@\infty$ technique for establishing stability from \citet{borkar2000ode}
and the asymptotic rate of change technique for averaging out the Markovian noise $\qty{Y_n}$.
As a result,
our results are more general than both \citet{borkar2021ode} and \citet{kushner2003stochastic} in the aforementioned sense.

Compared with \citet{benveniste1990MP},
our main improvement is that despite the proof under Assumption~\ref{assumption possion} essentially using Poisson's equation\footnote{
    Let $g$ be a function defined on $\fY$.
    The Poisson's equation holds for $g$ if there exists a finite function $\hat g$ such that 
    $\hat g(y) = g(y) - \E_{y\sim{d_\fY}}\left[g(y)\right] + \int_\fY P(y, y') \hat g(y') dy'$ holds for any $y \in \fY$,
    where $P$ denotes the transition kernel of $\qty{Y_n}$.
    The drift condition~\eqref{eq v4}, together with some other mild conditions,
    is sufficient to ensure the existence of Poisson's equation. 
    We refer the reader to Theorem 17.4.2 of \citet{meyn2012markov} for more details.
},
the proof under Assumption~\ref{assumption: lln} does not need Poisson's equation at all.
Notably, \citet{benveniste1990MP} assume Poisson's equation directly without specifying sufficient conditions to establish Poisson's equation.
Moreover,
to establish stability,
\citet{benveniste1990MP} require a Lyapunov function for the ODE~\eqref{eq original ode} that is always greater than or equal to $\alpha \norm{\cdot}^2$ for some $\alpha > 0$ (Condition (ii) of Theorem 17 in \citet{benveniste1990MP}).
By contrast,
our Assumption~\ref{assumption: lim h uniformly convergent} does not put any restriction on the possible Lyapunov functions.
We also note that \citet{borkar2021ode} is also based on an error representation similar to \citet{benveniste1990MP} enabled by Poisson's equation.

% \cite{ramaswamy2018stability} establish stability using Differential Inclusion (DI).
% In particular, 
% they assume a certain form of stability of the DI 
% \begin{align}
%     \label{eq di at limit}
%     \dv{x(t)}{t} \in \bar h_\infty(x(t)),\tag{DI$@\infty$} 
% \end{align}
% where $\bar h_\infty$ is a set-valued function defined as $\bar h_\infty(x) \doteq \text{closure}\left({\qty{h(x, y) | y \in \fY}}\right)$.
% The setting \cite{ramaswamy2018stability} consider is adversarial,
% i.e., their $\qty{Y_n}$ can be an arbitrary controlled Markov process as long as $\fY$ is compact.
% As a result,
% their definition of $\bar h_\infty$ is a set-valued function derived from all possible $y \in \fY$.
% By contrast,
% our $h_\infty$ in~\eqref{def: h} is only an expectation.
% We, therefore, argue that assuming the stability of~\eqref{eq ode at limit} is weaker than assuming the stability of~\eqref{eq di at limit}.
% Moreover,
% \cite{ramaswamy2018stability} require $\fY$ to be compact,
% which makes it hard to apply in the RL algorithms in Sections~\ref{sec gtd} \&~\ref{sec etd}.
% By contrast, our $\fY$ is a general space.

\paragraph*{Linear $H$.}
If we further assume that the function $H(x, y)$ has a linear form, i.e.,
\begin{align}
    H(x, y) = A(y) x + b(y),
\end{align}
% with $A: \fY \to \R^{d \times d}, y: \fY \to \R^d$,
there are several other results regarding the stability (and thus convergence),
e.g.,
\citet{konda2000actor,tadic2001convergence,yu2015convergence} and Proposition 4.8 of \citet{bertsekas1996neuro}.
They, however, all require that the matrix $A \doteq \E_{y\sim {d_\fY}}\left[A(y)\right]$ is negative definite\footnote{A real matrix $A$, not necessarily symmetric, is negative definite if and only if all the eigenvalues of the symmetric matrix $A + A^\top$ is strictly negative.}.
But contrast, 
our Assumption~\ref{assumption: lim h uniformly convergent} only requires $A$ to be Hurwitz\footnote{A real matrix $A$ is Hurwitz if and only if the real parts of all its eigenvalues are strictly negative.} (see, e.g., Theorem 4.5 of \citet{khalil2002nonlinear}),
which is a weaker condition.\footnote{All negative definite matrices are Hurwitz, but many Hurwitz matrices are not negative definite. See Chapter~2 of \citet{horn_johnson_1991} for more details.}
In Section~\ref{sec rl application},
we provide a concrete RL algorithm where the corresponding $A$ matrix is Hurwitz but not negative definite.

\paragraph*{Local clock.}
Another approach to deal with Markovian noise $\qty{Y_n}$ is to apply results in asynchronous schemes.
We refer the reader to Chapter 7 of \citet{borkar2009stochastic} for details.
The major limitation is that it requires count-based learning rates.
At the $n$-th iteration,
instead of using $\alpha(n)$, 
where $n$ can be regarded as a ``global lock'',
the asynchronous schemes use $\alpha(\varpi(n, Y_{n+1}))$ as the learning rate,
where $\varpi(n, y)$ counts the number of visits to the state $y$ until time $n$ and can be regarded as a ``local clock''.
The asynchronous schemes also have other assumptions regarding the local clock.
Successful examples include \citet{abounadi2001learning,wan2020learning}.
However, we are not aware of any successful applications of such count-based learning rates in RL with function approximation,
% Such count-based learning rates are unlikely to work for a general state space $\fY$.
% Even for a countable $\fY$,
% we are not aware of any meaningful tabular RL algorithm.
% So it seems function approximation is necessary for a general $\fY$,
where an RL algorithm typically only has access to some feature $\phi(Y_n)$ instead of $Y_n$ directly.
Unless $\phi$ is a one-to-one mapping,
there will be no way to count the state visitation.

\paragraph*{Other type of noise.}
The Borkar-Meyn theorem applies to only Martingale difference noise,
which is, later on, relaxed to allow more types of noise, e.g., 
\citet{bhatnagar2011borkar, ramaswamy2017generalization}.
% \sz{cite 5, 38, 39 of old version of \citet{borkar2021ode}}. \lsz{v2}
However,
none of those extensions applies to general Markovian noise.

\section{Proof of Theorem~\ref{thm: stability}}
\label{sec main proof stability}

This section is dedicated to proving Theorem~\ref{thm: stability}.
Overall,
we prove by contradiction.
Section~\ref{sec weaker assumptions} sets up notations and establishes the desired diminishing asymptotic rate of change of a few functions.
Section~\ref{sec: construct sequence} establishes the desired equicontinuity.
Section~\ref{sec: convergent sub seq} assumes the opposite and thus identifies a subsequence of interest.
Section~\ref{sec: sequence property} analyzes the property of the subsequence,
helping the \emph{reductio ad absurdum} in Section~\ref{sec: explicit property}.
% We recall that we proceed under Assumptions~\ref{assumption: stationary distribution},~\ref{assumption: alpha rate},~\ref{assumption: lim H uniformly convergent},~\ref{assumption: H Lipschitz},~\ref{assumption: lim h uniformly convergent}, \&~\ref{assumption: H h 0}.
Lemmas in this section are derived on an arbitrary sample path $\qty{x_0,\qty{Y_i}_{i=1}^\infty}$ such that 
the assumptions in Section~\ref{sec main result} hold.
% Assumptions~\ref{assumption: stationary distribution},~\ref{assumption: alpha rate},~\ref{assumption: lim H uniformly convergent},~\ref{assumption: H Lipschitz},~\ref{assumption: lim h uniformly convergent}, \&~\ref{assumption: H h 0} hold.
Thus, we omit ``$a.s.$'' on the lemma statements for simplicity.

\subsection{Diminishing Asymptotic Rate of Change}
\label{sec weaker assumptions}
% In the previous section,
% we present our stability result with 
% Assumptions~\ref{assumption: stationary distribution} - \ref{assumption: lln} \&~\ref{assumption possion},
% mostly for the ease of presentation.
% Many of those assumptions can actually be weakened using the idea of \emph{asymptotic rate of change}.
% To present the weaker assumptions, 
% however,
% requires some new technical definitions,
% which also play key roles in the proof of Theorem~\ref{thm: stability} in Section~\ref{sec main proof stability}.

We divide the non-negative real axis $[0, \infty)$ into segments of length $\qty{\alpha(i)}_{i=0,1,\dots}$. 
Those segments are then grouped into larger
intervals $\qty{[T_n, T_{n+1})}_{n=0,1,\dots}$.
The sequence $\qty{T_n}$ has the property that $T_{n+1} - T_n \approx T$ for some fixed $T$ and as $n$ tends to $\infty$,
the error in this approximation diminishes.
Precisely speaking, we define 
\begin{align}
    t(0) \doteq& 0, \\
    t(n) \doteq& \sum_{i=0}^{n-1} \alpha(i) \qq{$n=1,2,\dots$}.
\end{align}
For any $T > 0$, define 
\begin{align}
m(T) = \max\qty{{i | T \geq t(i)}} \label{def: m}
\end{align} to be the largest $i$ that has $t(i)$ smaller or equal to $T$. 
Intuitively,
$t(m(T))$ is ``just'' left to $T$ in the real axis.
Then $t(m(T))$ has the follow properties:
\begin{align}
& t(m(T)) \leq T < t(m(T) + 1) = t(m(T)) + \alpha(m(T)) \label{eq: t-m-inequality-1}, \\   
&  t(m(T)) > T -\alpha(m(T))\label{eq: t-m-inequality-2}.
\end{align}
Define 
\begin{align}
&T_0 = 0, \\   
&T_{n+1} = t(m(T_n+T) + 1) \label{def:T}.
\end{align}
Intuitively,
$T_{n+1}$ is ``just'' right to $T_n + T$ in the real axis.
% We show some useful properties for $T_n$.
For proving Theorem~\ref{thm: stability},
it suffices to work with solutions of ODEs in only $[0, \infty)$.
But for Corollary~\ref{cor: stability main},
it is necessary to consider solutions of ODEs in $(-\infty, \infty)$.
To this end, 
we define 
\begin{align}
\alpha(i) =& 0 \quad \forall i < 0, \label{eq: alpha 0} \\
m(t) =& 0 \quad \forall t \leq 0,  \label{eq: m 0}
\end{align}
for simplifying notations.
% The following lemma shows some elementary but useful properties of this segmentation.
For any given function $f$ with domain $\fY$,
its asymptotic rate of change is defined as
\begin{align}
    \limsup_n \sup_{-\tau \leq t_1 \leq t_2 \leq \tau} \norm{\sum_{i = m(t(n) + t_1)}^{m(t(n) + t_2) - 1} \alpha(i) [f(Y_{i+1}) - \E_{y\sim d_\fY}\qty[f(y)]] }.
\end{align}
The asymptotic rate of change characterizes the asymptotic regularity of the sequence $\qty{f(Y_n)}$ and is a powerful tool to study stochastic approximation iterates.
We refer the reader to Sections~5.3.2 and~6.2 of \citet{kushner2003stochastic} for an in-depth exposition of this tool.
In the following, we demonstrate that the asymptotic rate of change is 0 for the functions in Assumption~\ref{assumption: lln}. 
% \lsz{change Lemma~\ref{lemma: H h 0} to Lemma~\ref{lemma: H h 0}

% change Assumption~\ref{assumption: lim H uniformly convergent} to Lemma~\ref{lemma: lim H uniformly convergent}
% }

\begin{lemma}\label{lemma: H h 0}
Let Assumptions~\ref{assumption: stationary distribution},~\ref{assumption: alpha rate}, and \ref{assumption: H Lipschitz} hold.
Let Assumption~\ref{assumption: lln} or~\ref{assumption possion} hold.
Then the asymptotic rate of change of the functions~\eqref{eq lln g1},~\eqref{eq lln g2}, and~\eqref{eq lln g3} is 0, i.e., for any fixed $\tau > 0$ and $x$, it holds that
\begin{align}
    \limsup_n \sup_{-\tau \leq t_1 \leq t_2 \leq \tau} \norm{ \sum_{i = m(t(n) + t_1)}^{m(t(n) + t_2) - 1} \alpha(i) \left[H(x, Y_{i+1}) - h(x) \right] }    & = 0 \quad a.s., \label{eq: H h minus 0}\\
\limsup_n \sup_{-\tau \leq t_1 \leq t_2 \leq \tau} \norm{\sum_{i = m(t(n) + t_1)}^{m(t(n) + t_2) - 1} \alpha(i) [L_b(Y_{i+1}) - L_b] }    &= 0 \quad a.s., \label{eq: L b property minus 0}\\
\limsup_n \sup_{-\tau \leq t_1 \leq t_2 \leq \tau} \norm{\sum_{i = m(t(n) + t_1)}^{m(t(n) + t_2) - 1} \alpha(i) [L(Y_{i+1}) - L] }    &= 0 \quad a.s. \label{eq: L property minus 0}
    \end{align}
\end{lemma}
Its proof is in Appendix \ref{appendix: H h 0}.
% \sz{Rewrite the proof to include the proof under both 6 and 6'.}
Furthermore, the convergence of $H_c$ to $H_\infty$ in Assumption~\ref{assumption: specific H c H infinity} demonstrates a similar pattern.
\begin{lemma}\label{lemma: lim H uniformly convergent}
Let Assumptions~\ref{assumption: stationary distribution},~\ref{assumption: alpha rate},~\ref{assumption: specific H c H infinity}, and~\ref{assumption: H Lipschitz} hold.
Let Assumption~\ref{assumption: lln} or~\ref{assumption possion} hold. 
It then holds that
% \begin{align}
% \lim_{c\to\infty} \norm{ \sum_{i=m(T_n)}^{m(T_n+t) - 1} \alpha(i) \left[H_c(x, Y_{i+1}) - H_\infty(x, Y_{i+1}) \right] } = 0 \qq{a.s.,}
% \end{align}
% and uniformly in $n$, in $x$ on any compact subsets of $\R^d$, and in $t \in [0, T]$.
\begin{align}
&\lim_{c \to \infty} \sup_{x \in \fB} \sup_n \sup_{t \in [0,T]} \norm{ \sum_{i=m(T_n)}^{m(T_n+t) - 1} \alpha(i) \left[H_c(x, Y_{i+1}) - H_\infty(x, Y_{i+1}) \right] }  = 0  \qq{a.s.,}
\end{align}
where $\fB$ denote an arbitrary compact set of $\R^d$.
\end{lemma}
Its proof is in Appendix \ref{appendix: lim H uniformly convergent}.

\subsection{Equicontinuity of Scaled Iterates}\label{sec: construct sequence}

Fix a sample path $\qty{x_0,\qty{Y_n}}$. Let $\bar{x}(t)$ be the piecewise constant interpolation~\footnote{It also works if we consider a piecewise linear interpolation following \citet{borkar2009stochastic}.
The piecewise linear interpolation, however, will significantly complicate the presentation.
We, therefore, follow \citet{kushner2003stochastic} to use piecewise constant interpolation.
} of $x_n$ at points $\qty{t(n)}_{n=0,1,\dots}$, i.e.,
% \sz{Write the piecewise form explicitly first, then simplify it using $m(t)$}
\begin{align}
\bar{x}(t) \doteq  
\begin{cases}
x_0 & t \in [0,t(1)) \\
x_1 & t \in [t(1),t(2)) \\
x_2 & t \in [t(2),t(3)) \\
\vdots &
\end{cases}
\end{align}
Using \eqref{def: m} to simplify it, we get
\begin{align}\label{def: bar x}
\bar{x}(t) \doteq  x_{m(t)}.
\end{align}
Notably, $\bar x(t)$ is right continuous and has left limits.
By \eqref{eq: x n updates}, $\forall n \geq 0$, we have
\begin{align}
&\bar{x}(t(n+1)) = \bar{x}(t(n))  + \alpha(n)H(\bar{x}(t(n)) ,Y_{n+1}). 
\end{align}
Now we scale $\bar x(t)$ in each segment $[T_n, T_{n+1})$.
\begin{definition}
$\forall n \in \N, t \in [0,T) $, define 
\begin{align}\label{def: hat x}
\hat{x}(T_n + t)  &\doteq \frac{\bar{x}(T_n + t)}{r_n} 
\end{align}
where 
\begin{align}
r_n \doteq \max\qty{1, \norm{\bar{x}(T_n)}} \label{def: r n}.
\end{align}  
\end{definition}
This implies 
\begin{align}\label{eq: hat-x-norm-1}
\forall n \in \N, \norm{\hat{x}(T_n) } \leq 1.
\end{align}
Moreover\footnote{In this paper, we use the convention that $\sum_{k=i}^j \alpha(k) = 0$ when $j < i$}, 
% $\forall i, \,  T_n \leq t(i) < T_n + T$,
% we have  
% \begin{align}
% \hat{x}(t(i)) =& \frac{\bar{x}(t(i))}{r_n} \\
% =&  \frac{\bar{x}(t(i-1))   + \alpha(i-1)H(\bar{x}(t(i-1)) ,Y_{i})}{r_n} \\ 
% =& \hat{x}(t(i))  + \alpha(i-1)H_{r_n}(\hat{x}(t(i-1)) ,Y_{i}). 
% \end{align}
% \sz{You need to introduce the convention that $\sum_a^b = 0$ for $b < a$.}
$\forall n \in \N, t \in [0,T) $, 
\begin{align}
\hat{x}(T_n + t) &= \frac{\bar{x}(T_n) + \sum_{i = m(T_n)}^{m(T_n+t) - 1} \alpha(i) H(\bar{x}(t(i)), Y_{i+1})}{r_n}.\\
&= \hat{x}(T_n) + \sum_{i = m(T_n)}^{m(T_n+t) - 1} \alpha(i) H_{r_n}(\hat{x}(t(i)), Y_{i+1}).
\end{align}
The function $t \mapsto \hat x(T_n + t)$ is the scaled version of $\bar x(t)$ (by $r_n$) in the interval $[T_n, T_{n+1})$. 
Its domain is $[0, T_{n+1} - T_n)$.
In most of the rest of this work,
we will restrict it to $[0, T)$,
such that the sequence of functions $\qty{t \mapsto \hat x(T_n + t)}_{n=0,1,\dots}$ have the same domain $[0, T)$,
which is crucial in applying the Arzela-Ascoli Theorem.
The excess part $[T, T_{n+1} - T_n)$ diminishes asymptotically (cf. Lemma~\ref{lemma: T-minus}) and thus can be easily processed when necessary.
Notably, $\hat x(T_n + t)$ can be regarded as the Euler's discretization of $z_n(t)$ defined below.
\begin{definition} \label{definition: z n}
$\forall n \in \N, t\in [0,T)$, define $z_n(t)$ as the solution of the ODE 
\begin{align}
\frac{d z_n(t)}{d t} = h_{r_n}(z_n(t)) \label{def: z n}
% z_n(t)   &\doteq  \hat{x}(T_n) + \int_0^t h_{r_n}(z_n(s))ds.
\end{align}
with initial condition 
\begin{align}
z_n(0) = \hat{x}(T_n). \label{def: z n init} 
\end{align}
\end{definition}
% Since $h$ is Lipschitz continuous, by Picard-Lindelöf Theorem, $z_n(t)$ has a unique well-defined solution. Because $z_n$ is differentable, $\forall n, z_n$ is continuous on $t \in [0,T)$. 
Apparently, 
$z_n(t)$ can also be written as 
\begin{align}
z_n(t)   &=  \hat{x}(T_n) + \int_0^t h_{r_n}(z_n(s))ds. \label{def: z n integral}
\end{align}
Ideally,
we would like to see
that
the error of Euler's discretization diminishes asymptotically.
Precisely speaking,
the discretization error is defined as 
\begin{align}
f_n(t) \doteq \hat{x}(T_n + t) - z_n(t) \label{def: f}
\end{align}
and we would like that $f_n(t)$ diminishes to 0 as $n\to\infty$ in a certain sense.
To this end, we study the following three sequences of functions
\begin{align}
    \label{eq three seq}
    \qty{t \mapsto \hat x(T_n + t)}_{n=0}^\infty, \qty{z_n(t)}_{n=0}^\infty, \qty{f_n(t)}_{n=0}^\infty.
\end{align}
In particular,
we show that they are all equicontinuous in the extended sense. To understand equicontinuity in the extended sense, we first give the definition of equicontinuity.
% the scaled function $\hat x(T_n + t)$ converges to its ODE limit $z_n(t)$ as $n$ goes to $\infty$.
% To this end,
% we study their difference
% \lsz{How to point out continuous of $z_n(t)$?}
% \lsz{Define $z_n(t)$ as $\frac{z(t)}{dt} = h(z(t))$ Lipshize continuous implies solution of ODE is well-defined}

% \sz{Goes to appendix}\szstart
% We show two properties of $z_n$.
% \begin{lemma}\label{lemma: bound z}
% $\sup_{n, t \in [0, T)}   \norm{z_n(t)} < \infty$.
% \end{lemma}
% Its proof is in appendix \ref{appendix: bound z}.
% % Its proof is in appendix \ref{appendix: bound z}
% \begin{lemma}\label{lemma: bound h z}
% $\sup_{n, t \in [0, T)} \norm{h_{r_n}(z_n(t))} < \infty$.
% \end{lemma}
% Its proof is in appendix \ref{appendix: bound h z}.
% \szend

% \sz{I remember we need a common bound $\sup_n \norm{f_n(0)} < \infty$?}
% \sz{Replace $f$ with $g$ since we have a well-defined $f$.}
\begin{definition}
% [Equicontinuous in the Extended Sense]
A sequence of functions $\qty{g_n: [0, T) \to \R^K }$ is equicontinuous on $ [0, T)$ if \\
$\sup_n \norm{g_n(0)}  < \infty$ and $\forall \epsilon > 0$, $\exists \delta > 0$ such that 
\begin{align}
\sup_n \sup_{0\leq \abs{t_1-t_2} \leq \delta, \, 0 \leq t_1 \leq t_2 < T} \norm{g_n(t_1) - g_n(t_2)} \leq \epsilon.    
\end{align}
\end{definition}
One example of equicontinuity is a sequence of bounded Lipschitz continuous functions with a common Lipschitz constant.
Obviously,
if $\qty{g_n}$ is equicontinuous,
each $g_n$ must be continuous.
However,
the functions of interest in this work, 
i.e., $\hat x(T_n + t), f_n(t)$,
are not continuous  
so equicontinuity would not apply.
We, therefore,
introduce the following equicontinuity in the extended sense\footnote{
    We must use this equicontinuity in the extended sense because we have chosen to use piecewise constant instead of piecewise linear interpolation.
    For piecewise linear interpolation, the standard equicontinuity is enough.
    However,
    as also argued in \citet{kushner2003stochastic},
    piecewise linear interpolation complicates the presentation much more than the equicontinuity in the extended sense. 
}  akin to \citet{kushner2003stochastic}.
\begin{definition}\label{def: equicontinuous in the extend sense 0 T}
A sequence of functions $\qty{g_n: [0, T) \to \R^K }$ is equicontinuous in the extended sense on $[0, T)$ if $\sup_n \norm{g_n(0)}  < \infty$  and $\forall \epsilon > 0$,  $\exists \delta > 0$ such that
\begin{align}
\limsup_n \sup_{0\leq \abs{t_1-t_2} \leq \delta, \, 0 \leq t_1 \leq t_2 < T  } \norm{g_n(t_1) - g_n(t_2)} \leq \epsilon. 
\end{align}
\end{definition}
Notably,
\citet{kushner2003stochastic} show that $\qty{t \in (-\infty, \infty) \mapsto \bar x(t(n) + t) \in \R^d}_{n=0}^\infty$ is equicontinuous in the extended sense with \emph{a priori} that 
\begin{align}
    \label{eq a priori}
  \sup_{n} \norm{x_n} < \infty.
\end{align}
We do \emph{not} have this \emph{a priori}.
Instead,
we prove \emph{a posteriori} that 
\begin{align}
  \sup_{n \geq 0, t \in [0, T)} \norm{\hat x(T_n + t)} < \infty
\end{align}
and show that $\qty{t \in [0, T) \mapsto \hat x(T_n + t) \in \R^d}_{n=0}^\infty$ is equicontinuous in the extended sense.
We remark that 
our function $t \mapsto \hat x(T_n + t)$ actually belongs to the $J_1$ Skorokhod topology \citep{Skorokhod1956limit,bil1999convergence,Kern2023Skorokhod},
although we will not work on this topology explicitly.
Nevertheless,
the following lemmas establish the desired equicontinuity,
where Lemma~\ref{lemma: H h 0} plays a key role. 

\begin{lemma}\label{lemma: three functions equicontinuous}
The three sequences of functions $\qty{\hat{x}(T_n+t)}, \qty{z_n(t)},$ and $\qty{f_n(t)}$ are all equicontinuous in the extended sense on $t \in [0, T)$.
\end{lemma}
Its proof is in appendix \ref{appendix: three functions equicontinuous}.
% \sz{Merge the proof.}

% \sz{This section ends here. Comment out the below when you complete the proof.}

% \begin{lemma}\label{lemma: hat x equicontinuous}
% $\qty{\hat{x}(T_n+t)}_{n=0}^\infty$ is equicontinuous in the extended sense on $[0, T)$.    
% \end{lemma}
% Its proof is in appendix \ref{appendix: hat x equicontinuous}
% \begin{lemma}\label{lemma: z n equicontinuous}
% $\qty{z_n(t)}$ is equicontinuous on $[0, T)$.
% \end{lemma}
% Its proof is in appendix \ref{appendix: z n equicontinuous}.
% To capture the difference between $\hat{x}(T_n+t)$ and $z_n$, define a sequence of functions $\qty{f_n: [0,T) \to \R^K }_{n=0}^\infty$ where 
% \begin{align}
% \end{align}
% In this paper, for functions $\hat{x}(T_n + t), z_n(t), f_n(t)$ their domains are in  $t \in [0,T)$ and for the function $\bar{x}(T_n + t)$ its domain is in $t \in [0,T)$.
% Because $f_n$ is a linear combination of $\hat{x}$ and $z_n$, we show the following property.
% \begin{lemma}\label{lemma: f n equicontinuous}
% $\qty{f_n(t)}$ is equicontinuous in the extended sense on $[0, T)$.  
% \end{lemma}
% Its proof is in appendix \ref{appendix: f n equicontinuous}.
% \lsz{add prove change to lemma}

\subsection{A Convergent Subsequence}\label{sec: convergent sub seq}
According to the Arzela-Ascoli theorem in the extended sense (Theorem~\ref{appendix: AA theorem}), a sequence of
equicontinuous functions always has a subsequence of functions that uniformly converge to a continuous limit.
In the following, we use this to identify a particular subsequence of interest.

% \subsection{Proof by Contradiction}
% \sz{Need to advance Theorem~\ref{thm: stability} here. All the remaining is part of the proof of this theorem because we need the contradiction here.}

We observe the following inequality 
\begin{align}
\forall n, \quad \norm{x_{m(T_n)}} = \norm{\bar{x}(T_n)}  \leq  r_n. \label{eq: x m T n leq r n}
\end{align}
\color{black}
Thus, to prove Theorem \ref{thm: stability}, we first  show
\begin{align}\label{eq: sup r n bounded}
\sup_n r_n < \infty,  
\end{align}
and 
which is implied by
\begin{align}
\label{eq: contradiction r n}
\lim\sup_n r_n < \infty.
\end{align}
% because we assume $x_n$ is always well-defined.
% \color{black}
In the following, we aim to show \eqref{eq: contradiction r n} by contradiction. 
We first assume the opposite, i.e., $\limsup_n r_n = \infty$.
Based on this assumption and applying Gronwall's inequality a few times, 
we can find a particular subsequence of interest,
along which
all the three sequences of functions in~\eqref{eq three seq} converge uniformly.
% \sz{Link to its proof.}
\begin{lemma}\label{lemma: three convergence}
    Suppose $\limsup_n r_n = \infty$.
    Then there exists a subsequence $\qty{n_k}_{k=0}^\infty \subseteq \qty{0, 1, 2, \dots}$ that has the following properties:
    \begin{align}
        \lim_{k\to\infty} r_{n_k} &= \infty, \\
        \label{eq contradiction source}
        r_{n_k + 1} &> r_{n_k} \quad \forall k.
    \end{align}
    Moreover,
    there exist some continuous functions $f^{\lim}(t)$ and $\hat{x}^{\lim}(t)$ such that $\forall t \in [0, T)$,
\begin{align}
% z^{\lim}(t) &\doteq \lim_{k \to \infty}     z_{n_k}(t)  , \\ 
\lim_{k \to \infty} f_{n_k}(t) =& f^{\lim}(t)   \label{eq: f x r lim} , \\ 
\lim_{k \to \infty} \hat{x}(T_{n_k}+t) =& \hat{x}^{\lim}(t),  \label{eq: hat x x r lim}
\end{align}
where both convergences are uniform in $t$ on $[0, T)$.
Furthermore,
let $z^{\lim}(t)$ denote the unique solution to the~\eqref{eq ode at limit} with the initial condition
\begin{align}
z^{\lim}(0) = \hat{x}^{\lim}(0),
\end{align}
in other words,
\begin{align}\label{eq: z lim t}
z^{\lim}(t) &= \hat{x}^{\lim}(0) + \int_0^t h_{\infty}(z^{\lim}(s))ds.
\end{align}
Then $\forall t \in [0, T)$, we have
    \begin{align}
\lim_{k \to \infty} z_{n_k}(t)   =   z^{\lim}(t),
    \end{align}
where the convergence is uniform in $t$ on $[0, T)$.
\end{lemma}
% First, we first construct some sequences in the rest of Section \ref{sec: construct sequence}. Second, we prove useful properties of those sequences in Section \ref{sec: sequence property}. Thirdly, we demonstrate the contradiction in Section \ref{sec: explicit property}.
Its proof is in Appendix \ref{appendix: three convergence}.
We use the subsequence $\qty{n_k}$ intensively in the remaining proofs.

\subsection{Diminishing Discretization Error}\label{sec: sequence property}
Recall that $f_n(t)$ denotes the discretization error of $\hat x(T_n + t)$ of $z_n(t)$.
We now proceed to prove that this discretization error diminishes along $\qty{n_k}$.
We note that we are able to improve over \citet{borkar2021ode} because we only require the discretization error to diminish along the subsequence $\qty{n_k}$,
while \citet{borkar2021ode} aim to show that the discretization error diminishes along the entire sequence $\qty{n}$,
which is unnecessary given~\eqref{eq contradiction source}.

In particular, we aim to prove that
\begin{align}
 \lim_{k \rightarrow \infty} \norm{f_{n_k}(t)} = \norm{f^{\lim}(t)} = 0.
\end{align} 
This means $\hat{x}(T_{n_k} + t)$ is close to $z_{n_k}(t)$ as $k \rightarrow \infty$. 
For any $t \in [0, T)$,
we have
\begin{align}
&\lim_{k \rightarrow \infty} \norm{ f_{n_k}(t)} \\
=& \lim_{k \rightarrow \infty} \norm{ \hat{x}(T_{n_k}) + \sum_{i = m(T_{n_k})}^{m(T_{n_k} +t) -1 } \alpha(i)H_{r_{n_k}}(\hat{x}(t(i)),Y_{i+1}) - z_{n_k}(t)} \explain{by \eqref{def: f}}\\
=&  \lim_{k \rightarrow \infty} \norm{  \sum_{i = m(T_{n_k})}^{m(T_{n_k} +t) -1 } \alpha(i)H_{r_{n_k}}(\hat{x}(t(i)),Y_{i+1}) - \int_0^t h_{r_{n_k}}(z_{n_k}(s))ds} \explain{by \eqref{def: z n integral}}\\
% =&\lim_{k \to \infty} \norm{ \sum_{i = m(T_{n_k})}^{m(T_{n_k} +t) -1 } \alpha(i)H_{r_{n_k}}(\hat{x}(t(i)),Y_{i+1}) - \int_0^t h_{r_{n_k}}(\hat{x}^{\lim}(s))ds + \int_0^t h_{r_{n_k}}(\hat{x}^{\lim}(s))ds  - \int_0^t h_{r_{n_k}}(z_{n_k}(s))ds }\\
\leq&\lim_{k \to \infty} \norm{ \sum_{i = m(T_{n_k})}^{m(T_{n_k} +t) -1 } \alpha(i)H_{r_{n_k}}(\hat{x}(t(i)),Y_{i+1}) - \int_0^t h_{r_{n_k}}(\hat{x}^{\lim}(s))ds } \\
&+ \lim_{k \to \infty} \norm{\int_0^t h_{r_{n_k}}(\hat{x}^{\lim}(s))ds  - \int_0^t h_{r_{n_k}}(z_{n_k}(s))ds } . \label{eq: lim f k} \\
\end{align}
We now prove that the first term in the RHS of~\eqref{eq: lim f k} is 0.
Precisely speaking, we aim to prove
% To achieve our goal, we  eliminate the noise from the Markov chain in the first term of \eqref{eq: lim f k} by proving, 
$\forall t \in [0,T)$, 
\begin{align}\label{eq: single limit}
\lim_{k \to \infty} \norm{\sum_{i = m(T_{n_k})}^{m(T_{n_k} +t) -1 } \alpha(i)H_{r_{n_k}}(\hat{x}(t(i)),Y_{i+1}) - \int_0^t h_{r_{n_k}}(\hat{x}^{\lim}(s))ds} = 0.
\end{align}
% \sz{This goes to appendix}\szstart
% % To extend \eqref{eq: single limit} to a double limit, we use the following lemma.
% \begin{lemma}\label{lemma: split limit to double limit}
% For any function $f: \R \times \R \to \R$, 
% if $\lim \limits_{\begin{subarray}{l}
% a \to \infty\\
% b \to \infty
% \end{subarray}} f(a,b) = L
% $    then $\lim \limits_{c \to \infty} f(c,c) = L$ where $L$ is a constant.
% \end{lemma}
% Its proof is in Appendix \ref{appendix: split limit to double limit}.
% \szend
To compute the limit above,
we first fix any $t \in [0, T)$ and compute the following stronger double limit,
which implies the existence of the above limit (cf. Lemma~\ref{lemma: split limit to double limit}).
% Now, we compute the double limit, $\forall t$,
\begin{align}\label{eq: double limit}
\lim \limits_{\begin{subarray}{l}j \to \infty\\k \to \infty \end{subarray}}
\norm{\sum_{i = m(T_{n_k})}^{m(T_{n_k} +t) -1 } \alpha(i)H_{r_{n_j}}(\hat{x}(t(i)),Y_{i+1}) - \int_0^t h_{r_{n_j}} (\hat{x}^{\lim}(s))ds}.   
\end{align}
To compute this double limit, 
we use the Moore-Osgood theorem (Theorem \ref{appendix: Moore-Osgood Theorem}) to make it iterated limits.
To invoke the Moore-Osgood theorem,
we first prove the uniform convergence in $k$ when $j \to \infty$.
\begin{lemma}\label{lemma: double limit 1}
$\forall t \in [0,T)$,
% \lsz{Did you comment out $\forall t \in [0,T)$?}
\begin{align}
&\lim_{j \to \infty}\norm{\sum_{i = m(T_{n_k})}^{m(T_{n_k} +t) -1 } \alpha(i)H_{r_{n_j}}(\hat{x}(t(i)),Y_{i+1}) - \int_0^t h_{r_{n_j}} (\hat{x}^{\lim}(s))ds}  \\
=&\norm{\sum_{i = m(T_{n_k})}^{m(T_{n_k} +t) -1 } \alpha(i)H_{\infty}(\hat{x}(t(i)),Y_{i+1}) - \int_0^t h_{\infty} (\hat{x}^{\lim}(s))ds}	
\end{align}
uniformly in $k$. 
\end{lemma}
Its proof is in Appendix \ref{appendix: double limit 1},
where Lemma~\ref{lemma: lim H uniformly convergent} plays a key role.
Next, we prove, for each $j$, the convergence with $k\to\infty$.
% \begin{restatable}[]{lemma}{doublelimittwo}
\begin{lemma}
\label{lemma: double limit 2}
$\forall t \in [0,T)$, $\forall j$,
\begin{align}
\lim_{k \to \infty}  \norm{\sum_{i = m(T_{n_k})}^{m(T_{n_k} +t) -1 } \alpha(i)H_{r_{n_j}}(\hat{x}(t(i)),Y_{i+1}) - \int_0^t h_{r_{n_j}} (\hat{x}^{\lim}(s))ds} = 0.     
\end{align}   
\end{lemma}
% \end{restatable}
The proof of Lemma~\ref{lemma: double limit 2} 
% is intricate because we need to eliminate the Markove noise $Y_i$ in each step $i$. Thus, 
follows the proof sketch of a similar problem on page 168 of \citet{kushner2003stochastic} with some minor changes 
and is the central averaging technique of \citet{kushner2003stochastic}.
%  and decompose it into three terms.
We expect a reader familiar with \citet{kushner2003stochastic} should have belief in its correctness. 
We anyway still include all the details in the Appendix \ref{appendix: double limit 2} for completeness.
% \lsz{put its proof in Appendix C}
%  and we will use it to solve the limit in \eqref{eq: double limit}. We then use Lemma \ref{lemma: split limit to double limit} to connect  \eqref{eq: single limit} with \eqref{eq: double limit}. We propose the following lemma and its proof is in Appendix \ref{appendix: single limit}. 
We are now ready to compute the limit in~\eqref{eq: single limit}.
\begin{lemma} \label{lemma: single limit}
$\forall t\in [0,T)$,
\begin{align}
\lim_{k \to \infty} \norm{\sum_{i = m(T_{n_k})}^{m(T_{n_k} +t) -1 } \alpha(i)H_{r_{n_k}}(\hat{x}(t(i)),Y_{i+1}) - \int_0^t h_{r_{n_k}}(\hat{x}^{\lim}(s))ds} = 0.
\end{align}
\end{lemma}
\begin{proof}
    It follows immediately from Lemmas~\ref{lemma: double limit 1} \&~\ref{lemma: double limit 2},
the Moore-Osgood theorem, and Lemma~\ref{lemma: split limit to double limit}.
\end{proof}
Lemma~\ref{lemma: single limit} confirms that the first term in the RHS of~\eqref{eq: lim f k} is 0.
Moreover, it also
enables us to rewrite $\hat{x}^{\lim}(t)$ from a summation form to an integral form.
\begin{align}
&\hat{x}^{\lim}(t) \\
=& \lim_{k \to \infty} \hat{x}(T_{n_k}) + \sum_{i = m(T_{n_k})}^{m(T_{n_k} +t) -1 } \alpha(i)H_{r_{n_k}}(\hat{x}(t(i)),Y_{i+1}) \\
=& \explaind{\lim_{k \to \infty} \hat{x}(T_{n_k}) + \int_0^t h_{r_{n_k}} (\hat{x}^{\lim}(s))ds.}{by Lemma \ref{lemma: single limit}} \label{eq: x lim definition 2} 
\end{align} 
This, 
together with a few Gronwall's inequality arguments,
confirms that the discretization error indeed diminishes along $\qty{n_k}$. 

% \begin{proof}
% \begin{align}
% \hat{x}^{\lim}(t) =& \lim_{k \to \infty} \hat{x}(T_{n_k}) + \sum_{i = m(T_{n_k})}^{m(T_{n_k} +t) -1 } \alpha(i)H_{r_{n_k}}(\hat{x}(t(i)),Y_{i+1}) \\
% =& \lim_{k \to \infty} \hat{x}(T_{n_k}) + \lim_{k \to \infty}\sum_{i = m(T_{n_k})}^{m(T_{n_k} +t) -1 } \alpha(i)H_{r_{n_k}}(\hat{x}(t(i)),Y_{i+1}) \\
% =& \lim_{k \to \infty} \hat{x}(T_{n_k}) + \lim_{k \to \infty} \int_0^t h_{r_{n_k}} (\hat{x}^{\lim}(s))ds \\
% =& \lim_{k \to \infty} \hat{x}(T_{n_k}) +  \int_0^t h_{r_{n_k}} (\hat{x}^{\lim}(s))ds \\
% \end{align} 
% \end{proof}

% Now, we are ready to show the behavior of $\lim_{k \to \infty}  \norm{f_{n_k}(t)} $.
\begin{lemma}\label{lemma: f k 0}
$\forall t\in [0,T),$
\begin{align}
\lim_{k \to \infty}  \norm{f_{n_k}(t)} = 0.    
\end{align}
\end{lemma}
Its proof is in Appendix \ref{appendix: f k 0}.

% \begin{lemma}
% For a discrete random variable $X: \Omega \to \N$ 
% \begin{align}
% P(X = n) = \frac{1}{2^n},    
% \end{align}
% we have $X < \infty \, a.s.$
% \end{lemma}

\subsection{Identifying Contradiction and Completing Proof}\label{sec: explicit property}
Having made sure that the error of the discretization $\hat x(T_n + t)$ of $z_n(t)$ diminishes along $\qty{n_k}$,
we now study the behavior $\hat x(T_{n_k} + t)$ through $z_{n_k}(t)$
and identify a contradiction.
The underlying idea is identical to \citet{borkar2009stochastic}.
However,
the execution is different so we cannot use the arguments from \citet{borkar2009stochastic} directly.
Namely,
to use the arguments in Chapter 3 of \citet{borkar2009stochastic} directly,
we have to prove that the discretization error diminishes along the entire sequence.
This is impossible for us because the Arzela-Ascoli theorem only guarantees convergence along the subsequence $\qty{n_k}$.
Nevertheless,
after carefully choosing the subsequence in Lemma~\ref{lemma: three convergence},
we are still able to execute the contradiction idea
as documented below.

\begin{lemma}\label{lemma: contradiction}
Suppose $\lim\sup_n r_n = \infty$. Then there exists a $k_0$ such that
\begin{align}
    r_{n_{k_0} + 1} \leq r_{n_{k_0}}.
\end{align}
\end{lemma}
Its proof is in Appendix \ref{appendix: contradiction}.
This lemma constructs a contradiction to \eqref{eq contradiction source}.
This means the proposition $\lim\sup_n r_n = \infty$ is impossible.
This completes the proof of 
\begin{align}
\sup_n r_n < \infty \label{eq: sup r n bounded results}.  
\end{align}
By decomposition, 
\begin{align}
&\sup_n \norm{x_n}\\
% =& \sup_n \sup_{m(T_n) \leq  m(T_n) + i < m(T_{n+1}) } \norm{x_{m(T_n) + i}} \\
=& \sup_n \sup_{i \in \qty{i | m(T_n) \leq m(T_n) + i < m(T_{n+1}) }} \norm{x_{m(T_n) + i}}  -  \norm{x_{m(T_n)}} + \norm{x_{m(T_n)}}  \\
% \leq& \sup_n \sup_{i \in \qty{i | m(T_n) \leq m(T_n) + i < m(T_{n+1}) }} \norm{x_{m(T_n) + i}}  -  \norm{x_{m(T_n)}} + \sup_n \norm{x_{m(T_n)}} \\
\leq&  \explaind{\sup_n \sup_{i \in \qty{i | m(T_n) \leq m(T_n) + i < m(T_{n+1}) }} \norm{x_{m(T_n) + i}}  -  \norm{x_{m(T_n)}} + \sup_n r_n.}{by \eqref{eq: x m T n leq r n}} \label{eq: x n decomposition}
\end{align}
% After achieving \eqref{eq: sup r n bounded results}, 
We show the first term above is also bounded.
\begin{lemma}\label{lemma: x m T n + i - x m T n}
\begin{align}
\sup_n \sup_{i \in \qty{i | m(T_n) \leq m(T_n) + i < m(T_{n+1}) }} \norm{x_{m(T_n) + i}}     -  \norm{x_{m(T_n)}} < \infty.
\end{align}
\end{lemma}
Its proof is in Appendix \ref{appendix: x m T n + i - x m T n}. 
Thus,  \eqref{eq: sup r n bounded results}, \eqref{eq: x n decomposition} and Lemma \ref{lemma: x m T n + i - x m T n} conclude Theorem~\ref{thm: stability}.

\section{Applications in Reinforcement Learning}
\label{sec rl application}

In this section, we discuss broad applications of Corollary~\ref{cor: stability main} in RL.
In particular,
we both demonstrate state-of-the-art analysis in Section~\ref{sec gtd} 
and greatly simplify existing analysis in Section~\ref{sec etd}.
We first introduce notations and lay out the background of RL.

All vectors are column vectors. 
For a vector $d \in \R^N$ with strictly positive entries,
we use $\norm{x}_d$ to denote the $d$-weighted $\ell_2$ norm,
i.e., $\norm{x}_d \doteq \sqrt{\sum_{i=1}^N d_i x_i^2}$.
We also abuse $\norm{\cdot}_d$ to denote the corresponding induced matrix norm.
We use $\norm{\cdot}$ to denote a general norm that respects sub-multiplicity.
We use vectors and functions interchangeably when it does not confuse.
For example, for some $g: \fS \to \R$,
we also interpret $g$ as a vector in $\R^\ns$. 
We use $\Pi_{\Phi, d}$ to denote a projection operator that projects a vector to the column space of a matrix $\Phi$, 
assuming $\Phi$ has a full column rank.
In other words,
\begin{align}
    \Pi_{\Phi, d} v = \Phi\arg\min_{\theta} \norm{\Phi \theta - v}_d^2.
\end{align}
When it is clear from the context,
we write $\Pi_{\Phi, d}$ as $\Pi_d$ for simplifying presentation.

We consider an MDP with a finite state space\footnote{It is worth mentioning that even if the MDP problem itself is finite, the Markov chains used to analyze many RL algorithms still evolve in an uncountable and unbounded space. This will be seen shortly.} $\fS$,
a finite action space $\fA$,
a reward function $r: \fS \times \fA \to \R$,
a transition function $p: \fS \times \fS \times \fA \to [0, 1]$,
an initial distribution $p_0: \fS \to [0, 1]$,
and a discount factor $\gamma \in [0, 1)$.
At time step $0$,
an initial state $S_0$ is sampled from $p_0$.
At time $t$,
given the state $S_t$,
the agent samples an action $A_t \sim \pi(\cdot | S_t)$, 
where $\pi: \fA \times \fS \to [0, 1]$ is the policy being followed by the agent.
A reward $R_{t+1} \doteq r(S_t, A_t)$ is then emitted and the agent proceeds to a successor state $S_{t+1} \sim p(\cdot | S_t, A_t)$.
The return at time $t$ is defined as
% \begin{align}
    $G_t \doteq \sum_{i=1}^\infty \gamma^{i-1} R_{t+i}$,
% \end{align}
using which we define the state-value function $v_\pi(s)$ and action-value function $q_\pi(s)$ as
\begin{align}
    v_\pi(s) \doteq& \E_{\pi, p}\left[G_t | S_t = s\right], \\
    q_\pi(s, a) \doteq& \E_{\pi, p}\left[G_t | S_t = s, A_t = a\right].
\end{align}
The value function $v_\pi$
is the unique fixed point of the Bellman operator
\begin{align}
    \bop_\pi v \doteq r_\pi + \gamma P_\pi v,
\end{align}
where $r_\pi \in \R^\ns$ is the reward vector induced by the policy $\pi$, i.e., $r_\pi(s) \doteq \sum_a \pi(a|s) r(s, a)$,
and $P_\pi \in \R^{\ns \times \ns}$ is the transition matrix induced by the policy $\pi$, i.e.,
$P_\pi(s, s') \doteq \pi(a|s)p(s'|s, a)$.
% The Bellman operator $\bop_\pi$ can be generalized using a $\lambda \in [0, 1]$ as
% \begin{align}
    % \bop_{\pi, \lambda} v \doteq r_\pi + \gamma P_\pi ((1 - \lambda) v + \lambda \bop_{\pi, \lambda} v )
% \end{align}
With a $\lambda \in [0, 1]$,
we can rewrite $v_\pi = \bop_\pi v_\pi$ using the identity $v_\pi = (1 - \lambda) v_\pi + \lambda \bop_\pi v_\pi$ as
\begin{align}
    v_\pi =& r_\pi + \gamma P_\pi( (1 - \lambda) v_\pi + \lambda \bop_\pi v_\pi) \\
    =& r_\pi + \gamma (1 - \lambda) P_\pi v_\pi + \gamma \lambda P_\pi (r_\pi + \gamma P_\pi v_\pi ) \\
    =& r_\pi + \gamma \lambda P_\pi r_\pi + \gamma (1-\lambda) P_\pi v_\pi + \gamma^2 \lambda  P_\pi^2 ((1 - \lambda) v_\pi + \lambda \bop_\pi v_\pi) \\
    =& \dots \\
    =& \sum_{i=0}^\infty (\gamma \lambda P_\pi)^i r_\pi + (1 - \lambda)\sum_{i=1}^\infty \lambda^{i-1} \gamma^i P_\pi^i v_\pi, \\
    =& (I - \gamma \lambda P_\pi)^{-1} r_\pi + (1 - \lambda) \gamma (I - \gamma \lambda P_\pi)^{-1} P_\pi v_\pi.
\end{align}
This suggests that we define a $\lambda$-Bellman operator as
\begin{align}
    \bop_{\pi, \lambda} v \doteq r_{\pi, \lambda} + \gamma P_{\pi, \lambda} v,
\end{align}
where
% \begin{align}
    $r_{\pi, \lambda} \doteq (I - \gamma \lambda P_\pi)^{-1} r_\pi, 
    P_{\pi, \lambda} \doteq (1 - \lambda) (I - \gamma \lambda P_\pi)^{-1} P_\pi.$
% \end{align}
It is then easy to see that
when $\lambda = 0$,
$\bop_{\pi, \lambda}$ reduces to $\bop_\pi$.
When $\lambda = 1$,
$\bop_{\pi, \lambda}$ reduces to a constant function that always output $(I - \gamma P_\pi)^{-1} r_\pi$. 
It is proved that $\bop_{\pi, \lambda}$ is a $\frac{\gamma(1 - \lambda)}{1 - \gamma \lambda}$-contraction w.r.t. $\norm{\cdot}_{d_\pi}$ (see, e.g., Lemma 6.6 of \citep{bertsekas1996neuro}), 
where we use $d_\pi \in \R^\ns$ to denote the stationary distribution of the Markov chain induced by $\pi$.
Obviously, $v_\pi$ is the unique fixed point of $\bop_{\pi, \lambda}$.

One fundamental task in RL is prediction,
i.e.,
to estimate $v_\pi$,
for which temporal difference (TD, \citet{sutton1988learning}) learning is the most powerful method.
In particular,
\citet{sutton1988learning} considers a linear architecture.
Let $\phi: \fS \to \R^K$ be the feature function that maps a state to a $K$-dimensional feature.
Linear TD($\lambda$) \citep{sutton1988learning} aims to find a $\theta \in \R^K$ such that $\phi(s)^\top \theta$ is close to $v_\pi(s)$ for every $s \in \fS$.
To this end,
linear TD($\lambda$) updates $\theta$ recursively as
\begin{align}
    \label{eq on policy td}
    e_t =& \lambda \gamma e_{t-1} + \phi_t, \\
    \theta_{t+1} =& \theta_t + \alpha_t \left(R_{t+1} + \gamma \phi_{t+1}^\top \theta_t  - \phi_t^\top \theta_t \right) e_t,
\end{align}
where we have used $\phi_t \doteq \phi(S_t)$ as shorthand and $e_t \in \R^K$ is the \emph{eligiblity trace} with an arbitrary initial $e_{-1}$.
We use $\Phi \in \R^{\ns \times K}$ to denote the feature matrix, 
each row of which is $\phi(s)^\top$.
It is proved \citep{tsitsiklis1997analysis} that,
under some conditions,
$\qty{\theta_t}$ converges to the unique zero of 
% \begin{align}
    $J_\text{on}(\theta) \doteq \norm{\Pi_{d_\pi} \bop_{\pi, \lambda} \Phi \theta - \Phi \theta}^2_{d_\pi}$.
% \end{align}
% where $\Pi_{d_\pi}$ is shorthand for $\Pi_{\Phi, d_\pi}$ 
% \begin{align}
    % \Pi_{d_\pi} v \doteq \Phi \arg\min_{\theta} \norm{\Phi\theta - v}_{d_\pi}^2
% \end{align}
This $J_\text{on}(\theta)$ is referred to as the on-policy mean squared projected Bellman error (MSPBE).

In many scenarios,
due to the concerns of data efficiency \citep{lin1992self,sutton2011horde} or safety \citep{dulac2019challenges},
we would like to estimate $v_\pi$ but select actions using a different policy, called $\mu$.
This is off-policy learning,
where $\pi$ is called the target policy and $\mu$ is called the behaivor policy. 
In the rest of this section,
we always consider the off-policy setting,
i.e.,
the action $A_t$ is sampled from $\mu(\cdot | S_t)$.
Correspondingly,
off-policy linear TD($\lambda$) updates $\theta$ recursively as
\begin{align}
    \label{eq off policy td}
    e_t =& \lambda \gamma \rho_{t-1} e_{t-1} + \phi_t, \\
    \theta_{t+1} =& \theta_t + \alpha_t \rho_t \left(R_{t+1} + \gamma \phi_{t+1}^\top \theta_t  - \phi_t^\top \theta_t \right) e_t,
\end{align}
where $\rho_t \doteq \rho(S_t, A_t) \doteq \frac{\pi(A_t | S_t)}{\mu(A_t | S_t)}$ is the importance sampling ratio to account for the discrepancy in action selection between $\pi$ and $\mu$.
Obviously,
if $\pi = \mu$,
then~\eqref{eq off policy td} reduces to~\eqref{eq on policy td}. 
Let $d_\mu \in \R^\ns$ be the stationary distribution of the Markov chain induced by $\mu$.
If $\qty{\theta_t}$ in~\eqref{eq off policy td} converged,
it would converge to the unique zero of 
\begin{align}
    J_\text{off}(\theta) \doteq \norm{\Pi_{d_\mu} \bop_{\pi, \lambda} \Phi \theta - \Phi \theta}^2_{d_\mu},
\end{align}
which is the off-policy MSPBE.

\subsection{Eligibility Trace}
\label{sec trace}

The eligibility trace is one of the most fundamental ingredients in RL 
and is deeply rooted in RL since the very beginning of RL \citep{klopf1972brain,sutton1978single,barto1981goal,barto1981landmark,barto1983neuronlike,sutton1984temporal}. 
The eligibility trace in~\eqref{eq on policy td} is called the accumulating trace,
first introduced in \citet{barto1981goal}.
Later on, this trace is also used in control by \citet{rummery1994line}.
Its off-policy version in~\eqref{eq off policy td} is introduced by \citet{precup2000eligibility,precup2001off} and further developed by \citet{bertsekas2009projected,yu2012LSTD}.
Other forms of traces include the Dutch trace introduced by \citet{seijen2014true} and the followon trace introduced by \citet{sutton2016emphatic}.
In short,
traces are usually used to accelerate credit assignment,
which is a fundamental challenge in RL.
Intuitively, traces are able to achieve this goal because they function as memory of the past.
Empirically,
RL algorithms with traces usually outperform those without traces \citep{sutton2018reinforcement}.
Traces are also important in establishing the equivalence between backward and forward views of RL algorithms \citep{sutton2014new}.

Despite the superiority of traces in multiple aspects,
they usually complicate the analysis of RL algorithms.
Without any trace,
to analyze an RL algorithm it is usually sufficient to consider the Markov chain $\qty{(S_t, A_t)}$.
Under a finite MDP assumption,
this augmented Markov chain is still finite.
Once trace is introduced,
we, however,
must consider the Markov chain $\qty{(S_t, A_t, e_t)}$,
see, e.g., \citet{tsitsiklis1997analysis}.
This augmented Markov chain now immediately evolves in an uncountable space $\fS \times \fA \times \R^d$.
In the on-policy case (cf.~\eqref{eq on policy td}),
this is still managable.
It is clear from~\eqref{eq on policy td} that $e_t$ remains bounded almost surely.
So the augmented Markov chain evolves in a compact space.
In the off-policy case (cf.~\eqref{eq off policy td}),
the trace $e_t$ can easily be unbounded almost surely due to the importance sampling ratio $\rho_{t-1}$ \citep{yu2012LSTD}.
The augmented Markov chain then evolves in an \emph{unbounded and uncountable} space.
Even worse,
sometimes the second moment of $e_t$ can also be unbounded \citep{yu2012LSTD},
further complicating the analysis.
Despite that $e_t$ is demonstrated to obey a form of the strong law of large numbers \citep{yu2012LSTD},
there does not exist a general tool to make use of this in convergence analysis before this work.
In other words,
this work is the first to provide a general tool to analyze the stability (and thus convergence) of RL algorithms with off-policy traces.

\subsection{The Deadly Triad}
Despite the aforementioned superiority of off-policy learning in safety and data efficiency,
it complicates RL algorithms in at least two aspects.
The first is that it makes traces extremely hard to analyze,
as demonstrated in the section above.
Second,
it makes the RL algorithm behaves poorly in expectation.
In other words,
even if there is no noise (cf. replacing $H(x_n, Y_{n+1})$ with $h(x_n)$),
the RL algorithm can still behave poorly.
A concrete example is that,
% Unfortunately,
for a general $\lambda$,
the iterates $\qty{\theta_t}$ in~\eqref{eq off policy td} can possibly diverge to infinity,
as documented in \citet{baird1995residual,tsitsiklis1997analysis,sutton2018reinforcement}.
% the possible divergence of $\qty{\theta_t}$ in~\eqref{eq off policy td} is well documented (see, e.g., ).
This is the notorious \emph{deadly triad},
which refers to the instability of an RL algorithm when it combines bootstrapping,
function approximation,
and off-policy learning 
simultaneously
while maintaining a constant $\fO(K)$ computational complexity each step.

The deadly triad has been one of the central challenges of RL in the past three decades and numerous works have been done in this topic
\citep{precup2000eligibility, precup2001off, sutton2009convergent, sutton2009fast, maei2009convergent, maei2010toward, maei2010gq, maei2011gradient, sutton2011horde, yu2012LSTD, mahadevan2014proximal, liu2015finite, yu2015convergence, white2016investigating, mahmood2017multi, yu2017convergence, wang2017finite, touati2018convergent, liu2018breaking, zhang2019provably, nachum2019dualdice, xu2019two, zhang2020average, zhang2020learning, ghiassian2020gradient, wang2020finite, zhang2020gradientdice, guan2021per, zhang2021breaking, zhang2021truncated, qian2025revisiting,liu2025linear}.
We refer the reader to Chapter 11 of \citet{sutton2018reinforcement} and \citet{zhang2022thesis} for more detailed exposition.

Among all those works, 
gradient temporal difference learning (GTD, \citet{sutton2009convergent}) and emphatic temporal difference learning (ETD, \citet{sutton2016emphatic}) are the two most important solutions to the deadly triad in terms of policy evaluation.
GTD and ETD are also important building blocks for other algorithms.
They can be used in convergent off-policy actor-critic algorithms for control, see, e.g., \citet{imani2018off,maei2018convergent,zhang2019provably,xu2021doubly,graves2023off}.
They can also be used to learn value functions w.r.t. some augmented reward function to construct behavior policies for efficient and unbiased Monte Carlo policy evaluation,
see, e.g., \citet{liu2024efficient,liu2024efficientmul,chen2024efficient,liu2024doubly}.
But surprisingly,
the convergence analysis of
their ultimate form with eligibility trace, i.e., GTD($\lambda$) and ETD($\lambda$),
is still not fully settled down.
In the next,
we shall analyze GTD($\lambda$) and ETD($\lambda$) in the sequel.
Throughout the rest of Section~\ref{sec rl application},
we make the following assumptions.
\begin{xassumption}
    \label{assumption rl chain}
    Both $\fS$ and $\fA$ are finite.
    The Markov chain $\qty{S_t}$ induced by the behavior policy $\mu$ is irreducible.
    And $\mu(a|s) > 0$ for all $s, a$.
\end{xassumption}
We note again that in light of Section~\ref{sec trace},
even if the MDP itself is finite,
the augmented Markov chain used to analyze GTD($\lambda$) and ETD($\lambda$) still evolves in an unbounded and uncountable space.
The analysis is, therefore, very challenging. 
Assumption~\ref{assumption rl chain} is a standard assumption in off-policy RL to ensure enough exploration, see, e.g., \citet{precup2001off,sutton2016emphatic}. 
The condition $\mu(a|s) > 0$ can be easily relaxed to $\pi(a|s) > 0 \implies \mu(a|s) > 0$,
at the price of complicating the presentation.
\begin{xassumption}
\label{assumption rl lr}
The learning rates $\qty{\alpha_t}$ have the form $\alpha_t = \frac{B_1}{t + B_2}$.
%  satisfy Assumption~\ref{assumption: alpha rate} and further $\alpha_t = \fO\qty(\frac{1}{t})$.
\end{xassumption}
Assumption~\ref{assumption rl lr} is also used in existing works,
see, e.g., \citet{yu2012LSTD,yu2015convergence,yu2017convergence}.
\begin{xassumption}
\label{assumption rl feature}
The feature matrix $\Phi$ has a full column rank.
\end{xassumption}
Assumption~\ref{assumption rl feature} is a standard assumption in RL with linear function approximation to ensure the existence and uniqueness of the solution, see, e.g., \citet{tsitsiklis1997analysis}.

% \subsection{$Q$-Learning}
% \label{sec q learning}
% \begin{align}
%     A_t \sim \mu(\cdot | S_t).
% \end{align}
% \begin{align}
%     q_{t+1}(S_t, A_t) = q_t(S_t, A_t) + \alpha_t \left(R_{t+1} + \gamma \max_{a'} q_t(S_{t+1}, a') - q_t(S_t, A_t)\right)
% \end{align}
% Previous works usually require state-action dependent learning rates $\alpha_t(s, a)$.

% \subsection{On-Policy Temporal Difference Learning}
% \label{sec on policy td}

% \begin{align}
%     e_t =& \gamma \lambda e_{t-1} + \phi_t, \\
%     \theta_{t+1} =& \theta_t + \alpha_t \left(R_{t+1} + \gamma \phi_{t+1}^\top \theta_t  - \phi_t^\top \theta_t \right) e_t.
% \end{align}
% \sz{Try to prove law of large number applies for $(S_t, A_t, S_{t+1}, e_t)$ without assuming the finiteness of $\fS$.
% This would be a better proof than \citet{tadic2001convergence}}

% \subsection{Linear Temporal  Difference Learning}
% To analyze~\eqref{eq off policy td},

% The off-policy linear TD~\eqref{eq off policy td} can then be expressed as
% \begin{align}
%     \theta_{t+1} = \theta_t + \alpha_t H(\theta_t, Y_{t+1}),
% \end{align}
% where
% \begin{align}

% \end{align}

\subsection{Gradient Temporal Difference Learning}
\label{sec gtd}
The idea of GTD is to perform stochastic gradient descent on $J_\text{off}(\theta)$ directly and
use a weight duplication trick or Fenchel's duality to address a double sampling issue in estimating $\nabla J_\text{off}(\theta)$.
We refer the reader to \citet{sutton2009fast,liu2015finite} for detailed derivation.
GTD has many different variants, see, e.g., \citet{sutton2009convergent,sutton2009fast,maei2011gradient,yu2017convergence,zhang2020average,qian2025revisiting}.
In this paper,
we present and analyze the following arguably most representative one, 
referred to as GTD($\lambda$) for simplicity.\footnote{This is the GTDa in \citet{yu2017convergence} and is the GTD2 in \citet{sutton2009fast} with eligibility trace.}
In particular, GTD($\lambda$) employs an additional weight vector $\nu \in \R^K$ and update $\theta$ and $\nu$ simultaneously in a recursive way as
\begin{align}
    \label{eq gtd}
    e_t =& \lambda \gamma \rho_{t-1} e_{t-1} + \phi_t, \\
    \delta_t =& R_{t+1} + \gamma \phi_{t+1}^\top \theta_t - \phi_t^\top \theta_t, \\
    \nu_{t+1} =& \nu_t + \alpha_t \left(\rho_t \delta_t e_t - \phi_t \phi_t^\top \nu_t\right), \\
    \theta_{t+1} =& \theta_t + \alpha_t\rho_t (\phi_t - \gamma \phi_{t+1}) e_t^\top \nu_t.
\end{align}
This additional weight vector results from the weight duplication or Fenchel's duality.
To analyze~\eqref{eq gtd},
we first express the update to $\nu$ and $\theta$ in a compact form as
\begin{align}
    \mqty[\nu_{t+1} \\ \theta_{t+1}] = \mqty[\nu_t \\ \theta_t] + \alpha_t \left(\mqty[-\phi_t\phi_t^\top & \rho_t e_t(\gamma \phi_{t+1} - \phi_t)^\top \\ -(\gamma \phi_{t+1} - \phi_t) \rho_t e_t^\top & 0]\mqty[\nu_{t} \\ \theta_{t}] + \mqty[\rho_t R_{t+1} e_t \\ 0]\right).
\end{align}
To further simplify it,
we define an augmented Markov chain $\qty{Y_t}$ as
\begin{align}
    Y_{t+1} \doteq (S_t, A_t, S_{t+1}, e_t), \quad t=0,1,\dots.
\end{align}
We also define shorthands
\begin{align}
x \doteq& \mqty[\nu \\ \theta], x_t \doteq \mqty[\nu_t \\ \theta_t],  \\
y \doteq& (s, a, s', e), \\
A(y) \doteq& \rho(s, a) e (\gamma \phi(s') - \phi(s))^\top,  \label{def: A y}\\
b(y) \doteq& \rho(s, a) r(s, a) e, \label{def: b y}\\
C(y) \doteq& \phi(s) \phi(s)^\top, \label{def: C y}\\
H(x, y) \doteq& \mqty[-C(y) & A(y) \\ -A(y)^\top & 0] x + \mqty[b(y) \\ 0].
\end{align}
% \lsz{repeated $\kappa$, change later}
Then GTD($\lambda$) can be expressed as
\begin{align}\label{eq: GTD update}
x_{t+1} = x_t + \alpha_t H(x_t, Y_{t+1}),
\end{align}
which reduces to the form of~\eqref{eq: x n updates}.
We now proceed to prove the almost sure convergence of $\qty{x_t}$ using Corollary~\ref{cor: stability main}.
Apparently,
$\qty{Y_t}$ evolves in the state space
\begin{align}
    \fY \doteq \fS \times \fA \times \fS \times \R^K.
\end{align}
Despite that both $\fS$ and $\fA$ are finite, $\fY$ 
can still be unbounded and uncountable.
It is shown in Proposition 3.1 of \citet{yu2012LSTD} that
as long as there is a cycle in $\qty{S_t}$,
$e_t$ is unbounded almost surely in arguably almost all natural problems.
Nevertheless,
\citet{yu2012LSTD} shows that $\qty{Y_t}$ has the following property.
\begin{lemma} (Theorems 3.2 \& 3.3 of \citet{yu2012LSTD})
\label{lemma: yu invariant measure}
Let Assumption~\ref{assumption rl chain} hold. Then
\begin{enumerate}[(i)]
    \item $\qty{Y_t}$ has a unique invariant probability measure, referred to as ${d_\fY}$ .
    \item For any matrix/vector-valued function $g(s, a, s', e)$ on $\fY$ which is Lipschitz continuous in $e$ with a Lipschitz constant $L_g$, i.e.,
    \begin{align}
        \norm{g(s, a, s', e) - g(s, a, s', e')} \leq L_g \norm{e - e'}, \quad \forall s, a, s', e, e',
    \end{align}
    the expectation $\E_{y\sim{d_\fY}}\left[g(y)\right]$ exists and is finite,
    and
    the~\eqref{eq lln new} holds for the $g$ function.
\end{enumerate}
\end{lemma}
\citet{yu2012LSTD} also shows that
\begin{align}
    A \doteq& \E_{y\sim{d_\fY}}\left[ A(y)\right] = \Phi^\top D_\mu (\gamma P_{\pi, \lambda} - I) \Phi, \\
    b \doteq& \E_{y\sim{d_\fY}}\left[ b(y)\right] = \Phi^\top D_\mu r_{\pi, \lambda}, \\
    C \doteq& \E_{y\sim{d_\fY}}\left[ C(y)\right] = \Phi^\top D_\mu \Phi,
\end{align}
where we use $D_\mu$ to denote the diagonal matrix whose diagonal entry is $d_\mu$.

\begin{theorem}\label{theorem: GTD convergence}
\label{thm gtd}
Let Assumptions~\ref{assumption rl chain} - \ref{assumption rl feature} hold.
Assume $A$ is nonsingular.
Then the iterates $\qty{\theta_t}$ generated by GTD($\lambda$)~\eqref{eq gtd} satisfy
\begin{align}
    \lim_{t\to\infty} \theta_t = -A^{-1} b \qq{a.s.}
\end{align}
\end{theorem}
Its proof is in Appendix \ref{appendix: GTD convergence}.
It can be shown easily that $-A^{-1}b$ is the unique zero of $J_\text{off}(\theta)$, see, e.g., \citet{sutton2009fast}.
Notably, Theorem~\ref{thm gtd} is the first almost sure convergence analysis of GTD with eligibility trace without adding additional bias terms.
Most existing convergence analyses of GTD (see, e.g., \citet{sutton2009convergent,sutton2009fast,maei2011gradient,liu2015finite,wang2017finite,qian2025revisiting}) do not have eligibility trace.
To our knowledge,
the only previous analysis of GTD with eligibility trace is \citet{yu2017convergence},
which, however,
relies on additional projection operators or regularization to ensure the stability
and unavoidably introduces bias into the final limiting point.
As a result,
\citet{yu2017convergence} cannot establish the almost sure convergence of GTD($\lambda$) to the unique zero of $J_\text{off}(\theta)$.
\citet{yu2017convergence} also introduces extensions to $\lambda$.
Instead of being a constant,
it can be a state-dependent function $\lambda: \fS \to [0, 1]$.
The almost sure convergence of GTD($\lambda$) with a state-dependent $\lambda$ function follows similarly.
We present the simplest constant $\lambda$ case for clarity.
\citet{yu2017convergence} also introduces history-dependent $\lambda$ function,
which we leave for future work.

\subsection{Emphatic Temporal Difference Learning}
\label{sec etd}
The idea of ETD is to reweight the off-policy linear TD update~\eqref{eq off policy td} by an additional factor.
Similar to GTD,
ETD also has many different variants,
see, e.g.,
\citet{yu2015convergence,sutton2016emphatic,hallak2016generalized,zhang2019provably,zhang2021truncated,guan2021per}.
Variants of ETD have also been applied in deep RL,
see, e.g., \citet{jiang2021,DBLP:journals/corr/abs-2107-05405,mathieu2023alphastar}.
In this section,
we consider the original ETD($\lambda$) in \citet{yu2015convergence,sutton2016emphatic}.
ETD($\lambda$) updates $\theta$ recursively in the following way
\begin{align}
    \label{eq etd}
    F_t =& \gamma \rho_{t-1} F_{t-1} + i(S_t), \\
    M_t =& \lambda i(S_t) + (1 - \lambda) F_t, \\
    e_t =& \lambda \gamma \rho_{t-1} e_{t-1} + M_t \phi_t, \\
    \theta_{t+1} =& \theta_t + \alpha_t \rho_t \left(R_{t+1} + \gamma \phi_{t+1}^\top \theta_t  - \phi_t^\top \theta_t \right) e_t,
\end{align}
where $i: \fS \to (0, \infty)$ is an arbitrary ``interest'' function \citep{sutton2016emphatic},
specifying user's preference for different states,
despite that in most applications, $i(s)$ is a constant function which is always 1.
See \citet{zhang2019generalized} for an example where the interest function is not trivially 1.
Comparing the eligibility trace $e_t$ in~\eqref{eq etd} with that in~\eqref{eq off policy td},
one can find that there is an additional scalar multiplier $M_t$ proceeding $\phi_t$.
This $M_t$ is called ``emphasis'' \citep{sutton2016emphatic},
which is the accumulation of $F_t$, called ``followon trace'' \citep{sutton2016emphatic}.
We refer the reader to \citet{sutton2016emphatic} for the intuition behind ETD.
Nevertheless,
\citet{yu2015convergence} proves that, under mild conditions, $\qty{\theta_t}$ in~\eqref{eq etd} converges almost surely to the unique zero of 
\begin{align}
    J_\text{emphatic}(\theta) = \norm{\Pi_m \bop_{\pi, \lambda} \Phi \theta - \Phi \theta}_{m}^2,
\end{align}
where
% \begin{align}
    $m \doteq (I - \gamma P_{\pi, \lambda}^\top)^{-1} D_\mu i$. 
% \end{align}
We remark that the zero of $J_\text{emphatic}(\theta)$ has better theoretical guarantees than the zero of $J_\text{off}(\theta)$ in terms of the approximation error for $v_\pi$ \citep{hallak2016generalized}.
ETD, however,
usually suffers from a larger variance than GTD \citep{sutton2018reinforcement}.

To analyze ETD($\lambda$),
\citet{yu2015convergence} considers the following augmented Markov chain
\begin{align}
    Y_{t+1} =& (S_t, A_t, S_{t+1}, e_t, F_t).
\end{align}
Again,
$\qty{Y_t}$ behaves poorly in that $(e_t, F_t)$ can be unbounded almost surely and its variance can grow to infinity as time progresses.
We refer the reader to Remark A.1 in \citet{yu2015convergence} for an in-depth discussion regarding this poor behavior.
Nevertheless, \citet{yu2015convergence} shows that $\qty{Y_t}$ has the following property.
\begin{lemma}
    \label{lem yu etd}
    (Theorems 3.2 \& 3.3 of \citet{yu2015convergence})
    Let Assumption~\ref{assumption rl chain} hold. Then
    \begin{enumerate}[(i)]
        \item $\qty{Y_t}$ has a unique invariant probability measure, referred to as ${d_\fY}$.
        \item For any matrix / vector-valued function $g(s, a, s', e, f)$ on $\fY$ which is Lipschitz continuous in $(e, f)$ with a Lipschitz constant $L_g$, i.e.,
        \begin{align}
            \norm{g(s, a, s', e, f) - g(s, a, s', e', f')} \leq L_g \norm{e - e'}, \quad \forall s, a, s', e, e', f, f',
        \end{align}
        the expectation $\E_{y\sim{d_\fY}}\left[g(y)\right]$ exists and is finite,
        and
        the~\eqref{eq lln new} holds for the function $g$.
    \end{enumerate}
\end{lemma}
We now discuss how \citet{yu2015convergence} establishes the almost sure convergence of $\qty{\theta_t}$.
First, we define shorthands
\begin{align}
    y \doteq& (s, a, s', e, f), \\
A(y) =& \rho(s, a) e (\gamma \phi(s') - \phi(s))^\top, \\
b(y) =& \rho(s, a) r(s, a) e, \\
H(\theta, y) =& A(y)\theta + b(y).
\end{align}
Then the ETD($\lambda$) update can be expressed as
\begin{align}
    \theta_{t+1} = \theta_t + \alpha_t H(\theta_t, Y_{t+1}).
\end{align}
\citet{yu2015convergence} also shows that
\begin{align}
    A \doteq& \E_{y\sim{d_\fY}}\left[A(y)\right] = \Phi^\top D_m (\gamma P_{\pi, \lambda} - I) \Phi, \\
    b \doteq& \E_{y\sim{d_\fY}}\left[b(y)\right] = \Phi^\top D_m r_{\pi, \lambda},
\end{align}
and $-A^{-1}b$ is the unique zero of $J_\text{emphatic}(\theta)$.
Despite that $A$ is negative definite (see, e.g., Section 4 of \citet{sutton2016emphatic}) and 
the corresponding ODE$@\infty$ is, therefore, globally asymptotically stable,
\citet{yu2015convergence} is not able to establish the stability of $\qty{\theta_t}$ directly,
simply because the results in the stochastic approximation community are not ready yet.
See Section~\ref{sec related work} for a comprehensive review.
As a workaround,
\citet{yu2015convergence} analyzes a constrained variant of ETD($\lambda$) first:
\begin{align}
    \theta'_{t+1} = \Pi\left(\theta'_t + \alpha_t H(\theta'_t, Y_{t+1})\right),
\end{align}
where $\Pi$ is a projection to a centered ball of properly chosen radius w.r.t. $\ell_2$ norm.
\citet{yu2015convergence} then proves that the difference between $\qty{\theta_t}$ and $\qty{\theta_t'}$ diminishes almost surely
and therefore establishes the convergence of $\qty{\theta_t}$ indirectly.
To establish the convergence of $\qty{\theta_t'}$,
\citet{yu2015convergence} invokes Theorem~1.1 in Chapter 6 of \citet{kushner2003stochastic}.
Now with our Corollary~\ref{cor: stability main},
the same arguments \citet{yu2015convergence} use to invoke \citet{kushner2003stochastic} can lead to the convergence of $\qty{\theta_t}$ directly.
Our contribution is,
therefore,
a greatly simplified almost sure convergence analysis of ETD$(\lambda)$.
In particular, we have
\begin{theorem}
    \label{thm etd}
    Let Assumptions~\ref{assumption rl chain} - \ref{assumption rl feature} hold. Then the iterates $\qty{\theta_t}$ generated by ETD($\lambda$)~\eqref{eq etd} satisfy
    \begin{align}
        \lim_{t\to\infty} \theta_t = -A^{-1} b \qq{a.s.}
    \end{align}
\end{theorem}
The proof of Theorem~\ref{thm etd} is a verbatim repetition of the proof of Theorem~\ref{thm gtd} in Appendix \ref{appendix: GTD convergence} after noticing that $A$ is negative definite and Lemma~\ref{lem yu etd} and is thus omitted.
Notably, this proof does not involve the comparison between $\qty{\theta_t}$ and $\qty{\theta_t'}$.

We remark that the comparison technique between $\qty{\theta_t}$ and $\qty{\theta_t'}$ used by \citet{yu2015convergence} heavily relies on the fact that $A$ is negative definite (see Lemma 4.1 of \citet{yu2015convergence}).
But in GTD$(\lambda)$,
the corresponding matrix is 
$\mqty[-C & A \\ -A^\top & 0]$,
which is Hurwitz but not negative definite.
In fact,
it is only negative semidefinite.
As a result,
the comparison technique in \citet{yu2015convergence} does not apply to GTD($\lambda$).

\section{Conclusion}
In this work,
we develop a novel stability result of stochastic approximations,
extending the celebrated Borkar-Meyn theorem from the Martingale difference noise setting to the Markovian noise setting.
Our result is built on the diminishing asymptotic rate of change of a few functions,
which is implied by both a form of the strong law of larger numbers and a form of the law of the iterated logarithm.
%  lawthe Lyapunov drift condition~\eqref{eq v4}.
We demonstrate the wide applicability of our results in RL,
generating state-of-the-art analysis for important RL algorithms in breaking the notorious deadly triad.
There are many possible directions for future work.
One direction is to characterize the behavior of the iterates in~\eqref{eq: x n updates} in more aspects.
For example,
it is possible to establish a (functional) central limit theorem following \citet{borkar2021ode}.
It is also possible to establish an almost sure convergence rate, a high probability concentration bound, and an $L^p$ convergence rate following \citet{qian2024almost}.
Another direction is to weaken the required assumptions further.
In the context of RL,
Assumption~\ref{assumption: lim h uniformly convergent} is typically obtained by assuming $h$ is related to some contraction operator and the feature matrix $\Phi$ has a full column rank. 
It is possible to weaken $h$ to nonexpansive operators following \citet{blaser2024almost}.
It is also possible to allow $\Phi$ to have arbitrary ranks following \citet{wang2024almost}.

\acks{The authors thank Vagul Mahadevan for the detailed proofreading. The authors thank the anonymous reviewers for their highly constructive comments. This work is supported in part by the US National Science Foundation under grants III-2128019 and SLES-2331904.}

% \newpage
\appendix
\clearpage
\tableofcontents

%%%%%%%%%%%%%%%%%%%%%%%%%%%%%%%%%%%%%%%%%%%%%%%%%%%%%%%%%%%%
% APPENDIX
%%%%%%%%%%%%%%%%%%%%%%%%%%%%%%%%%%%%%%%%%%%%%%%%%%%%%%%%%%%%

\section{Mathematical Background}
% \sz{If it is a restatement, specify the exact place.}
% \sz{Combine the two Gronwall together.}

% \subsection{Gronwall Inequality}\label{appendix: Gronwall}
\begin{xtheorem}[\textbf{Gronwall Inequality}] (Lemma 6 in Section 11.2 in \citet{borkar2009stochastic})\label{theorem: Gronwall}
For a continuous function $u(\cdot) \geq 0$ and scalars $C,K,T \geq 0$,
\begin{align}
u(t) \leq C+K\int_0^t u(s) ds \quad \forall t \in [0,T]
\end{align}
implies
\begin{align}
u(t) \leq Ce^{tK}, \forall t \in [0,T].    
\end{align}
\end{xtheorem}

\begin{xtheorem}[\textbf{Gronwall Inequality in the Reverse Time}]\label{theorem: Gronwall reverse}
% [Gronwall Inequality in the Opposite Direction]
For a continuous function $u(\cdot) \geq 0$ and scalars $C,K,T \geq 0$,
\begin{align}
u(t) \leq C + K\int_t^0 u(s) ds \quad \forall t \in [-T,0]    \label{eq: reverse Gronwall condtion}
\end{align}
implies
\begin{align}
u(t) \leq Ce^{-tK}, \forall t \in [-T,0].    
\end{align}
\end{xtheorem}

\begin{proof}
$\forall s \in [0,T]$, define 
\begin{align}
v(s) \doteq e^{sK} K \int_s^0 u(r) dr. \label{eq: continuous Gronwall def v reverse}  
\end{align}
Taking the derivative of $v(s)$,
\begin{align}
v'(s) &=   -e^{sK} K u(s) + e^{sK} K^2 \int_s^0 u(r) dr \\
&=  e^{sK} K \left[ - u(s) +  K \int_s^0 u(r) dr \right] \explain{by \eqref{eq: reverse Gronwall condtion}} \\
&\geq  -C e^{sK} K . 
\end{align}
Thus,
\begin{align}
v(t)  =& v(0) - \int_t^0 v'(s) ds 
% \leq  v(0) - \int_t^0 -C e^{sK} K ds  
\leq  v(0) + \int_t^0 C e^{sK} K ds  
= K C \int_t^0 e^{sK}    ds.
\end{align}
By \eqref{eq: continuous Gronwall def v reverse},
\begin{align}
K \int_t^0 u(s) ds &= v(t) e^{-tK} \\
&\leq  K C \int_t^0 e^{sK}   ds  e^{-tK} \\
&\leq  K C \int_t^0 e^{(s-t)K}   ds   \\
&= K C [ \frac{1}{k} e^{(0-t)K} - \frac{1}{k} e^{(t-t)K} ] \\
&= -C +  C e^{-tK}. \label{eq: reverse Gronwall 2}    
\end{align} 
Thus,
\begin{align}
u(t) \leq& C + K\int_t^0 u(s) ds 
% \leq& C -C +  C e^{-tK} \explain{by \eqref{eq: reverse Gronwall 2}}\\
\leq C e^{-tK}. 
\end{align}

\end{proof}

% \subsection{Discrete Gronwall Inequality}\label{appendix: discrete Gronwall}
\begin{xtheorem}[\textbf{Discrete Gronwall Inequality}] (Lemma 8 in Section 11.2 in \citet{borkar2009stochastic})\label{theorem: discrete Gronwall}
For non-negative sequences $\qty{x_n,n\geq 0}$ and $\qty{a_n, n\geq 0}$ and scalars $C,L \geq 0$,
\begin{align}
x_{n+1} \leq C + L\sum_{i=0}^n a_i x_i \quad \forall n
\end{align}
implies
\begin{align}
x_{n+1} \leq Ce^{L\sum_{i=0}^n a_i }  \quad \forall n.  
\end{align}
\end{xtheorem}

% \subsection{The Arzela-Ascoli Theorem in the Extended Sense}\label{appendix: AA theorem}
\begin{xtheorem}
[The Arzela-Ascoli Theorem in the Extended Sense on $[0,T)$]
\label{appendix: AA theorem}
% [Restatement of Theorem 2.2 in Sec. 4 in \citet{kushner2003stochastic}]
Let \\
$\qty{t \in [0, T) \mapsto g_n(t)}$ be equicontinuous in the extended sense. 
Then,
there exists a subsequence $\qty{g_{n_k}(t)}$ that converges to some continuous limit $g^{\lim}(t)$, uniformly in $t$ on $[0, T)$.
\end{xtheorem}
The proof of the Arzela-Ascoli Theorem can be found in any standard analysis textbook, see, e.g., \citet{royden1968real,dunford1988linear}.
The proof of the Arzela-Ascoli Theorem in the extended sense is virtually the same.
The difference is that in the standard Arzela-Ascoli Theorem, 
one uses the compactness to find a finite subcover.
But in the extended one, $[0, T)$ is not compact.
However, finding a finite cover for this specific set $[0, T)$ is indeed trivial.
We anyway still include the full proof below for completeness.
\begin{proof}
% In Section 4 of \citet{kushner2003stochastic},
% the equicontinuity in the extended sense is defined for functions with the domain $(-\infty, \infty)$.
% In this paper, we work with functions with the domain $[0, T)$.
% The proof of this theorem is virtually the same as the proof of Theorem 2.2 in Section 4 of \citet{kushner2003stochastic},
% as well as the proof of the original version of the Arzela-Ascoli theorem \sz{cite}.
Fix an arbitrary $\epsilon > 0$, by Definition \ref{def: equicontinuous in the extend sense 0 T}, $\exists \delta > 0$ such that
\begin{align}
\limsup_n \sup_{0\leq \abs{t_1-t_2} \leq \delta, \, 0 \leq t_1 \leq t_2 < T  } \norm{g_n(t_1) - g_n(t_2)} \leq \epsilon. \label{eq: AA 0 T epsilon}
\end{align}
This means by the definition of equicontinuity in the extended sense, when $n$ is large enough, for any $0\leq \abs{t_1-t_2} \leq \delta$, the function values $g_n(t_1)$ and $g_n(t_2)$ are also close. 
To conveniently utilize this property, we divide $[0,T)$ into a set of disjoint intervals and each interval has a length $\delta$ such that the $t$ in each interval is close.
% To achieve this division,
% divide $[0,T)$ into a set of disjoint intervals,
In particular,
we define 
\begin{align}
N &\doteq \max \qty{i \mid i \delta < T, \, i \in \mathbb{Z}}, \\
I_i &\doteq [i \delta, (i + 1)\delta), \quad i=0,1,\dots, N.
\end{align}
% \sz{I don't think you really need to define$N$. Using sth like $\lfloor T/\delta \rfloor$ should be enough.}
% \lsz{There is a subtlety here. It is possible that 
% \begin{align}
% \delta \cdot \lfloor T/\delta \rfloor = T.
% \end{align}
% This will have some dividing point $t = T$ and, thus, jump out of the domain $[0,T)$.
% $N$ avoids this.
% }
The set of intervals $\qty{I_i}_{i = 0}^{N}$ covers the domain $[0, T)$,
\begin{align}
[0,T) \subseteq \bigcup_{i = 0}^{N}  I_i.    
\end{align}
We now show $g_n(t)$ is uniformly bounded uniformly on the set of dividing points
$\qty{i\delta}_{i = 0}^{N}$.
In particular, 
we have for any $i \in \qty{0, 1, \dots, N},$
\begin{align}
&\limsup_n \norm{g_n(i\delta)} \label{eq: limsup sup bounded} \\
% \leq& \limsup_n \norm{g_n(t) - \sum_{i = 0} ^{(t/\delta) - 1} ( g_n(i \delta) - g_n(i \delta))} \\
\leq& \limsup_n \norm{g_n(i\delta) - g_n((i-1)\delta)} \\
&+\limsup_n \norm{g_n((i-1)\delta) - g_n((i-2)\delta)} \\
&+\dots \\
&+\limsup_n \norm{g_n(\delta) - g_n(0)} \\
&+\limsup_n \norm{g_n(0)} \\
\leq& (N+1) \epsilon + \limsup_n \norm{g_n(0)} \explain{by \eqref{eq: AA 0 T epsilon}} \\
\leq&  (N+1) \epsilon + \sup_n \norm{g_n(0)} \\
<& \infty  \explain{$\sup_n \norm{g_n(0)} < \infty$ in Definition \ref{def: equicontinuous in the extend sense 0 T}}.
\end{align}
This implies 
% $\forall i \in $,
% $\forall t \in \qty{i\cdot \delta}_{i = 0}^{N}$ 
% \sz{Using $\forall i$ looks more straightforward.}, 
\begin{align}
\sup_{i \in \qty{0,1, \dots, N}, n \geq 0} \norm{g_n(i \delta)}  < \infty.
\end{align}
By the Bolzano-Weierstrass theorem, 
there exists a subsequence of functions $\qty{g_{n_{0, k}}}$ in $\qty{g_{n}}$ such that $\qty{g_{n_{0, k}}(0 \cdot \delta)}$ converges. 
% \sz{Why not $n_{0, k}$ following the previous convention?}
Repeating the same argument for the sequence of points $\qty{g_{n_{0, k}}(1 \cdot \delta)}$, 
there exists a subsequence $\qty{g_{n_{1, k}}}$ of $\qty{g_{n_{0, k}}}$ such that $\qty{g_{n_{1, k}}(1 \cdot \delta)}$ converges. 
Repeating this process, because $N$ is finite, 
there exists a subsequence $\qty{g_{n_k}}$ that  
converges at all dividing points $t \in \qty{i\delta}_{i = 0}^{N}$. 
% Use $\qty{g_{n^{N}_k}}$ to denote the $k$th function in the subsequence  $\qty{g_{n^{N}_k}}$. 
% Thus, 
% \sz{Emphasize the existence comes from finiteness.}
Due to the finiteness of $N$, 
$\exists k_0$, such that  $\forall i \in \qty{0, 1, \dots, N}$, $\forall k_1 \geq k_0, \forall k_2 \geq k_0$, we have
\begin{align}
\norm{g_{n_{k_1}}(i\delta) - g_{n_{k_2}}(i\delta)} \leq \epsilon. \label{eq: AA k 1 k 2 epsilon}
\end{align}
By \eqref{eq: AA 0 T epsilon}, $\exists k_1$ such that $\forall k \geq k_1$,
\begin{align}
&\sup_{0\leq \abs{t_1-t_2} \leq \delta, \, 0 \leq t_1 \leq t_2 < T  } \norm{g_{n_k}(t_1) - g_{n_k}(t_2)}  
% \\
% \leq& \left \vert \sup_{0\leq \abs{t_1-t_2} \leq \delta, \, 0 \leq t_1 \leq t_2 < T  } \norm{g_{n_k}(t_1) - g_{n_k}(t_2)} \right. \\
% & \left. -  \limsup_n \sup_{0\leq \abs{t_1-t_2} \leq \delta, \, 0 \leq t_1 \leq t_2 < T  } \norm{g_{n_{N}, n}(t_1) - g_{n_{N}, n}(t_2)} \right \vert \\
% & + \limsup_n \sup_{0\leq \abs{t_1-t_2} \leq \delta, \, 0 \leq t_1 \leq t_2 < T  } \norm{g_{n_{N}, n}(t_1) - g_{n_{N}, n}(t_2)} \\
\leq 2\epsilon.\label{eq: AA k t 1 t 2 2 epsilon}
\end{align}
Thus, $\forall t \in [0, T)$, $\forall k \geq \max\qty{k_0, k_1}, \forall k' \geq  \max\qty{k_0, k_1}$,
\begin{align}
&\norm{g_{n_k}(t) - g_{n_{k'}}(t)} \\
\leq& \norm{g_{n_k}(t) - g_{n_k}(\lfloor t / \delta \rfloor \cdot \delta)} +\norm{g_{n_k}(\lfloor t / \delta \rfloor \cdot \delta) - g_{n_{k'}}(\lfloor t / \delta \rfloor \cdot \delta)} \\
&+ \norm{g_{n_{k'}}(\lfloor t / \delta \rfloor \cdot \delta) - g_{n_{k'}}(t)}  \\
\leq& 2\epsilon + \norm{g_{n_k}(\lfloor t / \delta \rfloor \cdot \delta) - g_{n_{k'}}(\lfloor t / \delta \rfloor \cdot \delta)} + 2\epsilon \explain{by \eqref{eq: AA k t 1 t 2 2 epsilon}}\\
\leq& 2\epsilon + \epsilon + 2\epsilon \explain{by \eqref{eq: AA k 1 k 2 epsilon}}\\
=& 5 \epsilon.
\end{align}
This shows that the sequence $\qty{g_{n_k}}$  is uniformly Cauchy and therefore uniformly converges to a continuous function.

% From \eqref{eq: limsup sup bounded}, we know 
% \begin{align}
% \limsup_n \lim_{t \to T^- } \norm{g_n(t)} < 0.
% \end{align}
% We now extend the domain of $g_n(t)$ from $[0,T)$ to $(-\infty, \infty)$.
% $\forall n,$ define  
% \begin{align}
% g^+_n(t) \doteq 
% \begin{cases}
% g_n(0) & \text{ if } t < 0, \\  
% g_n(t) & \text{ if } 0 \leq t < T, \\   
% \limsup_n \lim_{t \to T^- } \norm{g_n(t)} & \text{ if } t \geq T. \\   
% \end{cases}
% \end{align}
% We now show $g^+_n(t)$ on $(-\infty, \infty)$ is equicontinuous in the extended sense by showing \eqref{eq: equicontinuity inf inf epsilon} in Defintion \ref{cor-def: equicontinuity broad}. Because the domain on $(-\infty, 0)$ and $[T, \infty)$ is an constant extension, we only discuss one case that is
% $\forall \tau \geq T$, $\forall \epsilon > 0$,  $\exists \delta > 0$,
% \begin{align}
% \limsup_n \sup_{0\leq \abs{t_1-t_2} \leq \delta, \,
% % \abs{t_1} \leq \tau,\abs{t_2} \leq \tau, 
% 0< t_1 < T, \, T \leq t_2 \leq \tau } \norm{g_n(t_1) - g_n(t_2)} \leq \epsilon.   
% \end{align}

\end{proof}

% \sz{This property is obvious. Remove it.}
% \subsection{Property of Equicontinuous}\label{appendix: subsequence equicontinuous theorem}
% \begin{theorem}
% If $\qty{f_n}$ is equicontinuous, any subsequence of it is also equicontinuous.
% \end{theorem}

% \sz{The following is standard. Remove it.}
% \subsection{Bounded Convergence Theorem}\label{appendix: interchange integral and limit}
% \begin{theorem}
% If  $f_n$ is bounded uniformly in $n$ and $f_n \to f$, then 
% \begin{align}
% \lim_{n\to \infty} \int_a^b f_n    = \int_a^b \lim_{n\to \infty}  f_n 
% \end{align}
% \end{theorem}

% \subsection{Moore-Osgood Theorem for Interchanging Limits }
\begin{xtheorem}[Moore-Osgood Theorem for Interchanging Limits]\label{appendix: Moore-Osgood Theorem}
If $\lim_{n\to \infty} a_{n,m} = b_m$ uniformly in $m$ and $\lim_{m\to \infty} a_{n,m} = c_n$ for each large $n$, then both $\lim_{m \to \infty}b_m$ and $\lim_{n\to\infty}c_n$ exists and are equal to the double limit, i.e., 
\begin{align}
\lim_{m\to\infty}   \lim_{n\to \infty} a_{n,m} =   \lim_{n\to \infty} \lim_{m\to \infty} a_{n,m} = \lim \limits_{\begin{subarray}{l}
n \to \infty\\
m \to \infty
\end{subarray}} a_{n,m}.
\end{align}
\end{xtheorem}

% \subsection{$\sup$ and $\limsup$ relationship}\label{appendix: sup limsup equivalence}
% \begin{theorem}
% Let $\qty{X_i}_{i=1}^n$ be a sequence of random variables.
% For any function $f: \fX \to \fY$,
% \begin{align}
% \sup_n f(X_n) < \infty \quad a.s.  \\
% \Longleftrightarrow \limsup_n f(X_n) < \infty \quad a.s. 
% \end{align}

% \end{theorem}

% Bounded convergence theorem

\section{Technical Proofs}

\subsection{Proof of Lemma~\ref{lemma: lim H uniformly convergent}}\label{appendix: lim H uniformly convergent}
% \label{appendix: H c H infty}

\begin{proof}
Let Assumptions~\ref{assumption: stationary distribution},~\ref{assumption: alpha rate},~\ref{assumption: specific H c H infinity}, and~\ref{assumption: H Lipschitz} hold.
Fix an arbitrary sample path $\qty{x_0,\qty{Y_i}_{i=1}^\infty}$.
Use $\fB$ to denote an arbitrary compact set of $x$.
\begin{align}
&\lim_{c \to \infty} \sup_{x \in \fB} \sup_n \sup_{t \in [0,T]} \norm{ \sum_{i=m(T_n)}^{m(T_n+t) - 1} \alpha(i) \left[H_c(x, Y_{i+1}) - H_\infty(x, Y_{i+1}) \right] }  \\     
=&\lim_{c \to \infty} \sup_{x \in \fB} \sup_n \sup_{t \in [0,T]} \norm{\sum_{i=m(T_n)}^{m(T_n+t) - 1} \alpha(i) \kappa(c) b(x, Y_{i+1})} \explain{by \eqref{eq: h c h inf kappa}} \\
=&\lim_{c \to \infty}  \kappa(c)   \sup_{x \in \fB} \sup_n\sup_{t \in [0,T]} \norm{\sum_{i=m(T_n)}^{m(T_n+t) - 1} \alpha(i) b(x, Y_{i+1})}  \\
=& 0   \sup_{x \in \fB} \sup_n \sup_{t \in [0,T]} \norm{\sum_{i=m(T_n)}^{m(T_n+t) - 1} \alpha(i) b(x, Y_{i+1})} \label{eq: H c H inf 0 finite} 
\end{align}
We now show that the function
\begin{align}
\label{eq a function of x}
x \mapsto \sup_n \sup_{t \in [0,T]} \norm{\sum_{i=m(T_n)}^{m(T_n+t) - 1} \alpha(i) b(x, Y_{i+1})}
\end{align}
is Lipschitz continuous.
$\forall x, x'$,
% \lsz{explain sup minus}
\begin{align}
&\abs{\sup_n \sup_{t \in [0,T]} \norm{\sum_{i=m(T_n)}^{m(T_n+t) - 1} \alpha(i) b(x, Y_{i+1})} - \sup_n \sup_{t \in [0,T]} \norm{\sum_{i=m(T_n)}^{m(T_n+t) - 1} \alpha(i) b(x', Y_{i+1})}} \\
\leq&\sup_n \sup_{t \in [0,T]} \abs{\norm{\sum_{i=m(T_n)}^{m(T_n+t) - 1} \alpha(i) b(x, Y_{i+1})} - \norm{\sum_{i=m(T_n)}^{m(T_n+t) - 1} \alpha(i) b(x', Y_{i+1})}} \explain{by $\abs{\sup_x f(x) - \sup_x g(x)} \leq \sup_x \abs{f(x) - g(x)} $} \\
\leq&\sup_n \sup_{t \in [0,T]} \norm{\sum_{i=m(T_n)}^{m(T_n+t) - 1} \alpha(i) b(x, Y_{i+1}) - \sum_{i=m(T_n)}^{m(T_n+t) - 1} \alpha(i) b(x', Y_{i+1})} \\
\leq&\sup_n \sup_{t \in [0,T]} \sum_{i=m(T_n)}^{m(T_n+t) - 1} \alpha(i) \norm{ b(x, Y_{i+1}) - b(x', Y_{i+1})} \\
\leq& \sup_n \sup_{t \in [0,T]} \left(\sum_{i=m(T_n)}^{m(T_n+t) - 1} \alpha(i) L_b(Y_{i+1})\right) \norm{x - x'} \explain{by \eqref{eq: L b Lipschitz}} 
\end{align}
Additionally, let Assumption~\ref{assumption: lln} or~\ref{assumption possion} hold. By Lemma~\ref{lemma: H h 0} and \eqref{eq: L b property bounded},
\begin{align}
\sup_n \sup_{t \in [0,T]} \left(\sum_{i=m(T_n)}^{m(T_n+t) - 1} \alpha(i) L_b(Y_{i+1})\right) < \infty
\end{align}
can be viewed as the Lipschitz constant.
Thus, \eqref{eq a function of x} is a continuous function. 
Since $\fB$ is compact,
the extreme value theorems asserts that the supremum of~\eqref{eq a function of x} in $\fB$ is attainable at some $x_{\fB}$ and is finite.
This means
the RHS of~\eqref{eq: H c H inf 0 finite} is 0,
% which completes the proof.
\begin{align}
&\lim_{c \to \infty} \sup_{x \in \fB} \sup_n \sup_{t \in [0,T]} \norm{ \sum_{i=m(T_n)}^{m(T_n+t) - 1} \alpha(i) \left[H_c(x, Y_{i+1}) - H_\infty(x, Y_{i+1}) \right] }  = 0. \label{eq: b e b 0}
\end{align}

\end{proof}

\subsection{Proof of Lemma~\ref{lemma: three functions equicontinuous}}\label{appendix: three functions equicontinuous}
% \subsection{Proof of Lemma~\ref{lemma: hat x equicontinuous}}\label{appendix: hat x equicontinuous}
\begin{proof}
By \eqref{eq: hat-x-norm-1},
\begin{align}
\sup_n \norm{\hat{x}(T_n + 0)} \leq 1   . 
\end{align}
$\forall \xi > 0$, by \eqref{eq: h property close t 0},  $\exists \delta_0,$ such that $\forall 0< \delta \leq \delta_0$,
\begin{align}
\sup_{c\geq 1} \limsup_n \sup_{0\leq t_2-t_1 \leq \delta } \norm{ \sum_{i = m(T_n+t_1)}^{m(T_n+t_2) - 1} \alpha(i) H_c(0, Y_{i+1}) } \leq \xi.  \label{eq: lemma hat x equicontinuous xi 1}   
\end{align}
By \eqref{eq: L property close t 0},  $\exists \delta_1,$ such that $\forall 0< \delta \leq \delta_1$,
\begin{align}
 \limsup_n \sup_{0\leq t_2-t_1 \leq \delta }  \sum_{i = m(T_n+t_1)}^{m(T_n+t_2) - 1} \alpha(i) L(Y_{i+1}) \leq \xi.   \label{eq: lemma hat x equicontinuous xi 2} 
\end{align}
Without loss of generality, 
let $t_1 \leq t_2$.
Then$\forall \delta \leq \min \qty{\delta_0,\delta_1} $,  we have
\begin{align}
&\limsup_n \sup_{ 0\leq t_2-t_1 \leq \delta } \norm{ \hat{x}(T_n + t_1) - \hat{x}(T_n + t_2)}  \\
=& \limsup_n \sup_{0\leq  t_2-t_1 \leq \delta } 
 \norm{\sum_{i = m(T_n+t_1)}^{m(T_n+t_2) - 1} \alpha(i) H_{r_n}(\hat{x}(t(i)), Y_{i+1}) }\\
% =& \limsup_n \sup_{0\leq  t_2-t_1 \leq \delta }  \norm{\sum_{i = m(T_n+t_1)}^{m(T_n+t_2) - 1} \alpha(i) H_{r_n}(\hat{x}(t(i)), Y_{i+1}) } -  \norm{\sum_{i = m(T_n+t_1)}^{m(T_n+t_2) - 1} \alpha(i) H_{r_n}(0, Y_{i+1}) } \\
% &+  \norm{\sum_{i = m(T_n+t_1)}^{m(T_n+t_2) - 1} \alpha(i) H_{r_n}(0, Y_{i+1}) }  \\
\leq&  \limsup_n \sup_{0\leq  t_2-t_1 \leq \delta } \norm{\sum_{i = m(T_n+t_1)}^{m(T_n+t_2) - 1} \alpha(i) H_{r_n}(\hat{x}(t(i)), Y_{i+1}) } -  \norm{\sum_{i = m(T_n+t_1)}^{m(T_n+t_2) - 1} \alpha(i) H_{r_n}(0, Y_{i+1}) } \\
&+ \limsup_n \sup_{0\leq  t_2-t_1 \leq \delta } \norm{\sum_{i = m(T_n+t_1)}^{m(T_n+t_2) - 1} \alpha(i) H_{r_n}(0, Y_{i+1}) }  \\
\leq&  \limsup_n \sup_{0\leq  t_2-t_1 \leq \delta } \norm{\sum_{i = m(T_n+t_1)}^{m(T_n+t_2) - 1} \alpha(i) H_{r_n}(\hat{x}(t(i)), Y_{i+1}) } -  \norm{\sum_{i = m(T_n+t_1)}^{m(T_n+t_2) - 1} \alpha(i) H_{r_n}(0, Y_{i+1}) } \\
&+ \sup_{c\geq 1} \limsup_n \sup_{0\leq  t_2-t_1 \leq \delta } \norm{\sum_{i = m(T_n+t_1)}^{m(T_n+t_2) - 1} \alpha(i) H_c(0, Y_{i+1}) }  \\
\leq&  \limsup_n \sup_{0\leq  t_2-t_1 \leq \delta } \norm{\sum_{i = m(T_n+t_1)}^{m(T_n+t_2) - 1} \alpha(i) H_{r_n}(\hat{x}(t(i)), Y_{i+1}) } -  \norm{\sum_{i = m(T_n+t_1)}^{m(T_n+t_2) - 1} \alpha(i) H_{r_n}(0, Y_{i+1}) } \\
&+ \xi  \explain{by \eqref{eq: lemma hat x equicontinuous xi 1}}\\
% \end{align}
% 
% 
% \begin{align}
\leq&  \limsup_n \sup_{0\leq  t_2-t_1 \leq \delta }  \norm{\sum_{i = m(T_n+t_1)}^{m(T_n+t_2) - 1} \alpha(i) H_{r_n}(\hat{x}(t(i)), Y_{i+1})  -  \sum_{i = m(T_n+t_1)}^{m(T_n+t_2) - 1} \alpha(i) H_{r_n}(0, Y_{i+1}) } \\
&+   \xi \\
\leq& \limsup_n \sup_{0\leq  t_2-t_1 \leq \delta } \sum_{i = m(T_n+t_1)}^{m(T_n+t_2) - 1} \alpha(i)  \norm{H_{r_n}(\hat{x}(t(i)), Y_{i+1}) -  H_{r_n}(0, Y_{i+1})  } +   \xi\\
\leq& \limsup_n  \sup_{0\leq  t_2-t_1 \leq \delta }  \sum_{i = m(T_n+t_1)}^{m(T_n+t_2) - 1} \alpha(i)  L(Y_{i+1})\norm{\hat{x}(t(i))}  +    \xi \\
\leq&  C_{\hat{x}} \limsup_n \sup_{0\leq  t_2-t_1 \leq \delta }  \sum_{i = m(T_n+t_1)}^{m(T_n+t_2) - 1} \alpha(i)  L(Y_{i+1})  +     \xi \explain{by Lemma~\ref{lemma: bound x hat}}  \\
\leq&  C_{\hat{x}}  \xi  +     \xi, \explain{by \eqref{eq: lemma hat x equicontinuous xi 2}}
\end{align}
which implies that
% Thus, $\forall \epsilon >0$, define $\xi = \frac{\epsilon}{C_{\hat{x}} + 1}$, for the corresponding $\delta = \min \qty{\delta_0,\delta_1}$, we have
% \begin{align}
% \limsup_n \sup_{ 0\leq t_2-t_1 \leq \delta } \norm{ \hat{x}(T_n + t_1) - \hat{x}(T_n + t_2)} \leq  (C_{\hat{x}} +  1)  \xi   = \epsilon.
% \end{align}
$\qty{\hat{x}(T_n+t)}$  is equicontinuous in the extended sense.    \\
% \end{proof}
% \subsection{Proof of Lemma~\ref{lemma: z n equicontinuous}}\label{appendix: z n equicontinuous}
% \begin{proof}
For $\qty{z_n(t)}$, by \eqref{eq: hat-x-norm-1} and \eqref{def: z n init}, we have
\begin{align}
\sup_n \norm{z_n(0)} \leq 1. 
\end{align}
Without loss of generality,
let $t_1 \leq t_2$.
Then $\forall \delta > 0$, we have
\begin{align}
& \sup_n \sup_{0\leq \abs{t_1-t_2} \leq \delta, \, 0 \leq t_1 \leq t_2 < T}  \norm{z_n(t_1) - z_n(t_2)}  \\
=&\sup_n \sup_{0\leq \abs{t_1-t_2} \leq \delta, \, 0 \leq t_1 \leq t_2 < T} \norm{ \int_{t_1}^{t_2} h_{r_n}(z_n(s))ds }\\
=&  \sup_n \sup_{0\leq \abs{t_1-t_2} \leq \delta, \, 0 \leq t_1 \leq t_2 < T}  \norm{\int_{t_1}^{t_2} \left[ h_{r_n}(z_n(s))  -  h_{r_n}(0) \right]ds + \int_{t_1}^{t_2} h_{r_n}(0) ds }& \\
\leq &\sup_n \sup_{0\leq \abs{t_1-t_2} \leq \delta, \, 0 \leq t_1 \leq t_2 < T}   \int_{t_1}^{t_2}\norm{h_{r_n}(z_n(s))  -  h_{r_n}(0) }ds + \sup_n \sup_{0\leq \abs{t_1-t_2} \leq \delta, \, 0 \leq t_1 \leq t_2 < T}  \int_{t_1}^{t_2} \norm{ h_{r_n}(0) }ds & \\
\leq &   \sup_n \sup_{0\leq \abs{t_1-t_2} \leq \delta, \, 0 \leq t_1 \leq t_2 < T} \int_{t_1}^{t_2}L\norm{z_n(s)} ds + \sup_n \sup_{0\leq \abs{t_1-t_2} \leq \delta, \, 0 \leq t_1 \leq t_2 < T} \int_{t_1}^{t_2} \norm{ h_{r_n}(0) }ds&  \explain{by Lemma \ref{lemma: h Lipschitz}} \\
\leq & \delta LC_{\hat{x}}  + \sup_n \sup_{0\leq \abs{t_1-t_2} \leq \delta, \, 0 \leq t_1 \leq t_2 < T} \int_{t_1}^{t_2} \norm{ h_{r_n}(0) }ds& 
\explain{by Lemma~\ref{lemma: bound z}} \\
\leq & \delta (L C_{\hat{x}}  + C_H),& 
\explain{by \eqref{eq: h c c_H}}
\end{align}
which implies that
% 
% 
% Thus, $\forall \epsilon >0$, define $\delta = \frac{\epsilon}{LC_{\hat{x}} + \frac{C_H}{T}}$, 
% \begin{align}
% \sup_n \sup_{0\leq  t_2-t_1 \leq \delta }  \norm{z_n(t_1) - z_n(t_2)}\leq   \epsilon.
% \end{align}
% 
$\qty{z_n(t)}$ is equicontinuous.    \\
% \end{proof}
For $\qty{f_n(t)}$, we have
% \subsection{Proof of Lemma~\ref{lemma: f n equicontinuous}}\label{appendix: f n equicontinuous}
% \begin{proof}
\begin{align}
\sup_n f_n(0) = \sup_n  \hat{x}(T_n) - z_n(0) =   \sup_n  \hat{x}(T_n) - \hat{x}(T_n) = 0 < \infty.
\end{align}
% By Lemma \ref{lemma: hat x equicontinuous} and Lemma \ref{lemma: z n equicontinuous}, 
Because $\qty{\hat{x}(T_n+t)}$ and $\qty{z_n(t)}$ are equicontinuous,
$\forall \epsilon > 0$, $\exists \delta$ such that 
\begin{align}
\limsup_n \sup_{0\leq  t_2-t_1 \leq \delta} \norm{\hat{x}(T_n + t_1)  - \hat{x}(T_n + t_2)} & \leq 
\frac{\epsilon}{2},\\
\sup_n \sup_{0\leq  t_2-t_1 \leq \delta} \norm{ z_n(t_1) - z_n(t_2)} &\leq 
\frac{\epsilon}{2}. 
\end{align}
Without loss of generality let $t_1 \leq t_2$.
Then $\forall \epsilon$, $\exists \delta$ such that
\begin{align}
&\limsup_n \sup_{0\leq  t_2-t_1 \leq \delta }  \norm{f_n(t_1) - f_n(t_2)}  \\
=&\limsup_n  \sup_{0\leq  t_2-t_1 \leq \delta } \norm{\hat{x}(T_n + t_1)  - \hat{x}(T_n + t_2) - (z_n(t_1) - z_n(t_2))} \\ 
\leq& \limsup_n \sup_{0\leq  t_2-t_1 \leq \delta } \norm{\hat{x}(T_n + t_1)  - \hat{x}(T_n + t_2)} + \limsup_n \sup_{0\leq  t_2-t_1 \leq \delta } \norm{ z_n(t_1) - z_n(t_2)} \\ 
\leq& \limsup_n \sup_{0\leq  t_2-t_1 \leq \delta } \norm{\hat{x}(T_n + t_1)  - \hat{x}(T_n + t_2)} + \sup_n \sup_{0\leq  t_2-t_1 \leq \delta } \norm{ z_n(t_1) - z_n(t_2)} \\ 
\leq& \epsilon,
\end{align}
which implies that
$\qty{f_n}$ is equicontinuous in the extended sense.     
\end{proof}

\subsection{Proof of Lemma~\ref{lemma: three convergence}}\label{appendix: three convergence}

\begin{proof}
% \sz{Move this to the appendix and reconcile the reference to equations and lemmas.}
% \lsz{Should I put it in Appendix B as a subsection and put a few subsubsection for the lemma proofs in this proof?}
% \sz{Just make the auxiliary lemmas.}
We can construct a subsequence $\qty{r_{n_{1, k}}}$ that diverges to infinity and satisfies $\forall k$,  $\forall n < n_{1, k}$, 
\begin{align}
    \label{eq wkr property}
r_n< r_{n_{1, k}}.
\end{align}
For example,
we can define
\begin{align}
n_{1, 0} &\doteq 1\\
n_{1, k} &\doteq \min \qty{n \mid n>n_{1, k-1}, r_n > r_{n_{1, k-1}} + 1 } \label{eq: wrk}. 
\end{align}
Because  $\limsup_n r_n = \infty$, we know $\forall k>0,  \qty{n \mid n>n_{1, k-1}, r_n > r_{n_{1, k-1}} + 1} \neq \emptyset$. Because $\forall k > 0$, $ r_{n_{1, k}} -  r_{n_{1, k-1}} > 1$,
\begin{align}\label{eq: r w r k infty}
\lim_{k \to \infty} r_{n_{1, k}}  = \infty. 
\end{align}
Because~\eqref{eq: wrk} defines $n_{1, k}$ to be the first index that is large enough after $n_{1, k-1}$,
~\eqref{eq wkr property} holds.
Otherwise $n_{1, k}$ would not be the first.
Define a sequence $\qty{n_{2, k}}$ as 
\begin{align}\label{def: w r' k}
n_{2, k} \doteq n_{1, k} - 1 \quad \forall  k.
\end{align}
We make two observations.
First, $n_{2, k}$ and $n_{1, k}$ are neighbors
so $r_{n_{2, k}}$ and $r_{n_{1, k}}$ correspond to $\bar x(T_n)$ and $\bar x(T_{n+1})$ for some $n$.
Second, by Lemma \ref{lemma: bar x bounded},
the increment of $\bar x(t)$ in $[T_n, T_{n+1})$ is bounded in the following sense
$\forall n,$ 
\begin{align}
\norm{ \bar{x}(T_{n+1})} \leq \left(\norm{ \bar{x}(T_n)}  C_H  + C_H \right) e^{C_H} + \norm{ \bar{x}(T_n)} 
\end{align}
where $C_H$ is a positive constant.
This means that if $r_{n_{2, k}}$ is not large enough,
$r_{n_{1, k}}$ will not be large enough either.
We can then prove by contradiction in Lemma \ref{lemma: r w r' k infty} that
\begin{align}
\lim\sup_k r_{n_{2, k}} = \infty. \label{eq: r w r' k infty}
\end{align}
Thus, using the similar method as \eqref{eq: wrk}, we can construct a subsequence $\qty{n_{3, k}}$ from $\qty{n_{2, k}}$ such that
\begin{align}
\lim_k r_{n_{3, k}} = \infty. \label{eq: r w r'' k infty}
\end{align}
Moreover,
since $\qty{n_{3, k} + 1}$ is a subsequence of $\qty{n_{1, k}}$,
~\eqref{eq wkr property} implies that
\begin{align}
    r_{n_{3, k}} < r_{n_{3, k} + 1}.
\end{align}
Since $\qty{f_n}$ is equicontinuous in the extended sense, 
% any subsequence of it is also equicontinuous in the extended sense.
% By Lemma \ref{lemma: f n equicontinuous} and Theorem \ref{appendix: subsequence equicontinuous theorem},
% \begin{align}
% \qty{n_{3, k}} \subseteq \qty{n_{2, k}}
% \end{align}
$\qty{f_{n_{3, k}}}_{k = 0,1,\dots}$ is also equicontinuous in the extended sense. By the Arzela-Ascoli Theorem (Theorem~\ref{appendix: AA theorem}), it has a uniformly convergent subsequence, 
referred to as $\qty{f_{n_{4, k} }}$. 
% $\qty{z_{n_k^f}}$ is also equicontinuous. Thus, it has a uniformly convergent subsequence, 
% referred to as $\qty{f_{n^{z}_k }}$.
Because the sequence $\qty{\hat{x}(T_{n_{4, k}} + t)}$ is also equicontinuous in the extended sense, it has a uniformly convergent subsequence $\qty{\hat{x}(T_{n_k} + t)}$. To summarize,
\begin{align}\label{eq: all uniformly convergent}
\qty{n_k} \subseteq    \qty{n_{4, k}}  \subseteq     \qty{n_{3, k}} \subseteq \qty{n_{2, k}} \subseteq \qty{n_{1, k} - 1} \subseteq \mathbb{N}. 
\end{align}
We construct $\qty{n_k}$ in this way 
because it then inherits all uniform convergence properties.
Precisely speaking,
by the Arzela-Ascoli theorem in Appendix \ref{appendix: AA theorem}, we have the following corollary.
\begin{corollary}\label{corollary: lim continuous}
There exist some continuous functions $f^{\lim}(t)$ and $\hat{x}^{\lim}(t)$ such that $\forall t \in [0, T)$,
\begin{align}
% z^{\lim}(t) &\doteq \lim_{k \to \infty}     z_{n_k}(t)  , \\ 
\lim_{k \to \infty} f_{n_k}(t) =& f^{\lim}(t),    \\ 
\lim_{k \to \infty} \hat{x}(T_{n_k}+t) =& \hat{x}^{\lim}(t).  
\end{align}
Moreover, the convergence is uniform in $t$ on $[0, T)$.
\end{corollary}
In terms of the three sequences of functions in~\eqref{eq three seq},
Corollary~\ref{corollary: lim continuous} has identified that two of them converge along $\qty{n_k}$. 
% We now proceed to show that so does the third.
% In particular,
% let $z^{\lim}(t)$ denote the unique solution to the~\eqref{eq ode at limit} with the initial condition
% \begin{align}
% z^{\lim}(0) = \hat{x}^{\lim}(0),
% \end{align}
% which be explicitly written as
% \begin{align}
% z^{\lim}(t) &= \hat{x}^{\lim}(0) + \int_0^t h_{\infty}(z^{\lim}(s))ds.
% \end{align}
Lemma \ref{lemma: z n limit} further confirms that $z^{\lim}$ is the limit of $\qty{z_{n_k}}$. That is 
 $\forall t \in [0, T)$, 
\begin{align}
\lim_{k \to \infty} z_{n_k}(t)   =   z^{\lim}(t).
\end{align}
Moreover, the convergence is uniform in $t$ on $[0, T)$.
By \eqref{eq: all uniformly convergent}, 
we have
\begin{align}\label{eq: limit r *}
\lim_{k \to \infty} r_{n_k}  =& \infty, \\  
\lim_{k \to \infty} r_{n_k + 1}  =& \infty,
\end{align}
which completes the proof.
\end{proof}

\subsection{Proof of Lemma~\ref{lemma: double limit 1}}\label{appendix: double limit 1}

\begin{proof}
$\forall j$, $\forall k$, $\forall t \in [0,T)$,
\begin{align}
&\aftergroup \dsf \left \vert  \norm{\sum_{i = m(T_{n_k})}^{m(T_{n_k} +t) -1 } \alpha(i)H_{r_{n_j}}(\hat{x}(t(i)),Y_{i+1}) - \int_0^t h_{r_{n_j}} (\hat{x}^{\lim}(s))ds}  \right. \\
&\aftergroup \dsf \left. -\norm{\sum_{i = m(T_{n_k})}^{m(T_{n_k} +t) -1 } \alpha(i)H_{\infty}(\hat{x}(t(i)),Y_{i+1}) - \int_0^t h_{\infty} (\hat{x}^{\lim}(s))ds}  \right \vert   
% \explain{\lsz{We have an abs here. This abs becomes implicit in the next step.}}  	
\\
\leq&	\Bigg\lVert\sum_{i = m(T_{n_k})}^{m(T_{n_k} +t) -1 } \alpha(i)H_{r_{n_j}}(\hat{x}(t(i)),Y_{i+1}) - \int_0^t h_{r_{n_j}} (\hat{x}^{\lim}(s))ds  \\
&-\sum_{i = m(T_{n_k})}^{m(T_{n_k} +t) -1 } \alpha(i)H_{\infty}(\hat{x}(t(i)),Y_{i+1}) + \int_0^t h_{\infty} (\hat{x}^{\lim}(s))ds \Bigg\rVert \explain{by $\abs{\norm{a} - \norm{b}} \leq \norm{a-b}$} 	
% \label{eq: H h H h drop abs}
\\
\leq&	 \norm{\sum_{i = m(T_{n_k})}^{m(T_{n_k} +t) -1 } \alpha(i)(H_{r_{n_j}}(\hat{x}(t(i)),Y_{i+1}) - H_{\infty}(\hat{x}(t(i)),Y_{i+1}))} + \norm{\int_0^t h_{r_{n_j}} (\hat{x}^{\lim}(s)) -  h_{\infty} (\hat{x}^{\lim}(s))ds }\\
\leq&	\norm{ \sum_{i = m(T_{n_k})}^{m(T_{n_k} +t) -1 } \alpha(i)(H_{r_{n_j}}(\hat{x}(t(i)),Y_{i+1}) - H_{\infty}(\hat{x}(t(i)),Y_{i+1}))} + \int_0^t \norm{ h_{r_{n_j}} (\hat{x}^{\lim}(s)) -  h_{\infty} (\hat{x}^{\lim}(s))} ds \label{eq: H h j H h inf 1}
\end{align}
By Lemma \ref{lemma: bound x hat},
$\hat{x}(t(i))$ is in a compact set $\fB_{\hat{x}}$.
By Lemma \ref{lemma: lim H uniformly convergent}, for the compact set  $\fB_{\hat{x}}$, $\forall \epsilon>0$, $\exists j_1$ such that $\forall j \geq j_1$, $\forall k$, $\forall x \in \fB$, $\forall t \in [0,T)$,
% \lsz{$\forall x$ in a compact set. This needs to go before $j$}
% \lsz{$\forall t$ comes after $j$}
% \sz{$\forall x$ in a compact set}
\begin{align}
\norm{\sum_{i=m(T_{n_k})}^{m(T_{n_k}+t) - 1} \alpha(i) \left[H_{r_{n_j}}(x, Y_{i+1}) - H_{\infty}(x, Y_{i+1})\right]} \leq \epsilon \label{eq: H y_{k_0} delta T}.
\end{align} 
% uniformly in $k$, compacts of $x$, and $t \in [0, T)$.  
% \lsz{change to $\forall$ format.}
Similar to the proof of Lemma \ref{lemma: h c z lim uniform}, 
we have
% we have 
% $\forall y \in \fY, \lim_{j \to \infty}H_{r_{n_j}}(\hat{x}(T_k + t),y) = H_{\infty}(\hat{x}(T_k + t),y)$ uniformly in $k$ and $t\in [0,T)$.
\begin{align}\label{eq: h hat x uniform}
\lim_{j \to \infty}h_{r_{n_j}}(\hat{x}(T_k + t)) = h_{\infty}(\hat{x}(T_k + t))
\end{align}
uniformly in $k$ and $t \in [0,T)$.
By \eqref{eq: h hat x uniform}, $\forall \epsilon>0$, $\exists j_2$ such that $\forall j > j_2$, $\forall k$, 
$\forall t \in [0,T)$,
\begin{align}
\norm{h_{r_{n_j}}(\hat{x}(T_k + t)) - h_{\infty}(\hat{x}(T_k + t))} \leq \epsilon. \label{eq: uniform k h epsilon}
\end{align}
Define $j_0 \doteq  \max \qty{j_1, j_2}$.
 $\forall j \geq j_0$,  $\forall k$,  $\forall t \in [0,T)$,
\begin{align}
&\aftergroup \dsf \left \vert  \norm{\sum_{i = m(T_{n_k})}^{m(T_{n_k} +t) -1 } \alpha(i)H_{r_{n_j}}(\hat{x}(t(i)),Y_{i+1}) - \int_0^t h_{r_{n_j}} (\hat{x}^{\lim}(s))ds}  \right. \\
&\aftergroup \dsf \left. -\norm{\sum_{i = m(T_{n_k})}^{m(T_{n_k} +t) -1 } \alpha(i)H_{\infty}(\hat{x}(t(i)),Y_{i+1}) - \int_0^t h_{\infty} (\hat{x}^{\lim}(s))ds}  \right \vert   \\
\leq&	\norm{\sum_{i = m(T_{n_k})}^{m(T_{n_k} +t) -1 } \alpha(i)(H_{r_{n_j}}(\hat{x}(t(i)),Y_{i+1}) - H_{\infty}(\hat{x}(t(i)),Y_{i+1}))} + T \epsilon \explain{by \eqref{eq: H h j H h inf 1}, \eqref{eq: uniform k h epsilon}}  \\
\leq&	\epsilon + T \epsilon \explain{by \eqref{eq: H h j H h inf 1}, \eqref{eq: H y_{k_0} delta T}} \\
\leq& (T+1) \epsilon.    
\end{align}
This completes the proof of uniform convergence.

\end{proof}

\subsection{Proof of Lemma~\ref{lemma: single limit}}\label{appendix: single limit}

\begin{proof}
\begin{align}
&\lim \limits_{\begin{subarray}{l}j \to \infty\\k \to \infty \end{subarray}}
\norm{\sum_{i = m(T_{n_k})}^{m(T_{n_k} +t) -1 } \alpha(i)H_{r_{n_j}}(\hat{x}(t(i)),Y_{i+1}) - \int_0^t h_{r_{n_j}} (\hat{x}^{\lim}(s))ds} \\
=& \lim_{j \to \infty} \lim_{k \to \infty}  \norm{\sum_{i = m(T_{n_k})}^{m(T_{n_k} +t) -1 } \alpha(i)H_{r_{n_j}}(\hat{x}(t(i)),Y_{i+1}) - \int_0^t h_{r_{n_j}} (\hat{x}^{\lim}(s))ds} \explain{by Lemma \ref{lemma: double limit 1}, \ref{lemma: double limit 2}, and Moore-Osgood Theorem for interchanging limits in Theorem \ref{appendix: Moore-Osgood Theorem}} \\
=&  \lim_{j \to \infty} 0 \explain{by Lemma \ref{lemma: double limit 2}} \\  
=& 0.  \label{eq: double limit equal 0}
\end{align}

\end{proof}

\subsection{Proof of Lemma~\ref{lemma: f k 0}}\label{appendix: f k 0}

\begin{proof}
% \sz{The current strucutre is really hard to read. You should make all the above a lemma.}
We now proceed to investigate the property of $f_{n_k}(t)$.
% For a fixed sample path satisfying all the above lemmas, 
$\forall t\in [0,T)$,
\begin{align}
&\lim_{k \to \infty}  \norm{f_{n_k}(t)} \\ 
% =& \lim_{k \to \infty} \norm{ \sum_{i = m(T_{n_k})}^{m(T_{n_k} +t) -1 } \alpha(i)H_{r_{n_k}}(\hat{x}(t(i)),Y_{i+1}) - \int_0^t h_{r_{n_k}}(z_{n_k}(s))ds }\explain{by \eqref{eq: lim f k}}\\
% =&\lim_{k \to \infty} \norm{ \sum_{i = m(T_{n_k})}^{m(T_{n_k} +t) -1 } \alpha(i)H_{r_{n_k}}(\hat{x}(t(i)),Y_{i+1}) - \int_0^t h_{r_{n_k}}(\hat{x}^{\lim}(s))ds + \int_0^t h_{r_{n_k}}(\hat{x}^{\lim}(s))ds  - \int_0^t h_{r_{n_k}}(z_{n_k}(s))ds }\\
\leq&\lim_{k \to \infty} \norm{ \sum_{i = m(T_{n_k})}^{m(T_{n_k} +t) -1 } \alpha(i)H_{r_{n_k}}(\hat{x}(t(i)),Y_{i+1}) - \int_0^t h_{r_{n_k}}(\hat{x}^{\lim}(s))ds } \\
&+\lim_{k \to \infty} \norm{\int_0^t h_{r_{n_k}}(\hat{x}^{\lim}(s))ds  - \int_0^t h_{r_{n_k}}(z_{n_k}(s))ds } \explain{by \eqref{eq: lim f k}} \\
=&  \lim_{k \to \infty} \norm{  \int_0^t h_{r_{n_k}} (\hat{x}^{\lim}(s))ds  - \int_0^t h_{r_{n_k}}(z_{n_k}(s))ds } \explain{by \eqref{eq: double limit equal 0}} \\
=& \explaind{\norm{\int_0^t h_{\infty} (\hat{x}^{\lim}(s))ds  - \int_0^t h_{\infty}(z^{\lim}(s))ds}.}{by Lemma \ref{lemma: h k x lim h inf x lim uniform} and Lemma \ref{lemma: h k z k h inf z lim uniform}} \label{eq: reduce f limit 1} 
\end{align}
We now show the relationship between $\hat{x}^{\lim}(t)$ and $z^{\lim}(t)$.
\begin{align}
&\norm{\hat{x}^{\lim}(t) - z^{\lim}(t)} \label{eq: norm x lim z lim} \\
% =& \norm{\lim_{k \to \infty} \left[\hat{x}(T_{n_k}) + \int_0^t h_{r_{n_k}} (\hat{x}^{\lim}(s))ds \right] - z^{\lim}(t) } \explain{by \eqref{eq: x lim definition 2}}\\
=& \norm{\lim_{k \to \infty} \left[\hat{x}(T_{n_k}) + \int_0^t h_{r_{n_k}} (\hat{x}^{\lim}(s))ds \right] - \left[\hat{x}^{\lim}(0) + \int_0^t h_{\infty}(z^{\lim}(s))ds\right] } \explain{by \eqref{eq: z lim t} and \eqref{eq: x lim definition 2}} \\
=& \norm{\hat{x}^{\lim}(0) + \int_0^t h_{\infty} (\hat{x}^{\lim}(s))ds - \left[\hat{x}^{\lim}(0) + \int_0^t h_{\infty}(z^{\lim}(s))ds\right] } \explain{by 
Lemma \ref{lemma: h k x lim h inf x lim uniform}}\\
=& \norm{ \int_0^t h_{\infty} (\hat{x}^{\lim}(s))ds -  \int_0^t h_{\infty}(z^{\lim}(s))ds } \label{eq: f limit construct Gronwall 1}\\
% =& \norm{ \int_0^t h_{\infty} (\hat{x}^{\lim}(s))- h_{\infty}(z^{\lim}(s))ds } \\
% \leq&  \int_0^t \norm{h_{\infty} (\hat{x}^{\lim}(s))- h_{\infty}(z^{\lim}(s))} ds \\
\leq & \int_0^t L \norm{ \hat{x}^{\lim}(s) - z^{\lim}(s)}ds \explain{by Lemma \ref{lemma: h Lipschitz}} \\
\leq& 0. \explain{by Gronwall inequality in Theorem \ref{theorem: Gronwall}}
\end{align}
Thus,
\begin{align}
&\norm{\lim_{k \to \infty}  f_{n_k}(t) }  \\
\leq& \norm{\int_0^t h_{\infty} (\hat{x}^{\lim}(s))ds  - \int_0^t h_{\infty}(z^{\lim}(s))ds} \explain{by \eqref{eq: reduce f limit 1}} \\
=& \norm{\hat{x}^{\lim}(t) - z^{\lim}(t)} \explain{by \eqref{eq: f limit construct Gronwall 1}} \\
\leq& 0. \explain{by \eqref{eq: norm x lim z lim}}
\end{align}

\end{proof}

\subsection{Proof of Lemma~\ref{lemma: contradiction}}
\label{appendix: contradiction}

\begin{proof}
According to~\eqref{def: z n},
to study $\qty{z_{n_k}(t)}$,
it is instrumental to study the following ODE  
\begin{align}
    \dv{\phi_c(t)}{t} = h_c(\phi_c(t))
\end{align}
for some $c \geq 1$.
Let $\phi_{c, x}(t)$ denote the unique solution of the ODE above with the initial condition $\phi_{c, x}(0) = x$.
Intuitively,
as $c \to \infty$,
the above ODE approaches the~\eqref{eq ode at limit}.
Since any trajectory of~\eqref{eq ode at limit} will diminish to 0 (Assumption~\ref{assumption: lim h uniformly convergent}),
$\phi_{c, x}(t)$ should also diminish to some extent for sufficiently large $c$.
Precisely speaking,
we have
the following lemma.
% Let $\phi_{c}(t,x)$ denote the solution of the o.d.e. 
% \begin{align}
% \frac{d \phi_{c}(t,x)}{d t} = h_{c}(\phi_{c}(t,x))
% z_n(t)   &\doteq  \hat{x}(T_n) + \int_0^t h_{r_n}(z_n(s))ds.
% \end{align}
% with initial condition $x$.
% We restate the following lemma in \cite{borkar2009stochastic},
\begin{lemma}\label{lemma: Bokar 1/4}
(Corollary 3.3 in \citet{borkar2009stochastic})
There exist $c_1 > 0$ and $\tau>0$ such that for all initial conditions $x$ with $\norm{x} \leq 1$, we have
\begin{align}
\norm{\phi_{c, x}(t)} \leq \frac{1}{4}  
\end{align}
for $t\in[\tau, \tau + 1]$ and $c \geq c_1$.    
\end{lemma}
Here the $\frac{1}{4}$ is entirely arbitrary.
Now we fix any $c_0 \geq \max\qty{c_1, 1}$ and set $T = \tau$.
% Because Lemma \ref{lemma: Bokar 1/4} holds for all $c\geq c_1$, this means we can always find a  
% without loss of generality, we can assume $c_0 > 1$.
% \sz{This doesn't make sense. You can't fix a $c_0$, which is well defined in the above lemma. It can be either $> 1$ or $< 1$, you have no control over it.}
% \lsz{say overload explicitly}
% \begin{align}
% c_0 > 1 \label{def: c 0}
% \end{align} 
% and $\tau$ such that Lemma \ref{lemma: Bokar 1/4} holds.
% Set
% \begin{align}
% T = \tau. \label{def: T}
% \end{align}
Then Lemma~\ref{lemma: Bokar 1/4} confirms that $z_{n_k}(t)$ will diminish to some extent as $t$ approaches $T$ for sufficiently large $k$,
so does $\hat x(T_{n_k} + t)$.
We,
however, 
recall that $\hat x(T_{n_k} + t)$ and $\bar x(T_{n_k} + t)$ are well defined on $[0, T_{n+1} - T_n)$ and we restrict them to $[0, T)$ for applying the Arzela-Ascoli theorem.
Lemma \ref{lemma: connect T and k+1} processes the excess part $[T, T_{n+1} - T_n)$,
by showing that $\bar x(T_{n_k} + t)$ cannot grow too much in the excess part. By Lemma~\ref{lemma: connect T and k+1}, 
\begin{align}
\lim_{k \to \infty}
\frac{\norm{\bar{x}(T_{n_k + 1}) } -  \lim_{t \to T^-}  \norm{\bar{x}(T_{n_k} + t)}}{ \norm{\bar{x}(T_{n_k})}}= 0. \label{eq: connect T and k+1}
\end{align}
We are now in the position to identify the contradiction.
By~\eqref{eq: limit r *},
$\exists k_1$ such that $\forall k \geq k_1$,
\begin{align}
r_{n_{k} + 1} > \left(c_0  C_H  + C_H \right) e^{C_H} + c_0 > c_0 > 1 \label{eq: r n < infty k_0}.
\end{align}
By Lemma \ref{lemma: f k 0}, $\exists k_2$ such that $\forall k\geq k_2$,
\begin{align}
\lim_{t \to T^-} \norm{f_{n_k}(t)} = \lim_{t \to T^-} \norm{\hat{x}(T_{n_k}+t) - z_{n_k}(t)} \leq \frac{1}{4}.   \label{eq: term 1 4 1}
\end{align}
By \eqref{eq: connect T and k+1}, $\exists k_3$ such that $\forall k\geq k_3$,
\begin{align}\label{eq: rn infty k2}
\frac{\norm{\bar{x}(T_{n_k + 1}) } -  \lim_{t \to T^-}  \norm{\bar{x}(T_{n_k} + t)}}{ \norm{\bar{x}(T_{n_k})}} \leq \frac{1}{4}. 
\end{align}
By \eqref{eq: limit r *},  $\exists k_4$ such that $\forall k\geq k_4$,
\begin{align}
 r_{n_k} > c_0.   
\end{align}
Define $k_0 \doteq \max\qty{k_1,k_2,k_3,k_4}$.
% \sz{You can now make $k_0$ large enough such that $r_{n_{k_0}} > c_0$ holds because the new subsequence.}
Because $r_{n_{k_0}} > c_0$, by Lemma \ref{lemma: Bokar 1/4} and  \eqref{def: z n}, we have
\begin{align}
\lim_{t \to T^-}  \norm{z_{n_{k_0}}(t)} \leq \frac{1}{4}. \label{eq: term 1 4 2}
\end{align}
We have
\begin{align}
&\lim_{t \to T^-}  \norm{\hat{x}(T_{n_{k_0}}+t)}  \\
\leq& \lim_{t \to T^-}  \norm{\hat{x}(T_{n_{k_0}}+t) - z_{n_{k_0}}(t)}  +  \norm{z_{n_{k_0}}(t)}  \\
\leq& \explaind{\frac{1}{2}.}{by \eqref{eq: term 1 4 1} and \eqref{eq: term 1 4 2}} \label{eq: rn infty 1/2}
\end{align}
This implies
\begin{align}
&\frac{\norm{\bar{x}(T_{n_{k_0} + 1}) } }{ \norm{\bar{x}(T_{n_{k_0} }) }  }  \\
=& \frac{\norm{\bar{x}(T_{n_{k_0} + 1}) } -  \lim_{t \to T^-}  \norm{\bar{x}(T_{n_{k_0}} + t)}}{ \norm{\bar{x}(T_{n_{k_0}})}}  +  \frac{\lim_{t \to T^-}   \norm{\bar{x}(T_{n_{k_0} } + t) } }{ \norm{\bar{x}(T_{n_{k_0} }) }  }  \\
\leq& \frac{1}{4} +   \frac{\lim_{t \to T^-}   \norm{\bar{x}(T_{n_{k_0} } + t) } }{ \norm{\bar{x}(T_{n_{k_0} }) }  }    \explain{by \eqref{eq: rn infty k2}}\\
=&\frac{1}{4} +  \frac{\lim_{t \to T^-}   \norm{\hat{x}(T_{n_{k_0} } + t) } }{ \norm{\hat{x}(T_{n_{k_0} }) }  }  \explain{by \eqref{def: hat x}} \\
=&\frac{1}{4} +  \lim_{t \to T^-}   \norm{\hat{x}(T_{n_{k_0} } + t)} \explain{$\norm{\hat{x}(T_{n_{k_0} }) } = 1$ because of $ r_{n_{k_0}} > c_0 > 1$ and \eqref{def: hat x}} \\
\leq& \explaind{\frac{3}{4}.}{by \eqref{eq: rn infty 1/2}}\label{eq: 3/4}
\end{align}
Now, we can derive the following inequality.
\begin{align}
r_{n_{k_0} + 1} &= \norm{\bar x(T_{n_{k_0} + 1})} \explain{by \eqref{eq: r n < infty k_0}}\\
&\leq \frac{3}{4} \norm{\bar x(T_{n_{k_0}})} \explain{by \eqref{eq: 3/4}}\\
&\leq  \norm{\bar x(T_{n_{k_0}})} \\
&\leq  r_{n_{k_0}} \explain{by $r_{n_{k_0}} > c_0 > 1$ and \eqref{def: r n}},
\end{align}
which completes the proof.
\end{proof}

\subsection{Proof of Lemma~\ref{lemma: x m T n + i - x m T n}}
\label{appendix: x m T n + i - x m T n}
\begin{proof}
\begin{align}
&\sup_n \sup_{i \in \qty{i | m(T_n) \leq m(T_n) + i < m(T_{n+1}) }} \norm{x_{m(T_n) + i}}     -  \norm{x_{m(T_n)}} \\
\leq&\sup_n \sup_{i \in \qty{i | m(T_n) \leq m(T_n) + i < m(T_{n+1}) }} \norm{x_{m(T_n) + i}   -  x_{m(T_n)}} \\ 
% =& \sup_n \sup_{i \in \qty{i | m(T_n) \leq m(T_n) + i < m(T_{n+1}) }} \norm{\sum_{j = m(T_n)}^{m(T_n) + i -1}  \alpha(j) H(\bar{x}(t(j)), Y_{j+1})}  \\ 
=&\sup_n \sup_{i \in \qty{i | m(T_n) \leq m(T_n) + i < m(T_{n+1}) }} \norm{\bar{x}(t(m(T_n) + i))   -  \bar{x}(T_n)}  \\ 
=& \sup_n \sup_{t \in [T_n, T_{n+1})} \norm{\bar{x}(T_n + t)   -  \bar{x}(T_n)} \explain{by \eqref{def: bar x}}\\ 
\leq& \sup_n  \sup_{t \in [T_n, T_{n+1})} \left[\norm{ \bar{x}(T_n)}  C_H  +  C_H  \right] e^{C_H} \explain{by \eqref{eq: bar x bound}}\\
\leq& \sup_n  \sup_{t \in [T_n, T_{n+1})} \left[r_n  C_H  +  C_H  \right] e^{C_H} \explain{by \eqref{def: r n}}\\
=& \sup_n  [r_n  C_H   +  C_H]  e^{C_H} \\
<& \infty \explain{by \eqref{eq: sup r n bounded results}}.
\end{align}

\end{proof}
\color{black}

\subsection{Proof of Corollary~\ref{cor: stability main}}
\label{appendix: convergence cor}

This proof follows the idea of the proof of Theorem 2.1 in Chapter 5 of \citet{kushner2003stochastic}.
\begin{proof}
Let Assumptions \ref{assumption: stationary distribution} - \ref{assumption: lim h uniformly convergent} hold. 
Let Assumption \ref{assumption: lln} or~\ref{assumption possion} hold.
To prove convergence results on $t \in(-\infty, \infty)$ in Corollary \ref{cor: stability main}, 
we fix an arbitrary sample path $\qty{x_0,\qty{Y_i}_{i=1}^\infty}$. The stability results from Theorem~\ref{thm: stability} hold.
% such that 
% Assumptions~\ref{assumption: stationary distribution},~\ref{assumption: alpha rate},~\ref{assumption: H Lipschitz}, \ref{assumption: H h 0} and the stability results from Theorem~\ref{thm: stability} hold.
To prove  properties on $t \in(-\infty, \infty)$, we first fix an arbitrary $\tau > 0$ and show properties on $\forall t \in [-\tau, \tau]$.

\begin{definition}
% \sz{where in \citet{kushner2003stochastic}}
% \lsz{I think we do not need to emphasize it again as we have mentioned we use  \citet{kushner2003stochastic}  a few lines above. It is better to think this is a similar definition of Definition \ref{definition: z n}.}
$\forall n \in \N $, define $\bar{z}_n(t)$ as the solution to the ODE~\eqref{eq original ode} in $(-\infty, \infty)$ with  
% \begin{align}
% \frac{d \bar{z}_n(t)}{d t} = h(\bar{z}_n(t)) \label{cor-def: z n}
% \bar{z}_n(t)   &\doteq  \bar{x}(t(n)) + \int_0^t h(\bar{z}_n(s))ds.
% \end{align}
an initial condition 
\begin{align}
\bar{z}_n(0) = \bar{x}(t(n)). \label{cor-def: z n init}
\end{align}
\end{definition}
% Since $h$ is Lipschitz continuous, by Picard-Lindelöf Theorem, $\bar{z}_n(t)$ has a unique well-defined solution. Because $z_n$ is differentable, $\forall n, z_n$ is continuous on $t \in [-\tau, \tau]$. 
Apparently, 
$\bar{z}_n(t)$ can also be written as 
\begin{align}
% \frac{d \bar{z}_n(t)}{d t} = h(\bar{z}_n(t))
\bar{z}_n(t)   &=  \bar{x}(t(n)) + \int_0^t h(\bar{z}_n(s))ds, \quad \forall t \in (-\infty, \infty). \label{cor-def: z n integral}
\end{align}
The major difference between the $\qty{\bar z_n(t)}$ here and the $\qty{z_n(t)}$ in \eqref{def: z n} is that
all $\qty{\bar z_n(t)}$ here are solutions to one same ODE~\eqref{eq original ode}, just with different initial conditions,
but $\qty{z_n(t)}$ is for different ODEs with different initial conditions and rescale factors $r_n$ and is written as 
\begin{align}
z_n(t)   &=  \hat{x}(T_n) + \int_0^t h_{r_n}(z_n(s))ds \explain{Restatement of \eqref{def: z n integral}}. 
\end{align}
Ideally,
we would like to see
that
the error of Euler's discretization diminishes asymptotically.
With  \eqref{eq: m 0} and \eqref{def: bar x}, $\forall \tau > 0$, $\forall t \in [-\tau, \tau]$,
% \sz{I don't think the first equality holds.}
\begin{align}\label{eq: bar x extend}
\bar{x}(t(n) + t) = x_{m(t(n) + t)} = 
\begin{cases}
\bar{x}(t(n)) + \sum_{i = n}^{m(t(n)+t) - 1} \alpha(i) H(\bar{x}(t(i)), Y_{i+1}) & \text{if } t \geq 0 \\
\bar{x}(t(n)) - \sum_{i = m(t(n)+t)}^{n - 1} \alpha(i) H(\bar{x}(t(i)), Y_{i+1}) & \text{if } t < 0. 
\end{cases}
\end{align}
Notably, the property \eqref{eq: m 0} that $\forall t<0, m(t) = 0$ in \eqref{eq: bar x extend} ensures $\bar{x}(t(n) + t)$ is well-defined when $t(n) + t < 0$.
Precisely speaking, $\forall \tau > 0$, $\forall t \in [-\tau, \tau]$,
the discretization error is defined as 
\begin{align}
\bar{f}_n(t) \doteq \bar{x}(t(n) + t) - \bar{z}_n(t). \label{cor-def: f}
\end{align}
and we would like $\bar{f}_n(t)$ diminishes to 0 as $n\to\infty$ in certain sense.
To this end, we study the following three sequences of functions
\begin{align}
    \label{cor-eq three seq}
    \qty{ \bar x(t(n) + t)}_{n=0}^\infty, \qty{\bar{z}_n(t)}_{n=0}^\infty, \qty{\bar{f}_n(t)}_{n=0}^\infty.
\end{align}
Equicontinuity in the extended sense on domain $(-\infty, \infty)$ is defined as following (Section 4.2.1 in \citet{kushner2003stochastic}).
\begin{definition}
% [Equicontinuous in the Extended Sense]
\label{cor-def: equicontinuity broad}
A sequence of functions $\qty{g_n: (-\infty, \infty) \to \R^K }$ is equicontinuous in the extended sense on $(-\infty, \infty)$ if $\sup_n \norm{g_n(0)}  < \infty$  and $\forall \tau > 0$, $\forall \epsilon > 0$,  $\exists \delta > 0$ such that
\begin{align}\label{eq: equicontinuity inf inf epsilon}
\limsup_n \sup_{0\leq \abs{t_1-t_2} \leq \delta, \abs{t_1} \leq \tau,\abs{t_2} \leq \tau  } \norm{g_n(t_1) - g_n(t_2)} \leq \epsilon.    
\end{align}
\end{definition}
We  show  $\qty{\bar{x}(t(n)+t)}$, $\qty{\bar{z}_n(t)}$ and $\qty{\bar{f}_n(t)}$  are all equicontinuous in the extended sense.
% \begin{lemma}\label{cor-lemma: hat x equicontinuous}
% $\qty{\bar{x}(t(n)+t)}_{n = 0}^{\infty}$ is equicontinuous in the extended sense on $t \in (-\infty, \infty)$. 
% \end{lemma}
% % Its proof is in appendix \ref{cor-appendix: hat x equicontinuous}
% \begin{lemma}\label{cor-lemma: z n equicontinuous}
% $\qty{\bar{z}_n(t)}_{n = 0}^{\infty}$ is equicontinuous in the extended sense on $t \in (-\infty, \infty)$. 
% \end{lemma}
% % Its proof is in appendix \ref{cor-appendix: z n equicontinuous}.
% \begin{lemma}\label{cor-lemma: f n equicontinuous}
% $\qty{\bar{f}_n(t)}_{n = 0}^{\infty}$ is equicontinuous in the extended sense on $t \in (-\infty, \infty)$.  
% \end{lemma}

\begin{lemma}\label{cor-lemma: three functions equicontinuous}
The three sequences of functions $\qty{\bar{x}(t(n)+t)}_{n = 0}^{\infty}$,
$\qty{\bar{z}_n(t)}_{n = 0}^{\infty}$, and 
$\qty{\bar{f}_n(t)}_{n = 0}^{\infty}$ are all equicontinuous in the extended sense on $t \in (-\infty, \infty)$.
\end{lemma}

To prove those lemmas, we need the Gronwall inequality in the reverse time in Appendix \ref{theorem: Gronwall reverse}.
Compared to lemmas in the main text which have domain $t \in [0,T)$, lemmas in this section have similar proofs because we first fix an arbitrary $\tau$ and prove properties on the domain $t \in [-\tau, \tau]$. 
% We omit proofs for Lemmas \ref{cor-lemma: hat x equicontinuous}, \ref{cor-lemma: z n equicontinuous}, \& \ref{cor-lemma: f n equicontinuous} because they are ditto to proofs of Lemmas \ref{lemma: hat x equicontinuous}, \ref{lemma: z n equicontinuous}, \& \ref{lemma: f n equicontinuous}.
We omit proofs for Lemma \ref{cor-lemma: three functions equicontinuous} because they are ditto to proofs of Lemma \ref{lemma: three functions equicontinuous}.
% which have symmetric but ditto proofs in this section. 
% According to the Arzela-Ascoli theorem in the extended sense (Appendix~\ref{appendix: AA theorem}), a sequence of
% equicontinuous functions always has a subsequence of functions that uniformly converge to a continuous limit.
% In the following, we use this to identify a particular subsequence of interest.
Similar to Lemma~\ref{lemma: three convergence}, we now construct a particular subsequence of interest.
\begin{lemma}\label{cor-lemma: three convergence}
There exists a subsequence $\qty{n_k}_{k=0}^\infty \subseteq \qty{0, 1, 2, \dots}$ and some continuous functions $\bar{f}^{\lim}(t)$ and $\bar{x}^{\lim}(t)$ such that  $\forall \tau$, $\forall t \in [-\tau, \tau]$,
\begin{align}
\lim_{k \to \infty} \bar{f}_{n_k}(t) =& \bar{f}^{\lim}(t)   \label{cor-eq: f x r lim} , \\ 
\lim_{k \to \infty} \bar{x}(t(n_k)+t) =& \bar{x}^{\lim}(t),  \label{cor-eq: bar x x r lim}
\end{align}
where both convergences are uniform in $t$ on $[-\tau, \tau]$.
Furthermore,
let $\bar{z}^{\lim}(t)$ denote the unique solution to the ODE~\eqref{eq original ode} with the initial condition
\begin{align}
\bar{z}^{\lim}(0) = \bar{x}^{\lim}(0),
\end{align}
in other words,
\begin{align}\label{cor-eq: z lim t}
\bar{z}^{\lim}(t) &= \bar{x}^{\lim}(0) + \int_0^t h(\bar{z}^{\lim}(s))ds.
\end{align}
Then $\forall \tau$, $\forall t \in [-\tau, \tau]$, we have
\begin{align}
\lim_{k \to \infty} \bar{z}_{n_k}(t)   =   \bar{z}^{\lim}(t), \label{cor-eq: z lim converge}
\end{align}
where the convergence is uniform in $t$ on $[-\tau, \tau]$.
\end{lemma}
Its proof is ditto to the proof of Lemma \ref{lemma: three convergence} and is omitted.
We use the subsequence $\qty{n_k}$ intensively in the remaining proofs. Recall that $\bar{f}_n(t)$ denotes the discretization error between $\bar x(t(n) + t)$ and $\bar{z}_n(t)$.
We now proceed to prove that this discretization error diminishes along $\qty{n_k}$.
In particular, we aim to prove that $\forall \tau$, $\forall t \in [-\tau, \tau]$, 
% We are ready to analyze $\bar{f}_{n_k}(t)$. Because $ \lim_{k \rightarrow \infty} \bar{f}_{n_k}(t) = \bar{f}^{\lim}(t)$ uniformly, we aim to prove $\forall t$,
\begin{align}
 \lim_{k \rightarrow \infty} \norm{\bar{f}_{n_k}(t)} = \norm{\bar{f}^{\lim}(t)} = 0.
\end{align} 
This means $\bar{x}(t(n_k) + t)$ is close to $\bar{z}_{n_k}(t)$ as $k \rightarrow \infty$. 
For $t \in (0,\tau]$, the proof for this part is the same as the proof we have done in Section \ref{sec: sequence property}.
Thus, we only discuss the proof for $t \in [-\tau, 0]$.
% By \eqref{def: bar x}, 
% $\forall \tau$, for large $n$ such that $\forall t \in [-\tau, \tau]$, $t(n) + t \geq 0$, we have
% \begin{align}
% \bar{x}(t(n) + t) = \bar{x}(t(n)) - \sum_{i = m(t(n)+t)}^{m(t(n)) - 1} \alpha(i) H(\bar{x}(t(i)),Y_{i+1}).
% \end{align}
$\forall \tau$,
% for large $k$ such that $\forall t \in [-\tau, \tau]$, $t(n_k) + t \geq 0$, 
$\forall t \in [-\tau, 0]$,
% we have
\begin{align}
&\lim_{k \rightarrow \infty} \norm{ \bar{f}_{n_k}(t)} \\
=& \lim_{k \rightarrow \infty} \norm{ \bar{x}(t(n_k)) -\sum_{i = m(t(n_k)+t)}^{n_k - 1} \alpha(i)H(\bar{x}(t(i)),Y_{i+1}) - \bar{z}_{n_k}(t)} \explain{by \eqref{eq: bar x extend} and \eqref{cor-def: f}}\\
=&  \lim_{k \rightarrow \infty} \norm{  -\sum_{i = m(t(n_k)+t)}^{n_k - 1} \alpha(i)H(\bar{x}(t(i)),Y_{i+1}) - \int_0^t h(\bar{z}_{n_k}(s))ds} \explain{by \eqref{cor-def: z n integral}}\\
% =&\lim_{k \to \infty} \norm{ -\sum_{i = m(t(n_k)+t)}^{n_k - 1} \alpha(i)H(\bar{x}(t(i)),Y_{i+1}) - \int_0^t h(\bar{x}^{\lim}(s))ds + \int_0^t h(\bar{x}^{\lim}(s))ds  - \int_0^t h(\bar{z}_{n_k}(s))ds }\\
\leq&\lim_{k \to \infty} \norm{ -\sum_{i = m(t(n_k)+t)}^{n_k - 1} \alpha(i)H(\bar{x}(t(i)),Y_{i+1}) - \int_0^t h(\bar{x}^{\lim}(s))ds } \\
&+ \lim_{k \to \infty} \norm{\int_0^t h(\bar{x}^{\lim}(s))ds  - \int_0^t h(\bar{z}_{n_k}(s))ds } . \label{cor-eq: lim f k} \\
\end{align}
We now prove that the first term in the RHS of~\eqref{cor-eq: lim f k} is 0.
\begin{lemma}
\label{cor-lemma: double limit 2}
$\forall \tau$, $\forall t \in [-\tau, 0]$,
\begin{align}
\lim_{k \to \infty}  \norm{-\sum_{i = m(t(n_k)+t)}^{n_k - 1} \alpha(i)H(\bar{x}(t(i)),Y_{i+1}) - \int_0^t h(\bar{x}^{\lim}(s))ds} = 0.     
\end{align}   
\end{lemma}
Its proof is ditto to the proof of Lemma \ref{lemma: double limit 2} and is omitted.
This convergence is also simpler than \eqref{eq: single limit} because here we have only a single $(H, h)$.
But in \eqref{eq: single limit}, 
we have a sequence $\qty{(H_{n_k}, h_{n_k})}$,
for which we have to split it to a double limit \eqref{eq: double limit} and then invoke the Moore-Osgood theorem to reduce it to the single $(H, h)$ case.

Lemma~\ref{cor-lemma: double limit 2} confirms that the first term in the RHS of~\eqref{cor-eq: lim f k} is 0.
Moreover, it also
enables us to rewrite $\bar{x}^{\lim}(t)$ from a summation form to an integral form. $\forall \tau$, $\forall t \in [-\tau, 0]$
\begin{align}
&\bar{x}^{\lim}(t) \\
=& \lim_{k \to \infty} \bar{x}(t(n_k)) -\sum_{i = m(t(n_k)+t)}^{n_k - 1} \alpha(i)H(\bar{x}(t(i)),Y_{i+1}) \\
=& \explaind{\lim_{k \to \infty} \bar{x}(t(n_k)) + \int_0^t h (\bar{x}^{\lim}(s))ds.}{by Lemma \ref{cor-lemma: double limit 2}} \label{cor-eq: x lim definition 2} 
\end{align} 
Thus, we can show the following diminishing discretization error.

\begin{lemma}\label{cor-lemma: f k 0}
$\forall \tau$, $\forall t\in [-\tau, \tau],$
\begin{align}
\lim_{k \to \infty}  \norm{\bar{f}_{n_k}(t)} = 0.    
\end{align}
Moreover, the convergence is uniform in $t$ on $[-\tau,\tau]$.
\end{lemma}
Its proof is ditto to the proof of Lemma \ref{lemma: f k 0} and is omitted. 
This immediately implies that for any $t \in (-\infty, \infty)$
\begin{align}
    \label{eq corollary ode convergence}
 \lim_{k \to \infty} \bar{x}(t(n_k) + t) = \bar{z}^{\lim}(t).   
\end{align}
Theorem~\ref{thm: stability} then yields that
\begin{align}
    \sup_{t \in (-\infty, \infty)} \norm{\bar z^{\lim}(t)} < \infty.
\end{align}

Let $X$ be the limit set of $\qty{x_n}$, 
i.e.,
$X$ consists of all the limits of all the convergent subsequences of $\qty{x_n}$.
By Theorem \ref{thm: stability}, $\sup_n \norm{x_n} < \infty$,
so $X$ is bounded and nonempty.
We now prove $X$ is an invariant set of the ODE~\eqref{eq original ode}.
For any $x \in X$, 
there exists a subsequence $\qty{x_{n_k}}$ such that
\begin{align}
    \lim_{k\to\infty} x_{n_k} = x. 
\end{align}
Since $\qty{\bar x(t(n_k) + t)}$ is equicontinuous in the extended sense, 
following the way we arrive at~\eqref{eq corollary ode convergence}, 
we can construct a subsequence $\qty{n'_k} \subseteq \qty{n_k}$ such that
\begin{align}\label{eq: x n k' limit}
    \lim_{k\to\infty} \bar x(t(n_k') + t)  = z^{\lim}(t),
\end{align}
where $z^{\lim}(t)$ is a solution to the ODE  \eqref{eq original ode} and $z^{\lim}(0) = x$.
The remaining is to show that $z^{\lim}(t)$ lies entirely in $X$.
For any $t \in (-\infty, \infty)$,
by the piecewise constant nature of $\bar x$ in \eqref{eq: bar x extend},
the above limit \eqref{eq: x n k' limit}
implies that there exists a subsequence of $\qty{x_n}$ that converges to $z^{\lim}(t)$,
indicating $z^{\lim}(t) \in X$ by the definition of the limit set.
We now have proved $\forall x \in X$,
there exists a solution $z^{\lim}(t)$ to the ODE~\eqref{eq original ode} such that $z^{\lim}(0) = x$ and
$\forall t \in (-\infty, \infty), z^{\lim}(t) \in X$.
This means $X$ is an invariant set, by definition.
In particular, $X$ is a bounded invariant set.

We now prove that $\qty{x_n}$ converges to $X$.
Let $\qty{x_{n_k}}$ be any convergent subsequence of $\qty{x_n}$ with its limit denoted by $x$.
We must have $x \in X$ by the definition of the limit set.
So we have proved that all convergent subsequences of $\qty{x_n}$ converge to a point in the bounded invariant set $X$.
If $\qty{x_n}$ does not converge to $X$,
there must exists a subsequence $\qty{x_{n_k'}}$ such that $\qty{x_{n_k'}}$ is always away from $X$ by some small $\epsilon_0 > 0$, i.e., $\forall k$,
\begin{align}
\inf_{x\in X} \norm{x_{n_k'} - x} \geq \epsilon_0. \label{eq: x n k epsilon 0}
\end{align}
But $\qty{x_{n_k'}}$ is bounded so it must have a convergent subsequence,
which, by the definition of the limit set,
converges to some point in $X$.
This contradicts \eqref{eq: x n k epsilon 0}.
So we must have $\qty{x_n}$ converges to $X$,
which is a bounded invariant set of the ODE~\eqref{eq original ode}.
This completes the proof.
\end{proof}

\subsection{Proof of Theorem~\ref{theorem: GTD convergence}}\label{appendix: GTD convergence}

\begin{proof}
For simplicity, 
we define
\begin{align}
    A' \doteq& \mqty[-C & A \\ -A^\top & 0], \\
    b' \doteq& \mqty[b \\ 0].
\end{align}
We first invoke Corollary~\ref{cor: stability main} to show that 
\begin{align} 
\lim_{t\to\infty} x_t = -A'^{-1} b' \qq{a.s.}
\end{align}

Assumption~\ref{assumption: stationary distribution} follows immediately from Lemma~\ref{lemma: yu invariant measure}.

Assumption~\ref{assumption: alpha rate} follows immediately from Assumption \ref{assumption rl lr}.

For Assumption~\ref{assumption: specific H c H infinity}, define
% \textbf{To prove Assumption \ref{assumption: stationary distribution}:} \\
% By Lemma \ref{lemma: yu invariant measure}, $\qty{Y_t}$ in \eqref{eq: GTD update} has a unique invariant probability measure (i.e., stationary distribution), denoted by ${d_\fY}$.
% \textbf{To prove Assumption \ref{assumption: alpha rate}:} \\
% \textbf{To prove Assumption \ref{assumption: specific H c H infinity}:} \\
\begin{align}
H_\infty(x, y) \doteq& \mqty[-C(y) & A(y) \\ -A(y)^\top & 0] x.   
\end{align}
Then we have
% $\forall x, y$
\begin{align}
H_c(x, y) - H_\infty(x, y) =\frac{1}{c} \mqty[b(y) \\ 0].
% H_c(x, y) - H_\infty(x, y) =& \frac{\mqty[-C(y) & A(y) \\ -A(y)^\top & 0] cx + \mqty[b(y) \\ 0]}{c} - \mqty[-C(y) & A(y) \\ -A(y)^\top & 0] x  \\
%  =&\frac{1}{c} \mqty[b(y) \\ 0]  
\end{align}
After noticing
% Because $\lim{c \to \infty} \frac{1}{c} = 0$, \eqref{eq: kappa to 0} is satisfied.
% By \eqref{def: b y},
\begin{align}
\norm{b((s, a, s', e)) - b((s, a, s', e'))} = \rho(s, a) \abs{r(s, a)} \norm{e - e'}, \quad \forall s, a, s', e, e',
\end{align}
Assumption~\ref{assumption: specific H c H infinity} follows immediately from Lemma~\ref{lemma: yu invariant measure}.
% By Lemma \ref{lemma: yu invariant measure}, the expectation $\E_{y\sim{d_\fY}}\left[b(y)\right]$ exists and is finite. Thus, $\mqty[b(y) \\ 0]$ satisfies \eqref{eq: b finite}.

For Assumption~\ref{assumption: H Lipschitz},
it can be easily verified that both $H(x, y)$ and $H_\infty(x, y)$ are Lipschitz continuous in $x$ for each $y$ with the Lipschitz constant being
\begin{align}
    L(y) \doteq \norm{\mqty[-C(y) & A(y) \\ -A(y)^\top & 0]}.
\end{align}
% After noticing $e \mapsto H(x, (s, a, s', e))$ is Lipschitz continious in $e$ for each $x, s, a, s'$,
Since $A(y), b(y), C(y)$ are Lipschitz continuous in $e$ for each $(s, a, s')$,
Lemma~\ref{lemma: yu invariant measure} implies that
\begin{align}
h(x) =& A'x + b', \\
% h_c(x) &\doteq  \E_{y \sim {d_\fY}}[H_c(x,y)] = \frac{h(cx)}{c} , \\ 
h_\infty(x) =&A'x, \\
L =& \norm{A'} .
\end{align}
Assumption~\ref{assumption: H Lipschitz} then follows.

For Assumption~\ref{assumption: lim h uniformly convergent},
we have
\begin{align}
    \norm{h_c(x) - h_\infty(x)} \leq \frac{\norm{b'}}{c},
\end{align}
the uniform convergence of $h_c$ to $h_\infty$ follows immediately.
Proving that $A'$ is Hurwitz is a standard exercise using the field of values of $A'$.
We refer the reader to Section~5 of \citet{sutton2009fast} for details
and omit the proof.
This immediately implies the globally asymptotically stability of the following two ODEs
\begin{align}
    \dv{x(t)}{t} = A'x(t) + b', \quad \dv{x(t)}{t} = A' x(t).
\end{align}
The unique globally asymptotically equilibrium of the former is $-A'^{-1}b'$.
That of the latter is $0$.
Assumption~\ref{assumption: lim h uniformly convergent} then follows.

Assumption~\ref{assumption: lln} follows immediately from Lemma~\ref{lemma: yu invariant measure} and Assumption~\ref{assumption rl lr}.

Corollary~\ref{cor: stability main} then implies that
\begin{align}
\lim_{t\to\infty} x_t = -A'^{-1}b' \qq{a.s.}
\end{align}
Block matrix inversion immediately shows that the lower half of $A'^{-1} b'$ is $A^{-1}b$,
yielding
\begin{align}
\lim_{t\to\infty} \theta_t = -A^{-1}b \qq{a.s.,}
\end{align}
which completes the proof.
\end{proof}

\section{Auxiliary Lemmas}\label{sec: auxiliary lemma}
% \subsection{Auxiliary Lemma Statements}

\begin{lemma}\label{lemma: T-minus}
% \begin{align}
% &\forall n,  t(m(T_n)) = T_n, \\
% &\sup_n  [T_{n+1} - T_n] < \infty.
% \end{align}
% \sz{Maybe replace the above with
% There exists some finite $T_{\max} > 0$ such that
\begin{align}
\forall n, \,
% T_{\max} \geq
T_{n+1} - T_n \geq& T,\\
\lim_{n\to\infty} T_{n+1} - T_n =& T.
\end{align}
% which is more precise and provides more insights.}
Moreover, $\forall \tau > 0, t_1, t_2$ such that $ -\tau \leq t_1 \leq t_2 \leq \tau$, we have
\begin{align}
\lim_{n \to \infty}   \sum_{i = m(t(n)+t_1)}^{m(t(n)+t_2) - 1} \alpha(i)   = t_2 - t_1 \label{eq: lim alpha t_1 t_2} .
\end{align}
\end{lemma}
% Its proof is in Appendix \ref{appendix: T-minus}.
% \subsection{Proof of Lemma~\ref{lemma: T-minus}}\label{appendix: T-minus}

\begin{proof}
$\forall n$,
\begin{align}
&T_{n+1} - T_n  \\
=&t(m(T_n+T) + 1) -  T_n  \explain{by  \eqref{def:T}} \\
\geq&T_n + T - T_n \explain{by \eqref{eq: t-m-inequality-1}} \\
\geq&T.    
\end{align}
Thus,
\begin{align}
&\lim_{n\to\infty} T_{n+1} - T_n 
% \\
% =&\lim_{n\to\infty} t(m(T_n+T) + 1) -  T_n  \explain{by  \eqref{def:T}} \\
% \geq&\lim_{n\to\infty} T_n + T - T_n \explain{by \eqref{eq: t-m-inequality-1}} \\
\geq T.    
\end{align}
With
\begin{align}
&\lim_{n\to\infty} T_{n+1} - T_n \\
=&\lim_{n\to\infty} t(m(T_n+T) + 1) -  T_n \\   
=&\lim_{n\to\infty} t(m(T_n+T)) + \alpha(m(T_n+T)) -  T_n \\
\leq&\lim_{n\to\infty} T_n + T + \alpha(m(T_n+T)) - T_n \explain{by \eqref{eq: t-m-inequality-1}}\\
=& T, 
\end{align}
by the squeeze theorem, 
we have
$\lim_{n\to\infty} T_{n+1} - T_n = T$.

To prove \eqref{eq: lim alpha t_1 t_2},
$\forall \tau$, $\forall -\tau \leq t_1 \leq t_2 \leq \tau$, it suffices to only consider large $n$ such that $t(n) - \tau \geq 0$. We have
\begin{align}
&\lim_{n \to \infty}  \sum_{i = m(t(n)+t_1)}^{m(t(n) + t_2)  - 1 } \alpha(i) \\
=& \lim_{n \to \infty} 
 t(m(t(n)+t_2))   - t(m(t(n)+t_1)) \\
\leq& \lim_{n \to \infty} 
 t(n)+t_2 - t(m(t(n)+t_1)) \explain{by \eqref{eq: t-m-inequality-1}}\\
\leq&  \lim_{n \to \infty} 
 t(n)+t_2 - ( t(n)+t_1 - \alpha(m(t(n)+t_1))  ) \explain{by \eqref{eq: t-m-inequality-2}}\\
=&  t_2 - t_1 + \lim_{n \to \infty} \alpha(m(t(n)+t_1))  \\
=& t_2 - t_1  \explain{by \eqref{eq 1/n lr}} 
\end{align}
and
\begin{align}
&\lim_{n \to \infty}  \sum_{i = m(t(n)+t_1)}^{m(t(n) + t_2)  - 1 } \alpha(i) \\
=& \lim_{n \to \infty} 
 t(m(t(n)+t_2))   - t(m(t(n)+t_1)) \\
\geq&   \lim_{n \to \infty}  t(n)+t_2 - \alpha(m(t(n)+t_2))   - t(m(t(n)+t_1)) \explain{by \eqref{eq: t-m-inequality-2}} \\
\geq&  \lim_{n \to \infty} 
 t(n)+t_2 - \alpha(m(t(n)+t_2))   - (t(n)+t_1) \explain{by \eqref{eq: t-m-inequality-1}} \\
=& \lim_{n \to \infty}  t_2 - t_1 - \alpha(m(t(n)+t_2)) \\
=&t_2 - t_1 \explain{by \eqref{eq 1/n lr}} . 
\end{align}
By the squeeze theorem, we have
\begin{align}
\lim_n   \sum_{i = m(t(n)+t_1)}^{m(t(n)+t_2) - 1} \alpha(i)   = t_2 - t_1. 
\end{align}

\end{proof}

% The above assumption immediately implies the Lipschitz continuity of $h_c(x)$.
% \lsz{Only for a reminder of orders. Will remove those comments in the final draft.}

% \lsz{The below lemma is first used in lemma \ref{lemma: hat x equicontinuous}}

\begin{lemma}\label{lemma: h Lipschitz}
For any $x, x', c \geq 1$, 
including $c = \infty$,
\begin{align}
\norm{H_c(x,y) - H_c(x',y)} &\leq L(y) \norm{x - x'}, \label{eq: H c L}\\
\norm{h_c(x) - h_c(x')} &\leq L \norm{x - x'}. \label{eq: h c L}
\end{align}
\end{lemma}
% Its proof is in Appendix \ref{appendix: h Lipschitz}.
\begin{proof}
To prove \eqref{eq: H c L}, we first consider $1 \leq c < \infty$,
\begin{align}
&\norm{H_c(x,y) - H_c(x',y)} \\
=& \norm{\frac{H(cx,y)}{c} - \frac{H(cx',y)}{c}} \explain{by \eqref{def: H c}}\\
\leq& \frac{\norm{H(cx,y)- H(cx',y)}}{c}\\
\leq& L(y) \frac{\norm{cx - cx'}}{c}  \explain{by \eqref{eq: H L}}  \\
=& L(y)\norm{x - x'}. \label{eq: H c L x x'}
\end{align}
By \eqref{eq: H L inf}, 
% Assumption \ref{assumption: H Lipschitz},
\begin{align}
\norm{H_\infty(x,y) - H_\infty(x',y)} &\leq L(y) \norm{x - x'}.
\end{align}

To prove \eqref{eq: h c L},
$\forall x$, $\forall x'$, $\forall c \geq 1$ including $c= \infty$,
\begin{align}
&\norm{h_c(x) - h_c(x')} \\
% =& \norm{ \E_{y\sim {d_\fY}}\left[H_c(x,y)\right] - \E_{y\sim {d_\fY}}\left[H_c(x',y)\right]}\\
=& \norm{ \E_{y\sim {d_\fY}}\left[H_c(x,y) - H_c(x',y)\right]} \\
\leq&  \E_{y\sim {d_\fY}}\left[\norm{ H_c(x,y) - H_c(x',y) } \right] \\
\leq& \E_{y\sim {d_\fY}}\left[ L(y) \norm{x - x'}\right] \\
\leq& L \norm{x - x'}.
\end{align}

\end{proof}

% \begin{lemma}\label{lemma: convergence equivalence}
% Let $\qty{X_n}_{n = 1}^\infty$ be a sequence of random variable,
% \begin{align}
% \lim_{n\to \infty}  \norm{X_n} = 0 \quad a.s.    
% \iff
% \forall \epsilon>0 , \lim_{n \to \infty} P\left(\sup_{j \geq n} \norm{X_j} \geq \epsilon \right) = 0. 
% \end{align}
% \end{lemma}
% Its proof is in Appendix \ref{appendix: convergence equivalence}.
% \lsz{The below lemma is first used in lemma \ref{lemma: hat x equicontinuous}}

\begin{lemma}\label{lemma: bounded-h0 inf y}
 $\forall x$,
\begin{align}
\sup_{c\geq 1} \norm{ h_c(0) }   &<\infty, \label{eq: h property h}\\
\sup_{c\geq 1} \limsup_n \sup_{0\leq t_1 \leq t_2 \leq T_{n+1} - T_n} \norm{ \sum_{i = m(T_n+t_1)}^{m(T_n+t_2) - 1} \alpha(i) \left[H_c(x, Y_{i+1}) - h_c(x) \right] }    & = 0 \quad a.s., \label{eq: h property H h minus 0} \\ 
\sup_{c\geq 1} \sup_n \sup_{0\leq t_1 \leq t_2 \leq T_{n+1} - T_n} \norm{ \sum_{i = m(T_n+t_1)}^{m(T_n+t_2) - 1} \alpha(i) H_c(0, Y_{i+1}) }   &< \infty \quad a.s., \label{eq: h property H bounded}\\
\lim_{\delta \to 0^+} \sup_{c\geq 1} \limsup_n \sup_{0\leq t_2-t_1 \leq \delta } \norm{ \sum_{i = m(T_n+t_1)}^{m(T_n+t_2) - 1} \alpha(i) H_c(0, Y_{i+1}) }   & = 0 \quad a.s. \label{eq: h property close t 0}\\
\end{align}
\end{lemma}
% Its proof is in Appendix \ref{appendix: bounded-h0 inf y}.

% \subsubsection{Proof of Lemma~\ref{lemma: bounded-h0 inf y}}\label{appendix: bounded-h0 inf y}
\begin{proof}
\\
\textbf{Proof of \eqref{eq: h property h}:}
\begin{align}
\sup_{c\geq 1}  \norm{h_c(0)} = \sup_{c\geq 1} \norm{ \frac{h(0)}{c} } \leq \sup_{c\geq 1}  \norm{h(0)}  = \norm{h(0)} < \infty \label{eq: h c bound}.
\end{align}
\textbf{Proof of \eqref{eq: h property H h minus 0}:}
% Fix a sample path $\qty{x_0,\qty{Y_i}_{i=1}^\infty}$ where all the above assumptions and lemmas hold. 
$\forall x$,
\begin{align}
&\sup_{c\geq 1} \limsup_n \sup_{0\leq t_1 \leq t_2 \leq T_{n+1} - T_n} \norm{ \sum_{i = m(T_n+t_1)}^{m(T_n+t_2) - 1} \alpha(i) \left[H_c(x, Y_{i+1}) - h_c(x) \right] }   \\
=&\sup_{c\geq 1} \limsup_n \sup_{0\leq t_1 \leq t_2 \leq T_{n+1} - T_n} \norm{ \sum_{i = m(T_n+t_1)}^{m(T_n+t_2) - 1} \alpha(i) \left[\frac{H(cx, Y_{i+1})}{c} - \frac{h(cx)}{c} \right] }\\
=& \sup_{c\geq 1} \frac{1}{c} \limsup_n \sup_{0\leq t_1 \leq t_2 \leq T_{n+1} - T_n} \norm{ \sum_{i = m(T_n+t_1)}^{m(T_n+t_2) - 1} \alpha(i) \left[H(cx, Y_{i+1}) - h(cx) \right] }\\
\leq& \sup_{c\geq 1} \frac{1}{c} \limsup_n \sup_{0\leq t_1 \leq t_2 \leq T + \sup_j \alpha(j)} \norm{ \sum_{i = m(T_n+t_1)}^{m(T_n+t_2) - 1} \alpha(i) \left[H(cx, Y_{i+1}) - h(cx) \right] } \explain{$\forall n, T_{n+1} - T_n \leq T + \sup_j \alpha(j)$} \\
=& \sup_{c\geq 1} \frac{1}{c} \cdot 0 \explain{by Lemma~\ref{lemma: H h 0}}\\
=& 0 . \label{eq: lim H h t 0}\\ 
\end{align}  
\textbf{Proof of \eqref{eq: h property H bounded}:}
\begin{align}
&\limsup_n \sup_{0\leq t_1 \leq t_2 \leq T_{n+1} - T_n} \norm{ \sum_{i = m(T_n + t_1)}^{m(T_n+t_2) - 1} \alpha(i) H(0, Y_{i+1}) } \\
=& \limsup_n \sup_{0\leq t_1 \leq t_2 \leq T_{n+1} - T_n} \norm{ \sum_{i = m(T_n + t_1)}^{m(T_n+t_2) - 1} \alpha(i) [H(0, Y_{i+1}) - h(0) + h(0)] } \\
\leq& \limsup_n \sup_{0\leq t_1 \leq t_2 \leq T_{n+1} - T_n} \norm{ \sum_{i = m(T_n + t_1)}^{m(T_n+t_2) - 1} \alpha(i) [H(0, Y_{i+1}) - h(0)] } \\
&+   \limsup_n \sup_{0\leq t_1 \leq t_2 \leq T_{n+1} - T_n} \norm{ \sum_{i = m(T_n + t_1)}^{m(T_n+t_2) - 1} \alpha(i)  h(0) } \\
\leq& \limsup_n \sup_{0\leq t_1 \leq t_2 \leq T + \sup_j \alpha(j)} \norm{ \sum_{i = m(T_n + t_1)}^{m(T_n+t_2) - 1} \alpha(i) [H(0, Y_{i+1}) - h(0)] } \\
&+   \limsup_n \sup_{0\leq t_1 \leq t_2 \leq T + \sup_j \alpha(j)} \norm{ \sum_{i = m(T_n + t_1)}^{m(T_n+t_2) - 1} \alpha(i)  h(0) } \explain{$\forall n, T_{n+1} - T_n \leq T + \sup_j \alpha(j)$} \\
=& \limsup_n \sup_{0\leq t_1 \leq t_2 \leq T + \sup_j \alpha(j)} \norm{ \sum_{i = m(T_n + t_1)}^{m(T_n+t_2) - 1} \alpha(i)  h(0) } \explain{by Lemma~\ref{lemma: H h 0}}\\
=&\norm{h(0) } \limsup_n   \sup_{0\leq t_1 \leq t_2 \leq T+\sup_j \alpha(j)} \sum_{i = m(T_n + t_1)}^{m(T_n+t_2) - 1} \alpha(i) \\
% =&\norm{h(0) } \limsup_n  \sum_{i = m(T_n + t_1)}^{m(T_{n+1}) - 1} \alpha(i)   \\
=&\norm{h(0) } (T + \sup_j \alpha(j))   \explain{by Lemma~\ref{lemma: T-minus}}  \\
<& \infty. \label{eq: bounded-h0 inf y sample path}
\end{align}
We now consider $c$ in the above bounds. We first get
\begin{align}
&\sup_{c\geq 1}\sup_n \sup_{0\leq t_1 \leq t_2 \leq T_{n+1} - T_n} \norm{ \sum_{i = m(T_n + t_1)}^{m(T_n+t_2) - 1} \alpha(i) H_c(0, Y_{i+1}) }  \\
=&\sup_{c\geq 1}\sup_n \sup_{0\leq t_1 \leq t_2 \leq T_{n+1} - T_n} \norm{ \sum_{i = m(T_n + t_1)}^{m(T_n+t_2) - 1} \alpha(i) \frac{H(0, Y_{i+1}) }{c}} \explain{by \eqref{def: H c}} \\
=&\sup_n \sup_{0\leq t_1 \leq t_2 \leq T_{n+1} - T_n} \norm{ \sum_{i = m(T_n + t_1)}^{m(T_n+t_2) - 1} \alpha(i) H(0, Y_{i+1})} \explain{by $c\geq 1$} \\
<& \infty. \explain{by \eqref{eq: bounded-h0 inf y sample path}} 
\end{align}
\textbf{Proof of \eqref{eq: h property close t 0}:}
% For close $t_1, t_2$, 
\begin{align}
&\lim_{\delta \to 0^+}   \sup_{c\geq 1} \limsup_n \sup_{0\leq t_2-t_1 \leq \delta } \norm{ \sum_{i = m(T_n+t_1)}^{m(T_n+t_2) - 1} \alpha(i) H_c(0, Y_{i+1}) }  \\
\leq &\lim_{\delta \to 0^+}   \sup_{c\geq 1} \limsup_n \sup_{0\leq t_2-t_1 \leq \delta } \norm{ \sum_{i = m(T_n+t_1)}^{m(T_n+t_2) - 1} \alpha(i) \left[ H_c(0, Y_{i+1}) -h_c(0)\right] } \\
&+ \lim_{\delta \to 0^+}  \sup_{c\geq 1} \limsup_n \sup_{0\leq t_2-t_1 \leq \delta } \norm{ \sum_{i = m(T_n+t_1)}^{m(T_n+t_2) - 1} \alpha(i) h_c(0) }  \\
\leq & 0  + \lim_{\delta \to 0^+}   \sup_{c\geq 1} \limsup_n \sup_{0\leq t_2-t_1 \leq \delta } \norm{ \sum_{i = m(T_n+t_1)}^{m(T_n+t_2) - 1} \alpha(i) h_c(0) } \explain{by \eqref{eq: lim H h t 0}}   \\
\leq & 0  + \lim_{\delta \to 0^+}   \sup_{c\geq 1} \limsup_n \sup_{0\leq t_2-t_1 \leq \delta } \norm{ \sum_{i = m(T_n+t_1)}^{m(T_n+t_2) - 1} \alpha(i) \frac{h(0)}{c} }   \\
\leq & 0  + \norm{ h(0)} \lim_{\delta \to 0^+}  \sup_{c\geq 1}   \frac{1}{c}   \limsup_n \sup_{0\leq t_2-t_1 \leq \delta }  \sum_{i = m(T_n+t_1)}^{m(T_n+t_2) - 1} \alpha(i)     \\
\leq&   \norm{ h(0)} \lim_{\delta \to 0^+}   \sup_{c\geq 1} \frac{1}{c}  \delta  \explain{by \eqref{eq: lim alpha t_1 t_2}} \\
=&   \norm{ h(0)} \lim_{\delta \to 0^+}    \delta \\
=& 0.
\end{align}
\end{proof}

% \lsz{declare which equation I am proving}

% \lsz{The below lemma is first used in lemma \ref{lemma: hat x equicontinuous}}

\begin{lemma}\label{lemma: alpha L(y) bound}
\begin{align}
\sup_n \sup_{0\leq t_1 \leq t_2 \leq T_{n+1} - T_n} \qty( \sum_{i = m(T_n+t_1)}^{m(T_n+t_2) - 1} \alpha(i) L(Y_{i+1}) )   &< \infty \quad a.s.,  \label{eq: L property bounded}\\
\lim_{\delta \to 0^+}  \limsup_n \sup_{0\leq t_2-t_1 \leq \delta } \qty( \sum_{i = m(T_n+t_1)}^{m(T_n+t_2) - 1} \alpha(i) L(Y_{i+1}) )   & = 0 \quad a.s., \label{eq: L property close t 0}\\
\sup_n \sup_{0\leq t_1 \leq t_2 \leq T_{n+1} - T_n} \qty( \sum_{i = m(T_n+t_1)}^{m(T_n+t_2) - 1} \alpha(i) L_b(Y_{i+1}) )   &< \infty \quad a.s.  \label{eq: L b property bounded}\\
\end{align}
\end{lemma}
Its proof is similar to the proof of Lemma \ref{lemma: bounded-h0 inf y}
% in Appendix \ref{appendix: bounded-h0 inf y} 
and is thus omitted.
% Its proof is in Appendix \ref{appendix: alpha L(y) bound}.

% \lsz{The below lemma is first used in lemma \ref{lemma: hat x equicontinuous}}
\begin{lemma}\label{lemma: def C H}
Fix a sample path $\qty{x_0,\qty{Y_i}_{i=1}^\infty}$, 
there exists a constant $C_H$ such that
\begin{align}
&LT \leq C_H,  \label{eq: LT C_H} \\
&\sup_{c\geq 1} \norm{ h_c(0) }   \leq \frac{C_H}{T}, \label{eq: h c c_H} \\
&\sup_{c\geq 1} \sup_n \sup_{0\leq t_1 \leq t_2 \leq T_{n+1} - T_n} \norm{ \sum_{i = m(T_n+t_1)}^{m(T_n+t_2) - 1} \alpha(i) H_c(0, Y_{i+1}) }   \leq C_H, \label{eq: H c c_H}\\
&\sup_n  \sup_{0\leq t_1 \leq t_2 \leq T_{n+1} - T_n} \sum_{i = m(T_n+t_1)}^{m(T_n+t_2) - 1} \alpha(i) L(Y_{i+1})   \leq C_H. \label{eq: alpha L bound C_H}
\end{align}
Moreover, for the presentation convenience, we denote 
\begin{align}
C_{\hat{x}} \doteq \left[1  + C_H\right] e^{C_H}. \label{def: C hat x}
\end{align}
\end{lemma}
% Its proof is in Appendix \ref{appendix: def C H}. 
% \subsubsection{Proof of Lemma~\ref{lemma: def C H}}\label{appendix: def C H}
\begin{proof}
Fix a sample path $\qty{x_0,\qty{Y_i}_{i=1}^\infty}$, 
\begin{align}
&\explaind{LT < \infty,}{$L$ and $T$ are constants}   \\
&\explaind{\sup_{c\geq 1} \norm{ h_c(0) }T < \infty,}{by \eqref{eq: h property h}}  \\
&\explaind{\sup_{c\geq 1} \sup_n \sup_{0\leq t_1 \leq t_2 \leq T_{n+1} - T_n} \norm{ \sum_{i = m(T_n+t_1)}^{m(T_n+t_2) - 1} \alpha(i) H_c(0, Y_{i+1}) }   < \infty,}{by \eqref{eq: h property H bounded}} \\
&\explaind{\sup_n  \sup_{0\leq t_1 \leq t_2 \leq T_{n+1} - T_n} \sum_{i = m(T_n+t_1)}^{m(T_n+t_2) - 1} \alpha(i) L(Y_{i+1})   < \infty.}{by \eqref{eq: L property bounded}} 
\end{align}
Thus, there exists a constant $C_H$ such that
\begin{align}
&LT \leq C_H \\
&\sup_{c\geq 1} \norm{ h_c(0) }   \leq \frac{C_H}{T}, \\
&\sup_{c\geq 1} \sup_n \sup_{0\leq t_1 \leq t_2 \leq T_{n+1} - T_n} \norm{ \sum_{i = m(T_n+t_1)}^{m(T_n+t_2) - 1} \alpha(i) H_c(0, Y_{i+1}) }   \leq C_H,\\
&\sup_n  \sup_{0\leq t_1 \leq t_2 \leq T_{n+1} - T_n} \sum_{i = m(T_n+t_1)}^{m(T_n+t_2) - 1} \alpha(i) L(Y_{i+1})   \leq C_H.
\end{align}

\end{proof}

% \lsz{The below lemma is first used in lemma \ref{lemma: hat x equicontinuous}}
\begin{lemma}\label{lemma: bound x hat}
$\sup_{n, t \in [0, T)}   \norm{\hat{x}(T_n + t)} \leq C_{\hat x}$. 
\end{lemma}
% Its proof is in Appendix \ref{appendix: bound x hat}. 
% \subsubsection{Proof of Lemma~\ref{lemma: bound x hat}}\label{appendix: bound x hat}
\begin{proof}
 % $C_{\hat{x}}$ is a constant independent of $c,n,t$. 
% For any sample path $\qty{x_0,\qty{Y_i}_{i=1}^\infty}$,
$\forall n \in \N, t \in [0, T)
$,
\begin{flalign}
&\norm{\hat{x}(T_n + t)}  & \\
=&  \norm{ \hat{x}(T_n) + \sum_{i = m(T_n)}^{m(T_n+t) - 1} \alpha(i) H_{r_n}(\hat{x}(t(i)), Y_{i+1})}   & \\
\leq & \norm{\hat{x}(T_n)}  + \norm{\sum_{i = m(T_n)}^{m(T_n+t) - 1} \alpha(i) H_{r_n}(\hat{x}(t(i)), Y_{i+1})}  & \\
=& \norm{\hat{x}(T_n)}  + \norm{\sum_{i = m(T_n)}^{m(T_n+t) - 1} \alpha(i) \left[ H_{r_n}(\hat{x}(t(i)), Y_{i+1})  -  H_{r_n}(0, Y_{i+1}) \right] + \sum_{i = m(T_n)}^{m(T_n+t) - 1} \alpha(i) H_{r_n}(0, Y_{i+1}) } & \\
\leq & \norm{\hat{x}(T_n)}  + \sum_{i = m(T_n)}^{m(T_n+t) - 1} \alpha(i)\norm{H_{r_n}(\hat{x}(t(i)), Y_{i+1})  -  H_{r_n}(0, Y_{i+1}) } +  \norm{ \sum_{i = m(T_n)}^{m(T_n+t) - 1} \alpha(i)H_{r_n}(0, Y_{i+1}) }  & \\
\leq & \norm{\hat{x}(T_n)}  + \sum_{i = m(T_n)}^{m(T_n+t) - 1} \alpha(i)L(Y_{i+1})\norm{\hat{x}(t(i))} +  \norm{ \sum_{i = m(T_n)}^{m(T_n+t) - 1} \alpha(i)H_{r_n}(0, Y_{i+1}) }  & \\
\leq & \norm{\hat{x}(T_n)}  + \sum_{i = m(T_n)}^{m(T_n+t) - 1} \alpha(i)L(Y_{i+1})\norm{\hat{x}(t(i))} +  C_H  &  \explain{by \eqref{eq: H c c_H}}\\
\leq &1  + \sum_{i = m(T_n)}^{m(T_n+t) - 1} \alpha(i)L(Y_{i+1})\norm{\hat{x}(t(i))} +  C_H  &  \explain{by \eqref{eq: hat-x-norm-1}}\\
\leq &\left[1  +  C_H  \right] e^{\sum_{i = m(T_n)}^{m(T_n+t) - 1} \alpha(i)L(Y_{i+1})} \explain{by $\hat{x}(T_n + t) = \hat{x}(t(m(T_n + t)))$ and  discrete Gronwall inequality in Theorem \ref{theorem: discrete Gronwall}} & \\
\leq &\left[ 1  +  C_H  \right] e^{C_H} \explain{by \eqref{eq: alpha L bound C_H}} & \\
= & C_{\hat{x}}.\explain{by \eqref{def: C hat x}} 
\end{flalign}
% Because $C_{\hat{x}}$ is independent of $n,t$, $\sup_{n, t \in [0, T)}   \norm{\hat{x}(T_n + t)} < \infty$.

\end{proof}

% \lsz{The below lemma is first used in lemma \ref{lemma: hat x equicontinuous}}
% \begin{lemma}\label{lemma: bound h x hat}
% $\sup_{c \geq 1, n, t \in [0, T)} \norm{h_c(\hat{x}(T_n+t))} < \infty$.
% \end{lemma}
% % Its proof is in Appendix \ref{appendix: bound h x hat}.
% % \subsubsection{Proof of Lemma~\ref{lemma: bound h x hat}}\label{appendix: bound h x hat}
% \begin{proof}
% % For a sample path where Assumption \ref{assumption: bounded-h0} holds and a fixed $n$, 
% $\forall c\geq 1, n, t \in [0, T)$,
% \begin{align}
% &\norm{h_c(\hat{x}(T_n+t))} \\
% \leq& \norm{h_c(\hat{x}(T_n+t)) -  h_c(0)} + \norm{h_c(0)}\\
% \leq& L\norm{\hat{x}(T_n + t)}  + \norm{h_c(0)} \explain{by Lemma \ref{lemma: h Lipschitz}} \\
% \leq& L C_{\hat{x}}  + \norm{h_c(0)} \explain{by Lemma~\ref{lemma: bound x hat}}  \\
% \leq& L C_{\hat{x}} + \frac{C_H}{T}. \explain{by \eqref{eq: h c c_H}} 
% \end{align}
% Because $C_{\hat{x}}, C_H$ are independent of $c,n,t$, $\sup_{c \geq 1, n, t \in [0, T)} \norm{h_c(\hat{x}(T_n+t))} <\infty$.
% \end{proof}

% \lsz{The below lemma is first used in lemma \ref{lemma: z n equicontinuous}}
\begin{lemma}\label{lemma: bound z}
$\sup_{n, t \in [0, T)}   \norm{z_n(t)} \leq C_{\hat x}$.
\end{lemma}
% Its proof is in appendix \ref{appendix: bound z}.
% \subsubsection{Proof of Lemma~\ref{lemma: bound z}}\label{appendix: bound z}
\begin{proof}
$\forall n, t \in [0, T)
$,
\begin{flalign}
&\norm{z_n(t)}&  \\
=&  \norm{ z_n(0) + \int_0^t  h_{r_n}(z_n(s)) ds}&   \\
\leq & \norm{z_n(0)}  + \norm{\int_0^t  h_{r_n}(z_n(s))ds}& \\
% =&  \norm{z_n(0)}  + \norm{\int_0^t  \left[ h_{r_n}(z_n(s))  -  h_{r_n}(0) \right]ds + \int_0^t  h_{r_n}(0) ds }& \\
\leq & \norm{z_n(0)}  + \int_0^t \norm{h_{r_n}(z_n(s))  -  h_{r_n}(0) }ds + \int_0^t  \norm{ h_{r_n}(0) }ds & \\
\leq & \norm{z_n(0)}  +  \int_0^t L\norm{z_n(s)} ds + \int_0^t  \norm{ h_{r_n}(0) }ds&  \explain{by Lemma  \ref{lemma: h Lipschitz}} \\
\leq & \norm{z_n(0)}  +  \int_0^t L\norm{z_n(s)} ds + T  \norm{  h_{r_n}(0)} &  \\
\leq & \norm{z_n(0)}  +  \int_0^t L\norm{z_n(s)} ds + T  \frac{C_H}{T}  \explain{by \eqref{eq: h c c_H}} \\
\leq & 1 +  \int_0^t L\norm{z_n(s)} ds + C_H \explain{by \eqref{eq: hat-x-norm-1}, \eqref{def: z n init}} \\
\leq &\left[1 +  C_H \right] e^{LT}& \explain{by Gronwall inequality in Theorem \ref{theorem: Gronwall}}\\
\leq &\left[1  +  C_H \right] e^{C_H}& \explain{by \eqref{eq: LT C_H}}\\
=& C_{\hat{x}} \explain{by \eqref{def: C hat x}}
% & \label{eq: bound z n}
\end{flalign}
% Because $ C_{\hat{x}}$ is independent of $n,t$, $\sup_{n, t \in [0, T)}   \norm{z_n(t)} < \infty$.

\end{proof}

% \lsz{The below lemma is first used in lemma \ref{lemma: three convergence}}
\begin{lemma}\label{lemma: bar x bounded}
$\forall n,$ 
\begin{align}
\norm{ \bar{x}(T_{n+1})} \leq \left(\norm{ \bar{x}(T_n)}  C_H  + C_H \right) e^{C_H} + \norm{ \bar{x}(T_n)} 
\end{align}
where $C_H$ is a positive constant defined in Lemma \ref{lemma: def C H}.
\end{lemma}
% Its proof is in Appendix \ref{appendix: bar x bounded}. 
% \subsubsection{Proof of Lemma~\ref{lemma: bar x bounded}}\label{appendix: bar x bounded}
\begin{proof}
% To show 
% $\forall n,$ 
% \begin{align}
% \norm{ \bar{x}(T_{n+1})} \leq \left(\norm{ \bar{x}(T_n)}  C_H  + C_H \right) e^{C_H} + \norm{ \bar{x}(T_n)}. 
% % =   \norm{ \bar{x}(T_n)}  (C_H e^{C_H} + 1) + C_H e^{C_H} 
% \end{align}
We first show the difference between  $\bar{x}(T_{n+1})$ and $\bar{x}(T_n)$ by the following derivations.
$\forall n$, 
$\forall t \in [0, T_{n+1} - T_n]$,
\begin{align}
&\norm{\bar{x}(T_n + t) - \bar{x}(T_n)}  \\
=& \norm{\bar{x}(t(m(T_n + t))) - \bar{x}(T_n)}  \\
=&  \norm{ \bar{x}(T_n) + \sum_{i = m(T_n)}^{m(T_n+t) - 1} \alpha(i) H(\bar{x}(t(i)), Y_{i+1}) - \bar{x}(T_n)}   \\
= &\norm{\sum_{i = m(T_n)}^{m(T_n+t) - 1} \alpha(i) H(\bar{x}(t(i)), Y_{i+1})}  \\
% =&\norm{\sum_{i = m(T_n)}^{m(T_n+t) - 1} \alpha(i) \left[ H(\bar{x}(t(i)), Y_{i+1})  -  H( \bar{x}(T_n), Y_{i+1}) \right] + \sum_{i = m(T_n)}^{m(T_n+t) - 1} \alpha(i) H( \bar{x}(T_n), Y_{i+1}) } \\
\leq &\sum_{i = m(T_n)}^{m(T_n+t) - 1} \alpha(i)\norm{H(\bar{x}(t(i)), Y_{i+1})  -  H( \bar{x}(T_n), Y_{i+1}) } + \norm{\sum_{i = m(T_n)}^{m(T_n+t) - 1} \alpha(i)  H( \bar{x}(T_n), Y_{i+1}) }  \\
\leq &\sum_{i = m(T_n)}^{m(T_n+t) - 1} \alpha(i)L(Y_{i+1})\norm{\bar{x}(t(i)) -  \bar{x}(T_n)}  + \norm{\sum_{i = m(T_n)}^{m(T_n+t) - 1} \alpha(i)  H( \bar{x}(T_n), Y_{i+1}) }  \\
\leq &\sum_{i = m(T_n)}^{m(T_n+t) - 1} \alpha(i)L(Y_{i+1})\norm{\bar{x}(t(i)) -  \bar{x}(T_n)}  + \sum_{i = m(T_n)}^{m(T_n+t) - 1} \alpha(i) \norm{ H( \bar{x}(T_n), Y_{i+1}) - H( 0, Y_{i+1}) } \\
&+  \norm{ \sum_{i = m(T_n)}^{m(T_n+t) - 1} \alpha(i)  H(0, Y_{i+1}) }\\
\leq &\sum_{i = m(T_n)}^{m(T_n+t) - 1} \alpha(i)L(Y_{i+1})\norm{\bar{x}(t(i)) -  \bar{x}(T_n)}  + \sum_{i = m(T_n)}^{m(T_n+t) - 1} \alpha(i)  L(Y_{i+1}) 
\norm{ \bar{x}(T_n)} \\
&+  \norm{ \sum_{i = m(T_n)}^{m(T_n+t) - 1} \alpha(i)  H(0, Y_{i+1}) }   \explain{by Assumption \ref{assumption: H Lipschitz}}\\
= &\sum_{i = m(T_n)}^{m(T_n+t) - 1} \alpha(i)L(Y_{i+1})\norm{\bar{x}(t(i)) -  \bar{x}(T_n)}  + \norm{ \bar{x}(T_n)}  \sum_{i = m(T_n)}^{m(T_n+t) - 1} \alpha(i)  L(Y_{i+1}) \\
&+  \norm{ \sum_{i = m(T_n)}^{m(T_n+t) - 1} \alpha(i)  H(0, Y_{i+1}) } \\
\leq &\sum_{i = m(T_n)}^{m(T_n+t) - 1} \alpha(i)L(Y_{i+1})\norm{\bar{x}(t(i)) -  \bar{x}(T_n)}  + \norm{ \bar{x}(T_n)}  C_H  +  \norm{ \sum_{i = m(T_n)}^{m(T_n+t) - 1} \alpha(i)  H(0, Y_{i+1}) } \explain{by \eqref{eq: alpha L bound C_H}} \\
\leq & \sum_{i = m(T_n)}^{m(T_n+t) - 1} \alpha(i)L(Y_{i+1})\norm{\bar{x}(t(i)) -  \bar{x}(T_n)}  +  \left[\norm{ \bar{x}(T_n)}  C_H  +  C_H  \right] \explain{by \eqref{eq: H c c_H}} \\
\leq &\left[\norm{ \bar{x}(T_n)}  C_H  +  C_H  \right]  e^{\sum_{i = m(T_n)}^{m(T_n+t) - 1} \alpha(i)L(Y_{i+1})} \explain{by discrete Gronwall inequality in Theorem \ref{theorem: discrete Gronwall}}\\
\leq &  \explaind{\left[\norm{ \bar{x}(T_n)}  C_H  +  C_H  \right] e^{C_H}}{by \eqref{eq: alpha L bound C_H}} \label{eq: bar x bound}
\end{align}

\end{proof}

% \lsz{The below lemma is first used in lemma \ref{lemma: three convergence}}
\begin{lemma}\label{lemma: r w r' k infty}
\begin{align}
\lim\sup_k r_{n_{2, k}} = \infty. 
\end{align}
\end{lemma}
% Its proof is in Appendix \ref{appendix: r w r' k infty}. 
% \subsubsection{Proof of Lemma~\ref{lemma: r w r' k infty}}\label{appendix: r w r' k infty}
\begin{proof}
We use proof by contradiction. Suppose 
\begin{align}
\lim\sup_k r_{n_{2, k}} = C_r < \infty
\end{align}
where $C_r$ is a constant.
$\forall \epsilon > 0$, $\exists k_0$ such that $\forall k \geq k_0$,
\begin{align}
r_{n_{2, k}} \leq C_r + \epsilon.
\end{align}
By Lemma \ref{lemma: bar x bounded}, $\forall k \geq k_0$,
\begin{align}
r_{n_{1, k}} =& \max \qty{\norm{\bar{x}(T_{n_{1, k}})},  1}  \explain{by \eqref{def: r n}} \\
=& \max \qty{\norm{\bar{x}(T_{n_{2, k} + 1})},  1}  \explain{by \eqref{def: w r' k}} \\
\leq& \norm{\bar{x}(T_{n_{2, k} + 1})} + 1 \\ 
\leq& \left(\norm{ \bar{x}(T_{n_{2, k}})}  C_H  + C_H \right) e^{C_H} + \norm{ \bar{x}(T_{n_{2, k}})} + 1 \\   
\leq& \left(r_{n_{2, k}} C_H  + C_H \right) e^{C_H} + r_{n_{2, k}} + 1 \\  
\leq& \left[(C_r + \epsilon) C_H  + C_H \right] e^{C_H} +(C_r + \epsilon) + 1 \\  
<& \infty.
\end{align}
This contradicts \eqref{eq: r w r k infty}.
Thus, 
\begin{align}
\lim\sup_k r_{n_{2, k}} = \infty. 
\end{align}
\end{proof}

%%%For lemma \ref{lemma: h k z k h inf z lim uniform}
% \lsz{The below lemma is first used in lemma \ref{lemma: three convergence}}
\begin{lemma}\label{lemma: bound h z}
$\sup_{n, t \in [0, T)} \norm{h_{r_n}(z_n(t))} < \infty$.
\end{lemma}
% Its proof is in appendix \ref{appendix: bound h z}.
% {Proof of Lemma~\ref{lemma: bound h z}}\label{appendix: bound h z}
\begin{proof}
% For a sample path where Assumption \ref{assumption: bounded-h0} holds and a fixed $n$, 
$\forall n, \forall t \in [0, T)$,
\begin{align}
&\norm{h_{r_n}(z_n(t))} \\
\leq& \norm{h_{r_n}(z_n(t)) -  h_{r_n}(0)} + \norm{h_{r_n}(0)}\\
\leq& L\norm{z_n(t)}  + \norm{h_{r_n}(0)} \explain{by Lemma \ref{lemma: h Lipschitz}} \\
\leq& L C_{\hat{x}}   + \norm{h_{r_n}(0)} \explain{by Lemma~\ref{lemma: bound z}}  \\
\leq& L C_{\hat{x}}    +  \frac{C_H}{T}.  \explain{by \eqref{def: r n} and \eqref{eq: h c c_H}} 
\end{align}
Thus, because $C_{\hat{x}}, C_H$ are independent of $n,t$, $\sup_{n, t \in [0, T)} \norm{h_{r_n}(z_n(t))} < \infty$.

\end{proof}

% \lsz{The below lemma is first used in lemma \ref{lemma: three convergence}}
\begin{lemma}\label{lemma: z lim bound}
$\sup_{t \in [0, T)}   \norm{z^{\lim}(t)} \leq C_{\hat x}$.    
\end{lemma}
% Its proof is in Appendix \ref{appendix: z lim bound}.
% \subsubsection{Proof of Lemma~\ref{lemma: z lim bound}}\label{appendix: z lim bound}
\begin{proof}
$\forall t \in [0, T)$,
\begin{flalign}
&\norm{z^{\lim}(t)}&  \\
=&  \norm{ z^{\lim}(0) + \int_0^t  h_{\infty}(z^{\lim}(s)) ds}&   \\
\leq & \norm{z^{\lim}(0)}  + \norm{\int_0^t  h_{\infty}(z^{\lim}(s))ds}& \\
=&  \norm{z^{\lim}(0)}  + \norm{\int_0^t  \left[ h_{\infty}(z^{\lim}(s))  -  h_{\infty}(0) \right]ds + \int_0^t  h_{\infty}(0) ds }& \\
\leq & \norm{z^{\lim}(0)}  + \int_0^t \norm{h_{\infty}(z^{\lim}(s))  -  h_{\infty}(0) }ds + \int_0^t  \norm{ h_{\infty}(0) }ds & \\
\leq & \norm{z^{\lim}(0)}  +  \int_0^t L\norm{z^{\lim}(s)} ds + \int_0^t  \norm{ h_{\infty}(0) }ds&  \explain{by Lemma \ref{lemma: h Lipschitz}} \\
\leq &1  +  \int_0^t L\norm{z^{\lim}(s)} ds + \int_0^t  \norm{ h_{\infty}(0) }ds&
\explain{by \eqref{eq: hat-x-norm-1}, \eqref{def: z n init}} \\
\leq &1  +  \int_0^t L\norm{z^{\lim}(s)} ds + T\norm{ h_{\infty}(0) } \\
\leq &1  +  \int_0^t L\norm{z^{\lim}(s)} ds + C_H 
\explain{by Assumption \ref{assumption: lim h uniformly convergent} and  \eqref{eq: h c c_H}} \\
\leq &\left[1 +  C_H  \right] e^{\int_0^t L ds }& \explain{by Gronwall inequality in Theorem \ref{theorem: Gronwall}}\\
\leq &\left[1 + C_H \right] e^{LT} &  \\ 
\leq& C_{\hat{x}}. \explain{by \eqref{eq: LT C_H}, \eqref{def: C hat x}}& 
\end{flalign}
% Because $ C_{\hat{x}}$ is a constant which is independent of $t$, $\sup_{t \in [0, T)}   \norm{z^{\lim}(t)} < \infty$.

\end{proof}

% \lsz{The below lemma is first used in lemma \ref{lemma: three convergence}}
\begin{lemma}\label{lemma: h c z lim uniform}
$\lim_{k \to \infty}h_{r_{n_k}}(z^{\lim}(t)) = h_{\infty}(z^{\lim}(t))$ uniformly in $t\in [0, T)$.
\end{lemma}
% Its proof is in Appendix \ref{appendix: h c z lim uniform}.
% \subsubsection{Proof of Lemma~\ref{lemma: h c z lim uniform}}\label{appendix: h c z lim uniform}

\begin{proof}
By Assumption \ref{assumption: lim h uniformly convergent}, $\lim_{k \to \infty}h_{r_{n_k}}(v) = h_{\infty}(v)$ uniformly in a compact set $\qty{v \vert v\in \R^d, \norm{v}\leq C_x }$. 
By Lemma \ref{lemma: z lim bound}, $\qty{z^{\lim}(t) \vert t \in [0,T)}  \subseteq \qty{v \vert v\in \R^d, \norm{v}\leq C_x }$.
Therefore, $\lim_{k \to \infty}h_{r_{n_k}}(z^{\lim}(t)) = h_{\infty}(z^{\lim}(t))$  uniformly in $\qty{z^{\lim}(t) \vert t \in [0,T)}$ and on $t \in [0,T)$.
\end{proof}

% \lsz{The below lemma is first used in lemma \ref{lemma: three convergence}}
\begin{restatable}[]{lemma}{znklimit}
% \begin{lemma}
\label{lemma: z n limit}
$\forall t \in [0, T)$, we have
\begin{align}
\lim_{k \to \infty} z_{n_k}(t)   =   z^{\lim}(t).
\end{align}
Moreover, the convergence is uniform in $t$ on $[0, T)$.
% \end{lemma}
\end{restatable}
% Its proof is in Appendix \ref{appendix: z n limit}.
% \subsubsection{Proof of Lemma~\ref{lemma: z n limit}}\label{appendix: z n limit}
\begin{proof}
By \eqref{eq: hat x x r lim}, $\forall \delta > 0$, there exists a $k_1$ such that $\forall k \geq k_1$, $\forall t \in [0,T)$,
\begin{align}
\norm{\hat{x}(T_{n_k} + t)  -     \hat{x}^{\lim}(t)} \leq \delta .\label{eq: z lim epsilion 1}
\end{align}
By Lemma \ref{lemma: h c z lim uniform}, there exists a $k_2$ such that $\forall k \geq k_2$, $\forall t \in [0,T)$,
\begin{align}
\norm{h_{r_{n_k}}(z^{\lim}(t))  -   h_{\infty}(z^{\lim}(t))}  \leq \delta. \label{eq: z lim epsilion 2}  
\end{align}
$\forall k \geq \max\qty{ k_1, k_2}$, $\forall t \in [0,T)$
\begin{align}
& \norm{z_{n_k}(t)  - z^{\lim}(t) } \\
=& \norm{\hat{x}(T_{n_k}) + \int_0^t h_{r_{n_k}}(z_{n_k}(s))ds  -     \hat{x}^{\lim}(0)  -  \int_0^t h_{\infty}(z^{\lim}(s))ds} \\
\leq&  \norm{\hat{x}(T_{n_k})  -     \hat{x}^{\lim}(0)}  + \norm{\int_0^t h_{r_{n_k}}(z_{n_k}(s))ds -  \int_0^t h_{\infty}(z^{\lim}(s))ds} \\
\leq& \delta + \norm{\int_0^t h_{r_{n_k}}(z_{n_k}(s)) -   h_{\infty}(z^{\lim}(s))ds} \explain{by \eqref{eq: z lim epsilion 1}} \\
% \leq& \delta + \norm{\int_0^t h_{r_{n_k}}(z_{n_k}(s)) - h_{r_{n_k}}(z^{\lim}(s))   +h_{r_{n_k}}(z^{\lim}(s))  -   h_{\infty}(z^{\lim}(s))ds} \\
\leq& \delta + \int_0^t  \norm{h_{r_{n_k}}(z_{n_k}(s)) - h_{r_{n_k}}(z^{\lim}(s))  }ds + \int_0^t \norm{  h_{r_{n_k}}(z^{\lim}(s))  -   h_{\infty}(z^{\lim}(s))} ds \\
\leq& \delta + L \int_0^t   \norm{z_{n_k}(s) -z^{\lim}(s)  }ds + \int_0^t \norm{  h_{r_{n_k}}(z^{\lim}(s))  -   h_{\infty}(z^{\lim}(s))} ds \explain{by Lemma \ref{lemma: h Lipschitz}} \\
\leq& \delta + t\delta +  L \int_0^t   \norm{z_{n_k}(s) -z^{\lim}(s)  }ds \explain{by \eqref{eq: z lim epsilion 2}} \\
\leq& (\delta + t\delta)e^{Lt}  \explain{by Gronwall inequality in Theorem \ref{theorem: Gronwall}}\\
\leq& (\delta + T\delta)e^{LT}, \label{eq: z z lim difference}
\end{align}
which completes the proof.
% Thus, $\forall \epsilon >0$,  setting $\delta = \frac{\epsilon}{(1+T)e^{LT}}$, there exists a $k_3 \geq \max\qty{ k_1,k_2}$, such that $\forall k \geq k_3$, $\forall t \in [0,T)$,
% \begin{align}
% & \norm{z_{n_k}(t)  - z^{\lim}(t) }  \leq (\delta + T\delta)e^{LT} \leq \epsilon.
% \end{align}
\end{proof}

%%%

%%%
% \lsz{The below lemma is first used in lemma \ref{lemma: single limit}}

\begin{lemma}\label{lemma: split limit to double limit}
For any function $f: \R \times \R \to \R$, 
if $\lim \limits_{\begin{subarray}{l}
a \to \infty\\
b \to \infty
\end{subarray}} f(a,b) = L
$    then $\lim \limits_{c \to \infty} f(c,c) = L$ where $L$ is a constant.
\end{lemma}
% Its proof is in Appendix \ref{appendix: split limit to double limit}.

% \subsubsection{Proof of Lemma~\ref{lemma: split limit to double limit}}\label{appendix: split limit to double limit}

\begin{proof}
By definition, $\forall \epsilon > 0, \exists a_0, b_0$ such that $\forall a > a_0, b > b_0$, $\norm{f(a,b) - L } < \epsilon$. Thus, $\forall \epsilon > 0, \exists c_0 = \max \qty{a_0, b_0}$ such that $\forall c > c_0$, $\norm{f(c,c) - L } < \epsilon$.
\end{proof}

% \lsz{The below lemma is first used in lemma \ref{lemma: f k 0}}
\begin{lemma}\label{lemma: h k x lim h inf x lim uniform}
$\forall t \in [0,T)$,
\begin{align}
\lim_{k \to \infty}  \int_0^t h_{r_{n_k}}(\hat{x}^{\lim}(s))ds =
 \int_0^t  h_{\infty}(\hat{x}^{\lim}(s))ds. \label{eq: f limit interchange limit int 1}    
\end{align}
% uniformly in $t \in [0,T)$.
\end{lemma}
% Its proof is in Appendix \ref{appendix: h k x lim h inf x lim uniform}.
% \subsubsection{Proof of Lemma~\ref{lemma: h k x lim h inf x lim uniform}}\label{appendix: h k x lim h inf x lim uniform}
\begin{proof}
From Lemma~\ref{lemma: bound x hat},
it is easy to see that
\begin{align}
    \sup_{t\in[0, T)} \norm{\hat x^{\lim}(t)} < \infty,
\end{align}
which, similar to Lemma~\ref{lemma: bound h z}, implies that
\begin{align}
    \sup_{k, t\in[0, T)} \norm{h_{r_{n_k}}\left(\hat x^{\lim}(t)\right)} < \infty.
\end{align}
% By Lemma \ref{lemma: bound h x hat}, $\norm{h_{r_{n_k}} (\hat{x}^{\lim}(s))}$ is uniformly bounded on $k$. 
% By Theorem  \ref{appendix: interchange integral and limit},  
By the dominated convergence theorem,
$\forall t\in [0,T)$,
\begin{align}
&  \lim_{k \to \infty}   \int_0^t h_{r_{n_k}} (\hat{x}^{\lim}(s))ds = \int_0^t \lim_{k \to \infty}    h_{r_{n_k}} (\hat{x}^{\lim}(s))ds = \int_0^t  h_{\infty} (\hat{x}^{\lim}(s))ds,
\end{align}
which completes the proof.
\end{proof}

% \lsz{The below lemma is first used in lemma \ref{lemma: f k 0}}
\begin{lemma}\label{lemma: h k z k h inf z lim uniform}
$\forall t \in [0,T)$,
\begin{align}
\lim_{k \to \infty}  \int_0^t h_{r_{n_k}}(z_{n_k}(s))ds =
 \int_0^t  h_{\infty}(z^{\lim}(s))ds. \label{eq: f limit interchange limit int 2}    
\end{align}
\end{lemma}
% Its proof is in Appendix \ref{appendix: h k z k h inf z lim uniform}.
% \subsubsection{Proof of Lemma~\ref{lemma: h k z k h inf z lim uniform}}\label{appendix: h k z k h inf z lim uniform}
\begin{proof}
% With Lemma \ref{lemma: bound h z}, $\norm{h_{r_{n_k}}(z_{n_k}(s))}$ is bounded uniformly  on $k$.   
% % By Theorem  \ref{appendix: interchange integral and limit}, 
% $\forall t \in [0,T)$,
% \begin{align}
% &  \lim_{k \to \infty}  \int_0^t h_{r_{n_k}}(z_{n_k}(s))ds =    \int_0^t \lim_{k \to \infty}  h_{r_{n_k}}(z_{n_k}(s))ds. 
% \end{align}
$\forall \epsilon>0$,  by Lemma \ref{lemma: h c z lim uniform}, $\exists k_0$ such that $\forall k \geq k_0$, $\forall t \in [0,T)$, 
\begin{align}
\norm{  h_{r_{n_k}}(z^{\lim}(s)) -   h_{\infty}(z^{\lim}(s)) } \leq \epsilon. \label{eq: f k 0 k_0}
\end{align}
By Lemma \ref{lemma: z n limit}, $\exists k_1$ such that $\forall k \geq k_1$, $\forall t \in [0,T)$, 
\begin{align}
\norm{ z_{n_k}(t)-  z^{\lim}(t) } \leq \epsilon. \label{eq: f k 0 k_1}
\end{align}
Thus, $\forall k \geq \max \qty{k_0,k_1}$, $\forall t \in [0,T)$, 
\begin{align}
&  \norm{ \int_0^t h_{r_{n_k}}(z_{n_k}(s))ds -
 \int_0^t  h_{\infty}(z^{\lim}(s))ds } \\
\leq&    \norm{ \int_0^t h_{r_{n_k}}(z_{n_k}(s))ds - \int_0^t h_{r_{n_k}}(z^{\lim}(s))ds}  +  \norm{ \int_0^t h_{r_{n_k}}(z^{\lim}(s))ds -  \int_0^t  h_{\infty}(z^{\lim}(s))ds } \\
\leq&   \int_0^t  \norm{h_{r_{n_k}}(z_{n_k}(s)) - h_{r_{n_k}}(z^{\lim}(s))}ds   +  \int_0^t  \norm{h_{r_{n_k}}(z^{\lim}(s)) -    h_{\infty}(z^{\lim}(s)) }ds  \\
\leq& \int_0^t  \norm{h_{r_{n_k}}(z_{n_k}(s)) -  h_{r_{n_k}}(z^{\lim}(s))} ds   + T\epsilon \explain{by \eqref{eq: f k 0 k_0}} \\
\leq& \int_0^t  L\norm{z_{n_k}(s) -  z^{\lim}(s)} ds   + T\epsilon \explain{by Lemma \ref{lemma: h Lipschitz}}\\
\leq& L T\epsilon    + T\epsilon. \explain{by \eqref{eq: f k 0 k_1}}
\end{align}
% This means $\forall \delta>0$, by setting $\epsilon = \frac{\delta}{LT+T}$, $k_2 = \max \qty{k_0,k_1}$, we have $\forall k \geq k_2$, $\forall t \in [0,T)$,
% \begin{align}
% &  \norm{ \int_0^t h_{r_{n_k}}(z_{n_k}(s))ds -
%  \int_0^t  h_{\infty}(z^{\lim}(s))ds } \leq   L T\epsilon    + T\epsilon \leq \delta.
%  \end{align}
% This shows
Thus, $\forall t \in [0,T)$,
\begin{align}
\lim_{k \to \infty}  \int_0^t h_{r_{n_k}}(z_{n_k}(s))ds =
 \int_0^t  h_{\infty}(z^{\lim}(s))ds.  
\end{align}
% uniformly in $t \in [0,T)$.
\end{proof}

% \lsz{The below lemma is first used in lemma \ref{lemma: contradiction}}

\begin{lemma}\label{lemma: lim n t to T - 0}
\begin{align}
\lim_{n} \lim_{t \to T^-}\norm{\sum_{i = m(T_n+ t)}^{m(T_{n + 1}) - 1} \alpha(i) L(Y_{i+1})} = 0, \label{eq: t T^- L 0} \\
\lim_{n} \lim_{t \to T^-} \norm{\sum_{i = m(T_n+ t)}^{m(T_{n + 1}) - 1} \alpha(i) H(0, Y_{i+1})  } = 0. \label{eq: t T^- H 0}
\end{align}
\end{lemma}
% Its proof is in Appendix \ref{appendix: lim n t to T - 0}.
% \subsubsection{Proof of Lemma~\ref{lemma: lim n t to T - 0}}\label{appendix: lim n t to T - 0}

\begin{proof}
\begin{align}
&\limsup_{n} \lim_{t \to T^-}\norm{\sum_{i = m(T_n+ t)}^{m(T_{n + 1}) - 1} \alpha(i) L(Y_{i+1})}  \\
=& \limsup_{n} \lim_{t \to T^-}\norm{\sum_{i = m(T_n+ t)}^{m(T_{n + 1}) - 1} \alpha(i) [L(Y_{i+1})-L] + \sum_{i = m(T_n+ t)}^{m(T_{n + 1}) - 1} \alpha(i) L}  \\
\leq& \limsup_{n} \lim_{t \to T^-}\norm{\sum_{i = m(T_n+ t)}^{m(T_{n + 1}) - 1} \alpha(i) [L(Y_{i+1})-L]} +  \limsup_{n} \lim_{t \to T^-}\norm{\sum_{i = m(T_n+ t)}^{m(T_{n + 1}) - 1} \alpha(i) L}  \\
\leq& \limsup_{n} \lim_{t \to T^-}\norm{\sum_{i = m(T_n+ t)}^{m(T_{n + 1}) - 1} \alpha(i) [L(Y_{i+1})-L]} +  L \limsup_{n} \alpha(m(T_{n + 1}) - 1)  \\
\leq& \limsup_{n} \lim_{t \to T^-}\norm{\sum_{i = m(T_n+ t)}^{m(T_{n + 1}) - 1} \alpha(i) [L(Y_{i+1})-L]} +  0 \explain{by \eqref{eq 1/n lr}}  \\
\leq& \limsup_{n} \sup_{0\leq t_1 \leq t_2 \leq T + \sup_j \alpha(j)} \norm{\sum_{i = m(T_n+ t_1)}^{m(T_n + t_2) - 1} \alpha(i) [L(Y_{i+1})-L]} \\
% =& \limsup_{n} \lim_{t \to T^-}\norm{\sum_{i = m(T_n)}^{m(T_{n + 1}) - 1} \alpha(i) [L(Y_{i+1})-L] - \sum_{i = m(T_n)}^{m(T_n + t) - 1} \alpha(i) [L(Y_{i+1})-L] }  \\
% \leq& \limsup_{n} \norm{\sum_{i = m(T_n)}^{m(T_{n + 1}) - 1} \alpha(i) [L(Y_{i+1})-L] }+ \limsup_{n} \lim_{t \to T^-}\norm{ \sum_{i = m(T_n)}^{m(T_n + t) - 1} \alpha(i) [L(Y_{i+1})-L] }  \\
% \leq& \limsup_{n} \norm{\sum_{i = m(T_{n})}^{m(T_{n+ 1}) - 1} \alpha(i) [L(Y_{i+1})-L] }+ \limsup_{n} \lim_{t \to T^-} \sup_{\tau}  \norm{ \sum_{i = m(T_n)}^{m(T_n + \tau) - 1} \alpha(i) [L(Y_{i+1})-L] } \\
% =&   \limsup_{n} \norm{\sum_{i = m(T_n)}^{m(T_{n + 1}) - 1} \alpha(i) [L(Y_{i+1})-L] }+ \limsup_{n} \sup_{\tau}  \norm{ \sum_{i = m(T_n)}^{m(T_n + \tau) - 1} \alpha(i) [L(Y_{i+1})-L] } \\
=& 0. \explain{by \eqref{eq: L property minus 0}}
\end{align}
% \sz{
%     It looks you have some typos. Shouldn't the last three inequalities be the following?
%     \begin{align}
% \leq& \limsup_{n} \norm{\sum_{i = m(T_n)}^{m(T_{n + 1}) - 1} \alpha(i) [L(Y_{i+1})-L] }+ \limsup_{n} \lim_{t \to T^-} \sup_{\tau}  \norm{ \sum_{i = m(T_n)}^{m(T_n + \tau) - 1} \alpha(i) [L(Y_{i+1})-L] }  \\
% \leq& \limsup_{n} \norm{\sum_{i = m(T_n)}^{m(T_{n + 1}) - 1} \alpha(i) [L(Y_{i+1})-L] }+ \limsup_{n} \sup_{\tau}  \norm{ \sum_{i = m(T_n)}^{m(T_n + \tau) - 1} \alpha(i) [L(Y_{i+1})-L] }  \\
% % \leq&  \lim_{t \to T^-}  \sup_{n} \norm{\sum_{i = m(T_n)}^{m(T_{n + 1}) - 1} \alpha(i) [L(Y_{i+1})-L] }+  \lim_{t \to T^-} \sup_{n}  \norm{ \sum_{i = m(T_n)}^{m(T_n + t) - 1} \alpha(i) [L(Y_{i+1})-L] }  \\
% \leq& 0. \explain{by \eqref{eq: L property minus 0}}
%     \end{align}
% }
This implies 
\begin{align}
\lim_{n} \lim_{t \to T^-}\norm{\sum_{i = m(T_n+ t)}^{m(T_{n + 1}) - 1} \alpha(i) L(Y_{i+1})} = 0. 
\end{align}
Following a similar proof, we have 
\begin{align}
\lim_{n} \lim_{t \to T^-} \norm{\sum_{i = m(T_n+ t)}^{m(T_{n + 1}) - 1} \alpha(i) H(0, Y_{i+1})  } = 0. 
\end{align}

\end{proof}

% \lsz{The below lemma is first used in lemma \ref{lemma: contradiction}}

\begin{lemma}\label{lemma: connect T and k+1}
$\lim_{k \to \infty}
\frac{\norm{\bar{x}(T_{n_k + 1}) } -  \lim_{t \to T^-}  \norm{\bar{x}(T_{n_k} + t)}}{ \norm{\bar{x}(T_{n_k})}}= 0$.
\end{lemma}
% Its proof is in Appendix \ref{appendix: connect T and k+1}.
% \subsubsection{Proof of Lemma~\ref{lemma: connect T and k+1}}\label{appendix: connect T and k+1}
\begin{proof}
% $\forall n$,
% \begin{align}
% &\norm{\bar{x}(T_{n + 1})     -  \bar{x}(T_n + T)}\\
% =&\norm{\bar{x}(T_n)  + \sum_{i = m(T_n)}^{m(T_{n+1}) - 1} \alpha(i) H(\bar{x}(t(i)), Y_{i+1})  -\bar{x}(T_n)  - \sum_{i = m(T_n)}^{m(T_n+T) - 1} \alpha(i) H(\bar{x}(t(i)), Y_{i+1}) } \\
% =& \norm{\sum_{i = m(T_n+T)}^{m(T_{n + 1}) - 1} \alpha(i) H(\bar{x}(t(i)), Y_{i+1})  } \\
% \leq&  \norm{\alpha(m(T_n+T)) H(\bar{x}(T_n+T), Y_{\alpha(m(T_n+T))+1})} \explain{by definition of $T_n$ \eqref{def:T}} \\
% \leq& \alpha(m(T_n+T)) \left[ \norm{H(\bar{x}(T_n+T), Y_{\alpha(m(T_n+T))+1}) - H(0, Y_{\alpha(m(T_n+T))+1})}   +  \norm{H(0, Y_{\alpha(m(T_n+T))+1})}\right] \\
% \leq& \alpha(m(T_n+T)) \left[ L\norm{\bar{x}(T_n+T)}   +  \norm{H(0, Y_{\alpha(m(T_n+T))+1})}\right] \explain{by Assumption \ref{assumption: lim H uniformly convergent}} \\
% \leq& \alpha(m(T_n+T)) \left[ L\norm{\bar{x}(T_n+T)}   + C_H\right] \explain{by Assumption \ref{assumption: bounded-h0}} \\
% \leq&  \alpha(m(T_n+T)) \left[ L \left[ (T_{n+1}-T_n]  \left( \norm{ \bar{x}(T_n)} + C_H \right)  e^{L(T_{n+1}-T_n]}  + \norm{\bar x(T_{n})} \right]    + C_H\right]
% \end{align}
We first analyze the numerator.
$\forall k$,
\begin{align}
& \abs{ \norm{\bar{x}(T_{n_k + 1})}     - \lim_{t \to T^-}  \norm{\bar{x}(T_{n_k} + t)} } \\
=&\lim_{t \to T^-} \abs{ \norm{\bar{x}(T_{n_k + 1})}     -  \norm{\bar{x}(T_{n_k} + t)} } \\
\leq &\lim_{t \to T^-} \norm{\bar{x}(T_{n_k + 1})     -  \bar{x}(T_{n_k} + t)}\\
=&\lim_{t \to T^-} \norm{\bar{x}(T_{n_k})  + \sum_{i = m(T_{n_k})}^{m(T_{n_k + 1}) - 1} \alpha(i) H(\bar{x}(t(i)), Y_{i+1})  -\bar{x}(T_{n_k})  - \sum_{i = m(T_{n_k})}^{m(T_{n_k}+ t) - 1} \alpha(i) H(\bar{x}(t(i)), Y_{i+1}) } \\
=&\lim_{t \to T^-} \norm{\sum_{i = m(T_{n_k}+ t)}^{m(T_{n_k + 1}) - 1} \alpha(i) H(\bar{x}(t(i)), Y_{i+1})  } \\
% \end{align}
% 
% \begin{align}
\leq& \lim_{t \to T^-} \norm{\sum_{i = m(T_{n_k}+ t)}^{m(T_{n_k + 1}) - 1} \alpha(i) \left[ H(\bar{x}(t(i)), Y_{i+1})  -  H(0, Y_{i+1})\right]} +  \norm{\sum_{i = m(T_{n_k}+ t)}^{m(T_{n_k + 1}) - 1} \alpha(i) H(0, Y_{i+1})  } \\
\leq& \lim_{t \to T^-} \sum_{i = m(T_{n_k}+ t)}^{m(T_{n_k + 1}) - 1} \alpha(i) L(Y_{i+1}) \norm{\bar{x}(t(i) } +  \norm{\sum_{i = m(T_{n_k}+ t)}^{m(T_{n_k + 1}) - 1} \alpha(i) H(0, Y_{i+1})  } \\
=& \norm{\bar{x}(t(m(T_{n_k+1}) - 1) }  \left[\lim_{t \to T^-} \sum_{i = m(T_{n_k}+ t)}^{m(T_{n_k + 1}) - 1} \alpha(i) L(Y_{i+1})\right]   +  \lim_{t \to T^-}\norm{\sum_{i = m(T_{n_k}+ t)}^{m(T_{n_k + 1}) - 1} \alpha(i) H(0, Y_{i+1})  } 
\explain{$\forall k, \lim_{t \to T^-} m(T_{n_k}+ t) = m(T_{n_k + 1}) - 1$}\\
\leq& \left(\left[\norm{ \bar{x}(T_{n_k})}  C_H  + C_H \right] e^{C_H} + \norm{ \bar{x}(T_{n_k})} \right) \left[\lim_{t \to T^-} \sum_{i = m(T_{n_k}+ t)}^{m(T_{n_k + 1}) - 1} \alpha(i) L(Y_{i+1}) \right]  \\
&+ \lim_{t \to T^-} \norm{\sum_{i = m(T_{n_k}+ t)}^{m(T_{n_k + 1}) - 1} \alpha(i) H(0, Y_{i+1})  }. \explain{by \eqref{eq: bar x bound}}
\end{align}
By \eqref{eq: limit r *}, we have 
\begin{align}
 \lim_{k\to\infty} \norm{\bar x(T_{n_k})} = \lim_{k \to \infty} r_{n_k}  = \infty. \label{eq: f k 0 sup geq 1}
\end{align}
% This implies
% \begin{align}
% \limsup_{k \to \infty} \norm{ \bar{x}(T_{n_{k} })} > 1. \label{eq: f k 0 sup geq 1}
% \end{align}
% \sz{
%     By adding the additional subsequence, we now have 
%     \begin{align}
%         \lim_{k\to\infty} \norm{\bar x(T_{n_k})} = \infty
%     \end{align}
%     directly. You can also simplify the below a bit.
% }
% \sz{Why does the above limit exist? I think you don't need this limit at all below.}
% \sz{Below you proved only a half, you also need to prove $> 0$. But the easiest thing is to add $\abs{\cdot}$.}
Thus,
\begin{align}
&\lim_{k \to \infty}\abs{ \frac{\norm{\bar{x}(T_{n_k + 1}) } -  \lim_{t \to T^-}  \norm{\bar{x}(T_{n_k} + t)}}{ \norm{\bar{x}(T_{n_k})}}}  \\
=&\lim_{k \to \infty}\frac{ \abs{ \norm{\bar{x}(T_{n_k + 1}) } -  \lim_{t \to T^-}  \norm{\bar{x}(T_{n_k} + t)}}} { \norm{\bar{x}(T_{n_k})}}  \\
=&  \lim_{k \to \infty}  \frac{ \left(\left[\norm{ \bar{x}(T_{n_k})}  C_H  + C_H \right] e^{C_H} + \norm{ \bar{x}(T_{n_k})} \right)  \left[\lim_{t \to T^-} \sum_{i = m(T_{n_k}+ t)}^{m(T_{n_k + 1}) - 1} \alpha(i) L(Y_{i+1}) \right]}{\norm{\bar{x}(T_{n_k})}}  \\
&+ \lim_{k \to \infty}  \frac{ \lim_{t \to T^-} \norm{\sum_{i = m(T_{n_k}+ t)}^{m(T_{n_k + 1}) - 1} \alpha(i) H(0, Y_{i+1})  } }{\norm{\bar{x}(T_{n_k})}}      \\
\leq&   \left(C_H  e^{C_H} + 1 \right)   \left[ \lim_{k \to \infty} \lim_{t \to T^-} \sum_{i = m(T_{n_k}+ t)}^{m(T_{n_k + 1}) - 1} \alpha(i) L(Y_{i+1})\right]  +  \lim_{k \to \infty}  \frac{ \lim_{t \to T^-} \norm{\sum_{i = m(T_{n_k}+ t)}^{m(T_{n_k + 1}) - 1} \alpha(i) H(0, Y_{i+1})  } }{\norm{\bar{x}(T_{n_k})}}
\explain{by \eqref{eq: f k 0 sup geq 1}} \\
\leq&  \left(C_H  e^{C_H} + 1 \right)  \cdot  0 + 0 \explain{by \eqref{eq: t T^- L 0} and \eqref{eq: t T^- H 0}}  \\
=& 0. 
\end{align}
\end{proof}

\section{Proofs for Completeness}

Proofs in this section have used ideas and sketches from \citet{kushner2003stochastic} but are self-contained and complete.
% Some of the proofs are quite technical.
% However,
% we expect a reader familiar with \citet{kushner2003stochastic} to have the belief in the correctness of the results in this section.

\subsection{Proof of Lemma~\ref{lemma: H h 0}}\label{appendix: H h 0}
\begin{proof}
\\
% \lsz{$S$ is a vector, no norm, need each dimension, assume $\alpha = \frac{1}{n}$}
\textbf{Case 1:}
Let Assumptions~\ref{assumption: stationary distribution},~\ref{assumption: alpha rate},~\ref{assumption: H Lipschitz}, and~\ref{assumption: lln} hold.\\
Fixed an arbitrary $\tau > 0$. For an arbitrary $x,t \in (-\infty, \infty)$, define
\begin{align}
\psi(i) &\doteq H(x,Y_{i+1}) - h(x), \\
S(n) &\doteq \sum_{i=0}^{n-1}\psi(i), \\
\Psi(t) &\doteq \sum_{i=0}^{m(t) - 1} \alpha(i) \psi(i).
\end{align}
Here, we use \eqref{eq: m 0} so that $\forall t < 0, m(t) = 0$ and the convention that $\sum_{k=i}^j \alpha(k) = 0$ when $j < i$.
% First, we show 
% \begin{align}
% \lim_{n \to \infty } S(n) =   0.
% \end{align}
% \lsz{detail more here}
% \lsz{Janey has proved this. Make this as an assumption}
% Because $\qty{Y_n}$ has a stationary distribution ${d_\fY}: \fY \to \R$, by strong law of large numberss,
% \begin{align}\label{eq: law of large numbers}
%  \Pr\left(\lim_{t\to \infty} \frac{1}{t} \sum^{t-1}_{\tau = 0}f(Y_\tau) = \E_{Y\sim{d_\fY}}[f(Y)] \right) =  1.  
% \end{align}
% \lsz{check this for infinite Markov chain}
Fix a sample path $\qty{x_0,\qty{Y_i}_{i=1}^\infty}$ where Assumptions~\ref{assumption: stationary distribution},~\ref{assumption: alpha rate},~\ref{assumption: H Lipschitz}, \&~\ref{assumption: lln} hold.
Assumption~\ref{assumption: lln} implies that
\begin{align}
    \lim_{n\to\infty} \alpha(n)S(n+1) = 0.
\end{align}
% . For any dimension $j$,
% \begin{align}
% &\lim_{n \to \infty } \frac{S(n)_j}{n}\\
% % =& \lim_{n \to \infty } \sum_{i=0}^{n-1} H(x,Y_{i+1}) - h(x) \\
% =& \lim_{n \to \infty } \frac{\left( \sum_{i=0}^{n-1} H(x,Y_{i+1})_j - h(x)_j \right)}{n} \\
% =& \lim_{n \to \infty } \frac{\sum_{i=0}^{n-1} H(x,Y_{i+1})_j }{n} - \lim_{n \to \infty } \frac{\sum_{i=0}^{n-1}h(x)_j}{n} \\
% =& \lim_{n \to \infty } \frac{\sum_{i=0}^{n-1} H(x,Y_{i+1})_j }{n} - h(x)_j \\
% =&  \E_{y\sim {d_\fY}}[H(x,y)]_j  - h(x)_j \explain{by \eqref{eq stronger lln}} \\
% =&  h(x)_j -  h(x)_j \explain{by \eqref{def: h c}}\\
% =& 0 .  \label{eq: s n / n}
% \end{align}
% Thus, for any dimension $j$,
% \begin{align}
% &\lim_{n \to \infty } \alpha(n) S(n+1)_j\\  
% =&\lim_{n \to \infty } \alpha(n) n  \frac{n+1}{n}\frac{S(n+1)_j}{n+1} \\
% =& \lim_{n \to \infty } \alpha(n) n \lim_{n \to \infty }  \frac{n+1}{n} \lim_{n \to \infty }  \frac{S(n+1)_j}{n+1} \\
% =&\lim_{n \to \infty } \alpha(n) n \cdot  1  \cdot  0 \explain{by \eqref{eq: s n / n}} \\
% % =&\lim_{n \to \infty } \frac{\alpha(n) - \alpha(n+1) + \alpha(n+1)}{\alpha(n)} \alpha(n)   n \cdot  1  \cdot  0  \\
% % =&\lim_{n \to \infty } \left[\fO\left(\frac{1}{n}\right) + \frac{\alpha(n+1)}{\alpha(n)} \right]\alpha(n)   n \cdot  1  \cdot  0  \\
% % \leq&\lim_{n \to \infty } \left[C \alpha(n)   + \alpha(n+1)n \right] \cdot  1  \cdot  0  \\
% =&0.\explain{By Assumption \ref{assumption: lln}}
% \end{align}
Use subscript $j$ to denote the $j$th dimension of a vector,
we then have
\begin{align}
\label{eq: alpha s}
    \limsup_{n\to\infty} \sup_{-\tau \leq t \leq \tau} \abs{\alpha(m(t(n) + t)) S(m(t(n) + t) + 1)_j} = 0.
\end{align}
Moreover, for
$\forall t \in [-\tau, \tau]$, we have
\begin{align}
\Psi(t) 
=& \sum_{i=0}^{m(t) - 1} \alpha(i) \psi(i) \\
=& \sum_{i=0}^{m(t) - 1} \alpha(i) \left[ \sum_{j = 0}^i \psi(j) -  \sum_{j = 0}^{i-1} \psi(j) \right] \\
=& \sum_{i=0}^{m(t) - 1} \alpha(i)  \sum_{j = 0}^i \psi(j) -  \sum_{i=0}^{m(t) - 1} \alpha(i) \sum_{j = 0}^{i-1} \psi(j)\\
= &  \sum_{i=0}^{m(t) - 1} \alpha(i)  \sum_{j = 0}^i \psi(j) -    \sum_{i=0}^{m(t) - 2  } \alpha(i+1) \sum_{j = 0}^{i} \psi(j) \\
=&\alpha(m(t) - 1) \sum_{i=0}^{m(t)-1}\psi(i) + \sum_{i=0}^{m(t) -  2} [\alpha(i) - \alpha(i+1)] \sum_{j=0}^{i}\psi(j) \\
=&\alpha(m(t) - 1) \sum_{i=0}^{m(t)-1}\psi(i) + \sum_{i=0}^{m(t) -  2}S(i+1)[\alpha(i) - \alpha(i+1)]\\
=& \alpha(m(t) - 1) S(m(t)) + \sum_{i=0}^{m(t) -  2}S(i+1)\frac{\alpha(i) - \alpha(i+1)}{\alpha(i)} \alpha(i) \label{eq: psi}. 
\end{align}
% \begin{align}
%  &\sum_{i=0}^{m(\tau) - 1} \alpha(i) \sum_{j = 0}^{i-1} \psi(j)\\ 
%  = &    \sum_{k=-1}^{m(\tau) - 2  } \alpha(k+1) \sum_{j = 0}^{k} \psi(j)\\ 
%  = &    \sum_{i=0}^{m(\tau) - 2  } \alpha(i+1) \sum_{j = 0}^{i} \psi(j)
% \end{align}
Thus, for any dimension $j$,
\begin{align}
&\limsup_{n \to \infty} \sup_{-\tau \leq t_1 \leq t_2 \leq \tau} \abs{ \sum_{i=m(t(n) + t_1)}^{m(t(n) + t_2) - 1} \alpha(i)(H(x,Y_{i+1})_j-h(x)_j) }  \\
=&\limsup_{n \to \infty} \sup_{-\tau \leq t_1 \leq t_2 \leq \tau}  \abs{\Psi(t(n) + t_2)_j -  \Psi(t(n) + t_1)_j } \\
\leq& \limsup_{n \to \infty} \sup_{-\tau \leq t_1 \leq t_2 \leq \tau} \abs{ \alpha(m(t(n) + t_2) -1) S(m(t(n) + t_2))_j} + \abs{\alpha(m(t(n) + t_1) -1) S(m(t(n) + t_1))_j} 
\\
&+  \abs{ \sum_{i=m(t(n) + t_1) - 1}^{m(t(n) + t_2) - 2}S(i+1)_j\frac{\alpha(i) - \alpha(i+1)}{\alpha(i)} \alpha(i)} \explain{by \eqref{eq: psi}} \\
=& \limsup_{n \to \infty} \sup_{-\tau \leq t_1 \leq t_2 \leq \tau}  \abs{  \sum_{i=m(t(n) + t_1) - 1}^{m(t(n) + t_2) -  2}S(i+1)_j\frac{\alpha(i) - \alpha(i+1)}{\alpha(i)} \alpha(i)} \explain{by \eqref{eq: alpha s}} \\
\leq& \limsup_{n \to \infty} \sup_{-\tau \leq t_1 \leq t_2 \leq \tau}   \sum_{i=m(t(n) + t_1) - 1}^{m(t(n) + t_2) -  2} \abs{ S(i+1)_j\frac{\alpha(i) - \alpha(i+1)}{\alpha(i)} \alpha(i)} \\
\leq& \limsup_{n \to \infty} \sup_{-\tau \leq t_1 \leq t_2 \leq \tau}  \sum_{i=m(t(n) + t_1) - 1}^{m(t(n) + t_2) -  2} \abs{ \alpha(i) S(i+1)_j  } \abs{\frac{\alpha(i) - \alpha(i+1)}{\alpha(i)} }  \\
% \end{align}
% \begin{align}
\leq&\limsup_{n \to \infty} \sup_{-\tau \leq t_1 \leq t_2 \leq \tau}  \left(\sup_{m(t(n) + t_1) -1 \leq i \leq m(t(n) + t_2)-2 } \abs{\alpha(i)S(i+1)_j} \right) \sum_{i=m(t(n) + t_1) - 1}^{m(t(n) + t_2) -  2}\abs{\frac{\alpha(i) - \alpha(i+1)}{\alpha(i)}} \\
\leq& \limsup_{n \to \infty} \sup_{-\tau \leq t_1 \leq t_2 \leq \tau} \left(\sup_{m(t(n) + t_1) -1 \leq i \leq m(t(n) + t_2)-2 } \abs{\alpha(i)S(i+1)_j}\right) C_{\alpha}\sum_{i=m(t(n) + t_1) - 1}^{m(t(n) + t_2) -  2} \alpha(i) \explain{by Assumption \ref{assumption: alpha rate}, $C_{\alpha}$ is a constant from the big $\fO$ notation--$\frac{\alpha(n)- \alpha(n+1)}{\alpha(n)} = \fO\left(\alpha(n)\right)$
% s.t. $\exists i_i$, $\forall i \geq i_0$, $\frac{\alpha(i) - \alpha(i+1)}{\alpha(i)} \leq C_\alpha \alpha(i)$
}\\
=& \limsup_{n \to \infty} \left[ \sup_{-\tau \leq t_1 \leq t_2 \leq \tau} \left(\sup_{m(t(n) + t_1) -1 \leq i \leq m(t(n) + t_2)-2 } \abs{\alpha(i)S(i+1)_j}\right) \right. \\
&\left. \cdot C_{\alpha}\left(\sum_{i=m(t(n) + t_1)}^{m(t(n) + t_2) -  1} \alpha(i)  + \alpha(m(t(n) + t_1) -1) \right) \right] \\
=& \limsup_{n \to \infty} \sup_{-\tau \leq t_1 \leq t_2 \leq \tau} \left(\sup_{m(t(n) + t_1) -1 \leq i \leq m(t(n) + t_2)-2 } \abs{\alpha(i)S(i+1)_j}\right) C_{\alpha}\left(t_2 - t_1 + \alpha(m(t(n) + t_1) -1) \right) \explain{by \eqref{eq: lim alpha t_1 t_2}}\\
% \end{align}
% \begin{align}
\leq& \limsup_{n \to \infty}\left(\sup_{m(t(n) - \tau ) -1 \leq i } \abs{\alpha(i)S(i+1)_j}\right) C_{\alpha}\left(t_2 - t_1 + \alpha(m(t(n) + t_1) -1) \right) \\
% \leq&  2C_{\alpha}\tau \limsup_{n \to \infty} \left(\sup_{m(t(n) + t_1) -1 \leq i \leq m(T_{n+1})-2 } \abs{\alpha(i)S(i+1)_j}\right) \\
\leq& 2C_{\alpha}\tau \limsup_{n \to \infty} \left(\sup_{m(t(n) - \tau) -1 \leq i } \abs{\alpha(i)S(i+1)_j}\right) \\
\leq& 2C_{\alpha}\tau \limsup_{n \to \infty} \left(\sup_{n \leq i } \abs{\alpha(i)S(i+1)_j}\right) \\
=& 0. \explain{by \eqref{eq: alpha s}}
\end{align}
Thus, $\forall \tau > 0$, $\forall x$,
\begin{align}
\limsup_n \sup_{-\tau \leq t_1 \leq t_2 \leq \tau} \norm{ \sum_{i = m(t(n) + t_1)}^{m(t(n) + t_2) - 1} \alpha(i) \left[H(x, Y_{i+1}) - h(x) \right] }    & = 0 \quad a.s. 
\end{align}
The proofs for \eqref{eq: L b property minus 0} and \eqref{eq: L property minus 0}
follow the same logic and thus are omitted.\\
\textbf{Case 2:}
Let Assumptions~\ref{assumption: stationary distribution},~\ref{assumption: alpha rate},~\ref{assumption: H Lipschitz}, and~\ref{assumption possion} hold.\\
By Assumption~\ref{assumption: H Lipschitz} and the equivalence between norms, we have
\begin{align}
\norm{H(x, y)}_2 \leq C\left(\norm{H(0, y)}_2 + L(y) \norm{x}_2\right)
\end{align}
for some constant $C$ independent of $x, y$.
So for any $x$,
\begin{align}
\sup_y \frac{\norm{H(x, y)}_2^2}{v(y)} \leq \sup_y \frac{2C^2\norm{H(0, y)}_2^2 + 2C^2 L(y)^2 \norm{x}_2^2}{v(y)} < \infty.
\end{align}
In other words, for any $x$,
\begin{align}
y \mapsto H(x, y) \in \fL^2_{v, \infty}.  
\end{align}
Similarly, we have for any $x$,
\begin{align}
y \mapsto L_b(y) \in \fL^2_{v, \infty}.  
\end{align}
Let $g$ denote any of the following functions:
\begin{align}
y \mapsto& H(x, y) \quad (\forall x), \\
y \mapsto& L_b(y) \quad (\forall x), \\
y \mapsto& L(y).
\end{align}
We now always have $g \in \fL^2_{v, \infty}$.
Proposition~6 of \citet{borkar2021ode} then confirms that 
\begin{align}
\label{eq possion average}
\sum_{i=0}^\infty \alpha(i)(g(Y_{i+1}) - \E_{y\sim{d_\fY}}\left[g(y)\right])
\end{align}
converges almost surely to a square-integrable random variable.
Lemma~\ref{lemma: H h 0} then follows immediately from the Cauchy convergence test.

\end{proof}

\subsection{Proof of Lemma~\ref{lemma: double limit 2}}\label{appendix: double limit 2}

To prove Lemma \ref{lemma: double limit 2}, we first decompose it into three terms. Then, we prove the convergence of each term in Lemmas \ref{lemma: three term 1}, \ref{lemma: three term 2}, \& \ref{lemma: three term 3}. Finally, we restate Lemma \ref{lemma: double limit 2} and connect everything.

For each $t$, let $\qty{\Delta_l}_{l = 1}^\infty $ be a strictly decreasing sequence of real numbers such that $\lim_{l\to\infty} \Delta_l = 0$ and $\forall l, \frac{t}{\Delta_l}  - 1 \in \N$,
e.g.,
$\Delta_l \doteq \frac{t}{l+1}$.
Because $\forall l$,
\begin{align}
&\sum_{i = m(T_{n_k})}^{m(T_{n_k} +t) -1 } \alpha(i)H_{r_{n_j}}(\hat{x}(t(i)),Y_{i+1}) 
= \sum_{a=0}^{\frac{t}{\Delta_l}  - 1}  \sum_{i = m(T_{n_k}+a\Delta_l)}^{m(T_{n_k} + a\Delta_l + \Delta_l)  - 1 } \alpha(i)H_{r_{n_j}}(\hat{x}(t(i)),Y_{i+1}), 
\end{align}
we have
\begin{align}
&\lim_{k \to \infty} \norm{\sum_{i = m(T_{n_k})}^{m(T_{n_k} +t) -1 } \alpha(i)H_{r_{n_j}}(\hat{x}(t(i)),Y_{i+1}) - \int_0^t h_{r_{n_j}} (\hat{x}^{\lim}(s))ds}  \label{eq: 3 limit 0}\\
=&\lim_{l \to \infty}  \lim_{k \to \infty} \norm{ \sum_{a=0}^{\frac{t}{\Delta_l}  - 1}  \sum_{i = m(T_{n_k}+a\Delta_l)}^{m(T_{n_k} + a\Delta_l + \Delta_l)  - 1 } \alpha(i)H_{r_{n_j}}(\hat{x}(t(i)),Y_{i+1})  - \int_0^t h_{r_{n_j}} (\hat{x}^{\lim}(s))ds} \\
% \leq& \lim_{l \to \infty}  \lim_{k \to \infty}  \norm{\sum_{a=0}^{\frac{t}{\Delta_l}  - 1}  \sum_{i = m(T_{n_k}+a\Delta_l)}^{m(T_{n_k} + a\Delta_l + \Delta_l)  - 1 } \alpha(i)h_{r_{n_j}}(\hat{x}^{\lim}(a\Delta_l)) - \int_0^t h_{r_{n_j}} (\hat{x}^{\lim}(s))ds}  \\
% &+   \norm{\sum_{a=0}^{\frac{t}{\Delta_l}  - 1}  \sum_{i = m(T_{n_k}+a\Delta_l)}^{m(T_{n_k} + a\Delta_l + \Delta_l)  - 1 } \alpha(i) \left(H_{r_{n_j}}(\hat{x}(t(i)),Y_{i+1})  - H_{r_{n_j}}(\hat{x}^{\lim}(a\Delta_l),Y_{i+1}) \right) }\\
% &+   \norm{\sum_{a=0}^{\frac{t}{\Delta_l}  - 1}  \sum_{i = m(T_{n_k}+a\Delta_l)}^{m(T_{n_k} + a\Delta_l + \Delta_l)  - 1 } \alpha(i) \left(H_{r_{n_j}}(\hat{x}^{\lim}(a\Delta_l),Y_{i+1}) - h_{r_{n_j}}(\hat{x}^{\lim}(a\Delta_l)) \right)} \\
\leq& \lim_{l \to \infty}  \lim_{k \to \infty}  \norm{\sum_{a=0}^{\frac{t}{\Delta_l}  - 1}  \sum_{i = m(T_{n_k}+a\Delta_l)}^{m(T_{n_k} + a\Delta_l + \Delta_l)  - 1 } \alpha(i)h_{r_{n_j}}(\hat{x}^{\lim}(a\Delta_l)) - \int_0^t h_{r_{n_j}} (\hat{x}^{\lim}(s))ds} \label{eq: 3 limit 1}  \\
&+  \lim_{l \to \infty}  \lim_{k \to \infty}  \norm{\sum_{a=0}^{\frac{t}{\Delta_l}  - 1}  \sum_{i = m(T_{n_k}+a\Delta_l)}^{m(T_{n_k} + a\Delta_l + \Delta_l)  - 1 } \alpha(i) \left(H_{r_{n_j}}(\hat{x}(t(i)),Y_{i+1})  - H_{r_{n_j}}(\hat{x}^{\lim}(a\Delta_l),Y_{i+1}) \right) }  \\ \label{eq: 3 limit 2}\\
&+  \lim_{l \to \infty}  \lim_{k \to \infty}  \norm{\sum_{a=0}^{\frac{t}{\Delta_l}  - 1}  \sum_{i = m(T_{n_k}+a\Delta_l)}^{m(T_{n_k} + a\Delta_l + \Delta_l)  - 1 } \alpha(i) \left(H_{r_{n_j}}(\hat{x}^{\lim}(a\Delta_l),Y_{i+1}) - h_{r_{n_j}}(\hat{x}^{\lim}(a\Delta_l)) \right)}. \\ 
\label{eq: 3 limit 3}
\end{align}
Now, we show the limit of \eqref{eq: 3 limit 1}, \eqref{eq: 3 limit 2}, and \eqref{eq: 3 limit 3} are $0$ in Lemmas \ref{lemma: three term 1}, \ref{lemma: three term 2}, and \ref{lemma: three term 3} with proofs in Appendix \ref{appendix: three term 1}, \ref{appendix: three term 2}, and \ref{appendix: three term 3}.
\begin{lemma}\label{lemma: three term 1}
$\forall j, \forall t\in [0,T),$
\begin{align}
\lim_{l \to \infty}  \lim_{k \to \infty} \norm{ \sum_{a=0}^{\frac{t}{\Delta_l}  - 1}  \sum_{i = m(T_{n_k}+a\Delta_l)}^{m(T_{n_k} + a\Delta_l + \Delta_l)  - 1 } \alpha(i)h_{r_{n_j}}(\hat{x}^{\lim}(a\Delta_l)) - \int_0^t  h_{r_{n_j}}(\hat{x}^{\lim}(s)) } = 0.
\end{align}
\end{lemma}

% Now, we prove the second limit.
\begin{lemma}\label{lemma: three term 2}
$\forall j, \forall t\in [0,T),$
\begin{align}
\lim_{l \to \infty } \lim_{k \to \infty}   \norm{
 \sum_{a=0}^{\frac{t}{\Delta_l}  - 1}  \sum_{i = m(T_{n_k}+a\Delta_l)}^{m(T_{n_k} + a\Delta_l + \Delta_l)  - 1 } \alpha(i) \left(H_{r_{n_j}}(\hat{x}(t(i)),Y_{i+1})  - H_{r_{n_j}}(\hat{x}^{\lim}(a\Delta_l),Y_{i+1}) \right)}    = 0.
\end{align}  
\end{lemma}
% To show the limit of \eqref{eq: 3 limit 3}, we add the following assumption on the learning rate.
% \lsz{Change this assumption to a Lemma using Assumption \ref{assumption: alpha property}?}
% Now, we are ready to show the following lemma for \eqref{eq: 3 limit 3} with proof in Appendix \ref{appendix: three term 3}.
\begin{lemma}\label{lemma: three term 3}
$\forall j, \forall t\in [0,T),$
\begin{align}
\lim_{l \to \infty}  \lim_{k \to \infty}  \norm{ \sum_{a=0}^{\frac{t}{\Delta_l}  - 1}  \sum_{i = m(T_{n_k}+a\Delta_l)}^{m(T_{n_k} + a\Delta_l + \Delta_l)  - 1 } \alpha(i) \left(H_{r_{n_j}}(\hat{x}^{\lim}(a\Delta_l),Y_{i+1}) - h_{r_{n_j}}(\hat{x}^{\lim}(a\Delta_l)) \right) } = 0.    
\end{align}
\end{lemma}
Plugging Lemmas \ref{lemma: three term 1}, \ref{lemma: three term 2}, and \ref{lemma: three term 3} back to \eqref{eq: 3 limit 0} completes the proof of Lemma~\ref{lemma: double limit 2}.

\subsection{Proof of Lemma~\ref{lemma: three term 1}}\label{appendix: three term 1}
\begin{proof}
$\forall j, \forall t\in [0,T),$
\begin{align}
&\lim_{l \to \infty}  \lim_{k \to \infty}  \sum_{a=0}^{\frac{t}{\Delta_l}  - 1}  \sum_{i = m(T_{n_k}+a\Delta_l)}^{m(T_{n_k} + a\Delta_l + \Delta_l)  - 1 } \alpha(i)h_{r_{n_j}}(\hat{x}^{\lim}(a\Delta_l)) \\
=& \lim_{l \to \infty} \sum_{a=0}^{\frac{t}{\Delta_l}  - 1}  h_{r_{n_j}}(\hat{x}^{\lim}(a\Delta_l)) \lim_{k \to \infty} \sum_{i = m(T_{n_k}+a\Delta_l)}^{m(T_{n_k} + a\Delta_l + \Delta_l)  - 1 } \alpha(i) \\
=& \lim_{l \to \infty} \sum_{a=0}^{\frac{t}{\Delta_l}  - 1}  h_{r_{n_j}}(\hat{x}^{\lim}(a\Delta_l)) \Delta_l \explain{by \eqref{eq: lim alpha t_1 t_2}} \\
=& \int_0^t  h_{r_{n_j}}(\hat{x}^{\lim}(s))  ds. \explain{by definition of integral}
\end{align}
% \sz{For the last equality, say sth like by the definition of integral.}

Thus, $\forall j, \forall t\in [0,T),$
\begin{align}
&\lim_{l \to \infty}  \lim_{k \to \infty} \norm{ \sum_{a=0}^{\frac{t}{\Delta_l}  - 1}  \sum_{i = m(T_{n_k}+a\Delta_l)}^{m(T_{n_k} + a\Delta_l + \Delta_l)  - 1 } \alpha(i)h_{r_{n_j}}(\hat{x}^{\lim}(a\Delta_l)) - \int_0^t  h_{r_{n_j}}(\hat{x}^{\lim}(s)) }\\
=&\norm{\lim_{l \to \infty}  \lim_{k \to \infty}  \sum_{a=0}^{\frac{t}{\Delta_l}  - 1}  \sum_{i = m(T_{n_k}+a\Delta_l)}^{m(T_{n_k} + a\Delta_l + \Delta_l)  - 1 } \alpha(i)h_{r_{n_j}}(\hat{x}^{\lim}(a\Delta_l)) - \int_0^t  h_{r_{n_j}}(\hat{x}^{\lim}(s)) }\\
=&\norm{\int_0^t  h_{r_{n_j}}(\hat{x}^{\lim}(s)) - \int_0^t  h_{r_{n_j}}(\hat{x}^{\lim}(s)) } \\
=& 0. 
\end{align}

\end{proof}

\subsection{Proof of Lemma~\ref{lemma: three term 2}}\label{appendix: three term 2}
\begin{proof}
$\forall j, \forall t\in [0,T), \forall l$
% \lsz{$t(i)$ range in the third inequality}
\begin{align} 
&\lim_{k \to \infty } \norm{
\sum_{a=0}^{\frac{t}{\Delta_l}  - 1}  \sum_{i = m(T_{n_k}+a\Delta_l)}^{m(T_{n_k} + a\Delta_l + \Delta_l)  - 1 } \alpha(i) \left(H_{r_{n_j}}(\hat{x}(t(i)),Y_{i+1})  - H_{r_{n_j}}(\hat{x}^{\lim}(a\Delta_l),Y_{i+1}) \right)}    \\
\leq& \lim_{k \to \infty }
\sum_{a=0}^{\frac{t}{\Delta_l}  - 1}  \sum_{i = m(T_{n_k}+a\Delta_l)}^{m(T_{n_k} + a\Delta_l + \Delta_l)  - 1 } \alpha(i)  \norm{H_{r_{n_j}}(\hat{x}(t(i)),Y_{i+1})  - H_{r_{n_j}}(\hat{x}^{\lim}(a\Delta_l),Y_{i+1}) }  \\
\leq& \lim_{k \to \infty }\sum_{a=0}^{\frac{t}{\Delta_l}  - 1}  \sum_{i = m(T_{n_k}+a\Delta_l)}^{m(T_{n_k} + a\Delta_l + \Delta_l)  - 1 } \alpha(i) L(Y_{i+1}) \norm{ \hat{x}(t(i)) -  \hat{x}^{\lim}(a\Delta_l) }  \explain{by Assumption \ref{assumption: H Lipschitz}}   \\
\leq& \lim_{k \to \infty } \left[ \sup_{0 \leq a \leq \frac{t}{\Delta_l} - 1} \sup_{m(T_{n_k}+a\Delta_l) \leq i \leq m(T_{n_k} + a\Delta_l + \Delta_l)  - 1  }  \norm{ \hat{x}(t(i)) -  \hat{x}^{\lim}(a\Delta_l) }  \right]
\sum_{a=0}^{\frac{t}{\Delta_l}  - 1}  \sum_{i = m(T_{n_k}+a\Delta_l)}^{m(T_{n_k} + a\Delta_l + \Delta_l)  - 1 } \alpha(i) L(Y_{i+1}) \\
=& \lim_{k \to \infty } \left[ \sup_{0 \leq a \leq \frac{t}{\Delta_l} - 1} \sup_{m(T_{n_k}+a\Delta_l) \leq i \leq m(T_{n_k} + a\Delta_l + \Delta_l)  - 1  }  \norm{ \hat{x}(t(i)) -  \hat{x}^{\lim}(a\Delta_l) }  \right]
\sum_{i = m(T_{n_k})}^{m(T_{n_k} + t)  - 1 } \alpha(i) L(Y_{i+1}). \label{eq: term 2 1} 
\end{align}
% \lsz{bad approach, introduced sup too early. Make k inner first.}
We show the limit of the following term.
\begin{align}
&\lim_{k \to \infty } \left[ \sup_{0 \leq a \leq \frac{t}{\Delta_l} - 1} \sup_{m(T_{n_k}+a\Delta_l) \leq i \leq m(T_{n_k} + a\Delta_l + \Delta_l)  - 1  }  \norm{ \hat{x}(t(i)) -  \hat{x}^{\lim}(a\Delta_l) }  \right]   \\
=&  \lim_{k \to \infty } \left[ \sup_{0 \leq a \leq \frac{t}{\Delta_l} - 1} \sup_{ t(m(T_{n_k}+a\Delta_l)) \leq t(i) \leq  t(m(T_{n_k} + a\Delta_l + \Delta_l)  - 1)  }  \norm{ \hat{x}(t(i)) -  \hat{x}^{\lim}(a\Delta_l) }  \right]   \\
\leq&  \lim_{k \to \infty } \left[ \sup_{0 \leq a \leq \frac{t}{\Delta_l} - 1} \sup_{ t(m(T_{n_k}+a\Delta_l)) \leq \tau \leq  t(m(T_{n_k} + a\Delta_l + \Delta_l)  - 1)  }  \norm{ \hat{x}(\tau) -  \hat{x}^{\lim}(a\Delta_l) }  \right]   \\
=&  \lim_{k \to \infty } \left[ \sup_{0 \leq a \leq \frac{t}{\Delta_l} - 1} \sup_{ T_{n_k}+a\Delta_l  \leq \tau \leq  t(m(T_{n_k} + a\Delta_l + \Delta_l)  - 1)  }  \norm{ \hat{x}(\tau) -  \hat{x}^{\lim}(a\Delta_l) }  \right]   \explain{$\hat{x}$ is a constant function on interval $[t(m(T_{n_k} + a\Delta_l )), T_{n_k} + a\Delta_l ]$ by \eqref{def: bar x} and \eqref{def: hat x}}\\
\leq&  \lim_{k \to \infty } \left[ \sup_{0 \leq a \leq \frac{t}{\Delta_l} - 1} \sup_{ T_{n_k}+a\Delta_l  \leq \tau <  T_{n_k} + a\Delta_l + \Delta_l }  \norm{ \hat{x}(\tau) -  \hat{x}^{\lim}(a\Delta_l) }  \right]   \explain{by \eqref{eq: t-m-inequality-2}}\\
=&  \lim_{k \to \infty } \left[ \sup_{0 \leq a \leq \frac{t}{\Delta_l} - 1} \sup_{ 0 \leq \tau <   \Delta_l }  \norm{ \hat{x}(T_{n_k} + a\Delta_l +\tau) -  \hat{x}^{\lim}(a\Delta_l) }  \right].
% \\
% =&   \sup_{0 \leq a \leq \frac{t}{\Delta_l} - 1} \lim_{k \to \infty }  \left[ \sup_{ 0 \leq \tau <   \Delta_l }  \norm{ \hat{x}(T_{n_k} + a\Delta_l +\tau) -  \hat{x}^{\lim}(a\Delta_l) }  \right] 
\label{eq: term 2 2}
\end{align}
% \sz{Explain why you can interchange $\lim$ and $\sup$ in the last equality?}
% \lsz{check more}
By \eqref{eq: hat x x r lim}, $\forall \delta > 0$, $\exists k_0$ such that $\forall k \geq k_0$, $\forall t \in [0,T)$,
\begin{align}
\norm{\hat{x}(T_{n_k} + t) - \hat{x}^{\lim}(t) } \leq \delta.   
\end{align}
% By \eqref{eq 1/n lr} and Assumption $\ref{assumption: alpha property}$, $\forall \delta > 0$, $\exists k_1$ such that $\forall k \geq k_1$, 
% \begin{align}
% \alpha(m(T_{n_{k_1}}+a\Delta_l ))  \leq \delta.
% \end{align}
$\forall t\in [0,T), \forall l,  \forall a, $ 
$\forall k \geq k_0$,
\begin{align}
& \abs{ \sup_{0 \leq a \leq \frac{t}{\Delta_l} - 1}  \sup_{ 0 \leq \tau <   \Delta_l }  \norm{ \hat{x}(T_{n_k} + a\Delta_l +\tau) -  \hat{x}^{\lim}(a\Delta_l) }  - \sup_{0 \leq a \leq \frac{t}{\Delta_l} - 1}  \sup_{0 \leq \tau   < \Delta_l } \norm{ \hat{x}^{\lim}(a\Delta_l+\tau) -  \hat{x}^{\lim}(a\Delta_l) } } \\  
\leq& \sup_{0 \leq a \leq \frac{t}{\Delta_l} - 1}  \sup_{ 0 \leq \tau <   \Delta_l } \abs{  \norm{ \hat{x}(T_{n_k} + a\Delta_l +\tau) -  \hat{x}^{\lim}(a\Delta_l) }  - \norm{ \hat{x}^{\lim}(a\Delta_l+\tau) -  \hat{x}^{\lim}(a\Delta_l) }  }  
\explain{by $\abs{\sup_x f(x) - \sup_x g(x)} \leq \sup_x \abs{f(x) - g(x)} $}
\\
\leq& \sup_{0 \leq a \leq \frac{t}{\Delta_l} - 1}  \sup_{ 0 \leq \tau <   \Delta_l }   \norm{ \hat{x}(T_{n_k} + a\Delta_l +\tau) -  \hat{x}^{\lim}(a\Delta_l)   -  \hat{x}^{\lim}(a\Delta_l+\tau) +  \hat{x}^{\lim}(a\Delta_l) }    \\  
\leq& \sup_{0 \leq a \leq \frac{t}{\Delta_l} - 1}  \sup_{ 0 \leq \tau <   \Delta_l }   \norm{ \hat{x}(T_{n_k} + a\Delta_l +\tau)    -  \hat{x}^{\lim}(a\Delta_l+\tau)  }   \\  
\leq& \sup_{0 \leq a \leq \frac{t}{\Delta_l} - 1}  \sup_{ 0 \leq \tau <   \Delta_l }    \delta    \\  
\leq&  \delta.  
\end{align}
% \lsz{add more explanation}
% $\forall t\in [0,T), \forall l,  \forall a, $ 
% $\forall k \geq k_0$,
% \begin{align}
% & \sup_{0 \leq \tau   < \Delta_l } \norm{ \hat{x}^{\lim}(a\Delta_l+\tau) -  \hat{x}^{\lim}(a\Delta_l) }- \sup_{ 0 \leq \tau <   \Delta_l }  \norm{ \hat{x}(T_{n_k} + a\Delta_l +\tau) -  \hat{x}^{\lim}(a\Delta_l) } \\  
% \leq&  \sup_{ 0 \leq \tau <   \Delta_l } \left[   \norm{ \hat{x}^{\lim}(a\Delta_l+\tau) -  \hat{x}^{\lim}(a\Delta_l) } -  \norm{ \hat{x}(T_{n_k} + a\Delta_l +\tau) -  \hat{x}^{\lim}(a\Delta_l) }   \right]  \\
% \leq&  \sup_{ 0 \leq \tau <   \Delta_l }   \norm{ \hat{x}(T_{n_k} + a\Delta_l +\tau) -  \hat{x}^{\lim}(a\Delta_l)   -  \hat{x}^{\lim}(a\Delta_l+\tau) +  \hat{x}^{\lim}(a\Delta_l) }    \\  
% \leq&  \sup_{ 0 \leq \tau <   \Delta_l }   \norm{ \hat{x}(T_{n_k} + a\Delta_l +\tau)    -  \hat{x}^{\lim}(a\Delta_l+\tau)  }   \\  
% \leq&  \delta.
% \end{align}
% \sz{You can make $\abs{\sup x - \sup y} \leq \sup \abs{x - y}$ a lemma to make the presentation clearer.}
% \sz{Use $\abs{\cdot}$ instead of $\norm{\cdot}$.}
% Thus, $\forall t\in [0,T), \forall l,  \forall a, \exists k_0$ such that $\forall k \geq k_0$,
% \begin{align}
%  \norm{\sup_{ 0 \leq \tau <   \Delta_l }  \norm{ \hat{x}(T_{n_k} + a\Delta_l +\tau) -  \hat{x}^{\lim}(a\Delta_l) }  - \sup_{0 \leq \tau   < \Delta_l } \norm{ \hat{x}^{\lim}(a\Delta_l+\tau) -  \hat{x}^{\lim}(a\Delta_l) } }   \leq \delta.
% \end{align}
Thus, $\forall t\in [0,T), \forall l,  \forall a,$
\begin{align}
\lim_{k \to \infty }  \sup_{0 \leq a \leq \frac{t}{\Delta_l} - 1}  \sup_{ 0 \leq \tau <   \Delta_l }  \norm{ \hat{x}(T_{n_k} + a\Delta_l +\tau) -  \hat{x}^{\lim}(a\Delta_l) } = \sup_{0 \leq a \leq \frac{t}{\Delta_l} - 1}  \sup_{0 \leq \tau   < \Delta_l } \norm{ \hat{x}^{\lim}(a\Delta_l+\tau) -  \hat{x}^{\lim}(a\Delta_l) }.
\end{align}
Therefore,
\begin{align}
&\lim_{k \to \infty } \left[ \sup_{0 \leq a \leq \frac{t}{\Delta_l} - 1} \sup_{m(T_{n_k}+a\Delta_l) \leq i \leq m(T_{n_k} + a\Delta_l + \Delta_l)  - 1  }  \norm{ \hat{x}(t(i)) -  \hat{x}^{\lim}(a\Delta_l) }  \right]  \\
=&  \lim_{k \to \infty }  \sup_{0 \leq a \leq \frac{t}{\Delta_l} - 1}  \sup_{ 0 \leq \tau <   \Delta_l }  \norm{ \hat{x}(T_{n_k} + a\Delta_l +\tau) -  \hat{x}^{\lim}(a\Delta_l) }  \explain{by \eqref{eq: term 2 2}} \\ 
= & \sup_{0 \leq a \leq \frac{t}{\Delta_l} - 1} \sup_{0 \leq \tau   < \Delta_l } \norm{ \hat{x}^{\lim}(a\Delta_l+\tau) -  \hat{x}^{\lim}(a\Delta_l) }. \label{eq: term 2 4}   
\end{align}
$\forall j, \forall t\in [0,T), \forall l$,
\begin{align} 
&\lim_{k \to \infty } \norm{
\sum_{a=0}^{\frac{t}{\Delta_l}  - 1}  \sum_{i = m(T_{n_k}+a\Delta_l)}^{m(T_{n_k} + a\Delta_l + \Delta_l)  - 1 } \alpha(i) \left(H_{r_{n_j}}(\hat{x}(t(i)),Y_{i+1})  - H_{r_{n_j}}(\hat{x}^{\lim}(a\Delta_l),Y_{i+1}) \right)}    \\
\leq& \lim_{k \to \infty } \left[ \sup_{0 \leq a \leq \frac{t}{\Delta_l} - 1} \sup_{m(T_{n_k}+a\Delta_l) \leq i \leq m(T_{n_k} + a\Delta_l + \Delta_l)  - 1  }  \norm{ \hat{x}(t(i)) -  \hat{x}^{\lim}(a\Delta_l) }  \right]
\sum_{i = m(T_{n_k})}^{m(T_{n_k} + t)  - 1 } \alpha(i) L(Y_{i+1}) \explain{by \eqref{eq: term 2 1}} \\
\leq& \lim_{k \to \infty } \left[ \sup_{0 \leq a \leq \frac{t}{\Delta_l} - 1} \sup_{m(T_{n_k}+a\Delta_l) \leq i \leq m(T_{n_k} + a\Delta_l + \Delta_l)  - 1  }  \norm{ \hat{x}(t(i)) -  \hat{x}^{\lim}(a\Delta_l) }  \right] \\
&\limsup_{k \to \infty }  \left[\sum_{i = m(T_{n_k})}^{m(T_{n_k} + t)  - 1 } \alpha(i) L(Y_{i+1})   \right] \\
\leq& \lim_{k \to \infty } \left[ \sup_{0 \leq a \leq \frac{t}{\Delta_l} - 1} \sup_{m(T_{n_k}+a\Delta_l) \leq i \leq m(T_{n_k} + a\Delta_l + \Delta_l)  - 1  }  \norm{ \hat{x}(t(i)) -  \hat{x}^{\lim}(a\Delta_l) }  \right] C_H \explain{by \eqref{eq: alpha L bound C_H}}\\
=& \explaind{C_H  \sup_{0 \leq a \leq \frac{t}{\Delta_l} - 1} \sup_{0 \leq \tau   < \Delta_l } \norm{ \hat{x}^{\lim}(a\Delta_l+\tau) -  \hat{x}^{\lim}(a\Delta_l) }. }{by \eqref{eq: term 2 4}}  \label{eq: term 2 6}
% =&  t L(Y_{i+1})  \sup_{0 \leq a \leq \frac{t}{\Delta_l} - 1} \sup_{0 \leq \tau   < \Delta_l } \norm{ \hat{x}^{\lim}(a\Delta_l+\tau) -  \hat{x}^{\lim}(a\Delta_l) }  \label{eq: term 2 6}
\end{align}
By Corollary \ref{corollary: lim continuous}, $\hat{x}^{\lim}$ is continuous and $[0,t]$ is a compact set, $\forall \epsilon >0, \exists \eta$ such that
\begin{align}
\sup_{0\leq |t_1-t_2| \leq \eta, t_1 \in [0,t],t_2 \in [0,t]  } \norm{\hat{x}^{\lim}(t_1) - \hat{x}^{\lim}(t_2)} \leq \epsilon.   \label{eq: term 2 5} 
\end{align}
Thus, $\forall \epsilon >0$, $\exists l_0$ such that $\forall l \geq l_0, \Delta_{l} \leq \eta$ and we will have
\begin{align}
0\leq \sup_{0 \leq a \leq \frac{t}{\Delta_l} - 1} \sup_{0 \leq \tau   < \Delta_l } \norm{ \hat{x}^{\lim}(a\Delta_l+\tau) -  \hat{x}^{\lim}(a\Delta_l) } \leq \epsilon. \explain{by \eqref{eq: term 2 5}}
\end{align}
Therefore, $\forall t,$
\begin{align}
\lim_{l \rightarrow \infty} \left[ \sup_{0 \leq a \leq \frac{t}{\Delta_l} - 1} \sup_{0 \leq \tau   < \Delta_l } \norm{ \hat{x}^{\lim}(a\Delta_l+\tau) -  \hat{x}^{\lim}(a\Delta_l) }   \right]  = 0  \label{eq: term 2 7}.
\end{align}
This concludes
$\forall j, \forall t\in [0,T),$
\begin{align}
&\lim_{l \to \infty } \lim_{k \to \infty}   \norm{
 \sum_{a=0}^{\frac{t}{\Delta_l}  - 1}  \sum_{i = m(T_{n_k}+a\Delta_l)}^{m(T_{n_k} + a\Delta_l + \Delta_l)  - 1 } \alpha(i) \left(H_{r_{n_j}}(\hat{x}(t(i)),Y_{i+1})  - H_{r_{n_j}}(\hat{x}^{\lim}(a\Delta_l),Y_{i+1}) \right)}   \\
=& \lim_{l \to \infty } C_H  \sup_{0 \leq a \leq \frac{t}{\Delta_l} - 1} \sup_{0 \leq \tau   < \Delta_l } \norm{ \hat{x}^{\lim}(a\Delta_l+\tau) -  \hat{x}^{\lim}(a\Delta_l) }  \explain{by \eqref{eq: term 2 6}} \\ 
=&C_H  \cdot 0 \explain{by \eqref{eq: term 2 7}}\\
=& 0. 
\end{align}

\end{proof}

\subsection{Proof of Lemma~\ref{lemma: three term 3}}\label{appendix: three term 3}
\begin{proof}
By \eqref{eq: h property H h minus 0}, $\forall j, \forall a,  \forall l$,
\begin{align}
\lim_{k \to \infty}  \norm{ \sum_{i = m(T_{n_k} + a\Delta_l)}^{m(T_{n_k} + a\Delta_l + \Delta_l)  - 1 } \alpha(i) \left[H_{r_{n_j}}(\hat{x}^{\lim}(a\Delta_l),Y_{i+1}) - h_{r_{n_j}}(\hat{x}^{\lim}(a\Delta_l)) \right] } = 0. \label{eq: term 3 3}
\end{align}
Thus,$\forall j, \forall t\in [0,T),$
\begin{align}
&\lim_{l \to \infty}  \lim_{k \to \infty}  \norm{ \sum_{a=0}^{\frac{t}{\Delta_l}  - 1}  \sum_{i = m(T_{n_k}+a\Delta_l)}^{m(T_{n_k} + a\Delta_l + \Delta_l)  - 1 } \alpha(i) \left[H_{r_{n_j}}(\hat{x}^{\lim}(a\Delta_l),Y_{i+1}) - h_{r_{n_j}}(\hat{x}^{\lim}(a\Delta_l)) \right]} \\
\leq& \lim_{l \to \infty}   \sum_{a=0}^{\frac{t}{\Delta_l}  - 1}  \lim_{k \to \infty} \norm{\sum_{i = m(T_{n_k}+a\Delta_l)}^{m(T_{n_k} + a\Delta_l + \Delta_l)  - 1 } \alpha(i) \left[H_{r_{n_j}}(\hat{x}^{\lim}(a\Delta_l),Y_{i+1}) - h_{r_{n_j}}(\hat{x}^{\lim}(a\Delta_l)) \right]} \\
=& \lim_{l \to \infty}   \sum_{a=0}^{\frac{t}{\Delta_l}  - 1}  0 \explain{by \eqref{eq: term 3 3}}\\
=&  0.
\end{align}

\end{proof}

\clearpage
%%%%%%%%%%%%%%%%%%%%%%%%%%%%%%%%%%%%%%%%%%%%%%%%%%%%%%%%%%%%
% \bibliographystyle{apalike}
\bibliography{bibliography}

\newpage

%%%%%%%%%%%%%%%%%%%%%%%%%%%%%%
%END DOCUMENT
%%%%%%%%%%%%%%%%%%%%%%%%%%%%%%
\end{document}